\documentclass[journal]{IEEEtran}
\usepackage{amsmath,amsfonts}
\usepackage{algorithmic}
\usepackage{algorithm}
\usepackage{array}
\usepackage{textcomp}
\usepackage{stfloats}
\usepackage{url}
\usepackage{verbatim}
\usepackage{graphicx}
\usepackage{cite}
\usepackage{amsmath}
\usepackage{amsthm}
\usepackage{subfigure}
\usepackage{color}
\usepackage{makecell}
\usepackage{booktabs}
\usepackage[hidelinks]{hyperref}
\usepackage{xr}     

\usepackage{xcolor}
\definecolor{db}{rgb}{0.0, 0.2, 0.7}

\newtheorem{thm}{Theorem}
\newtheorem{lem}{Lemma}
\newtheorem{rem}{Remark}
\newtheorem{cor}{Corollary}
\newtheorem{prop}{Proposition}
\newtheorem{defn}{Definition}

\begin{document}

\title{Time-optimal Convexified Reeds-Shepp Paths on a Sphere}

\author{Sixu Li,~\IEEEmembership{Student Member,~IEEE}, Deepak Prakash Kumar,~\IEEEmembership{Student Member,~IEEE}, Swaroop Darbha,~\IEEEmembership{Fellow,~IEEE}, Yang Zhou ,~\IEEEmembership{Member,~IEEE}
\thanks{(Corresponding authors: Swaroop Darbha, Yang Zhou)}
\thanks{Sixu Li and Yang Zhou are with the Zachry Department of Civil $\&$ Environmental Engineering, Texas A$\&$M University, College Station, TX 77843, USA (e-mail: sixuli@tamu.edu; yangzhou295@tamu.edu).}%
\thanks{Deepak Prakash Kumar and Swaroop Darbha are with the Department of Mechanical Engineering, Texas A\&M University, College Station, TX 77843 USA (e-mail: deepakprakash1997@gmail.com; dswaroop@tamu.edu).}}



\maketitle

\begin{abstract}

This article studies the time-optimal path planning problem for a convexified Reeds–Shepp (CRS) vehicle on a unit sphere, capable of both forward and backward motion, with speed bounded in magnitude by 1 and turning rate bounded in magnitude by a given constant. For the case in which the turning-rate bound is at least 1, using Pontryagin’s Maximum Principle and a phase-portrait analysis, we show that the optimal path connecting a given initial configuration to a desired terminal configuration consists of at most six segments drawn from three motion primitives: tight turns, great circular arcs, and turn-in-place motions. A complete classification yields a finite sufficient list of 23 optimal path types with closed-form segment angles derived. The complementary case in which the turning-rate bound is less than 1 is addressed via an equivalent reformulation. The proposed formulation is applicable to underactuated satellite attitude control, spherical rolling robots, and mobile robots operating on spherical or gently curved surfaces. The source code for solving the time-optimal path problem and visualization is publicly available at \href{https://github.com/sixuli97/Optimal-Spherical-Convexified-Reeds-Shepp-Paths}{https://github.com/sixuli97/Optimal-Spherical-Convexified-Reeds-Shepp-Paths}.

\end{abstract}


\section{Introduction}
\IEEEPARstart{A}{utonomous} vehicles and robotics have seen significant advancements and applications recently, creating a growing demand for effective path planning across diverse scenarios. In the realm of path planning, time-optimal paths hold significant value because they provide a fundamental metric on the configuration space \cite{balkcom2006time}. Additionally, they act as effective motion primitives \cite{frazzoli2005maneuver,mueller2015computationally}, which can be sampled or sequentially combined to tackle more complex tasks such as obstacle avoidance. Time-optimal path planning problems typically seek the minimum-time path connecting an initial and a desired terminal configuration\footnote{The configuration of a vehicle is defined by its position and orientation.} of a vehicle subject to kinematic constraints. While time-optimal path planning has been extensively studied in the plane and in certain three-dimensional settings, the structure of time-optimal paths on curved manifolds, such as the sphere, remains far less understood.
In particular, among the existing literature, the Dubins vehicle \cite{dubins1957curves} is the only model for which time-optimal paths on a sphere have been characterized \cite{darbha2023optimal,kumar2024generalization}.
\begin{figure}[!t]
    \centering
    {\includegraphics[width=0.20\textwidth]{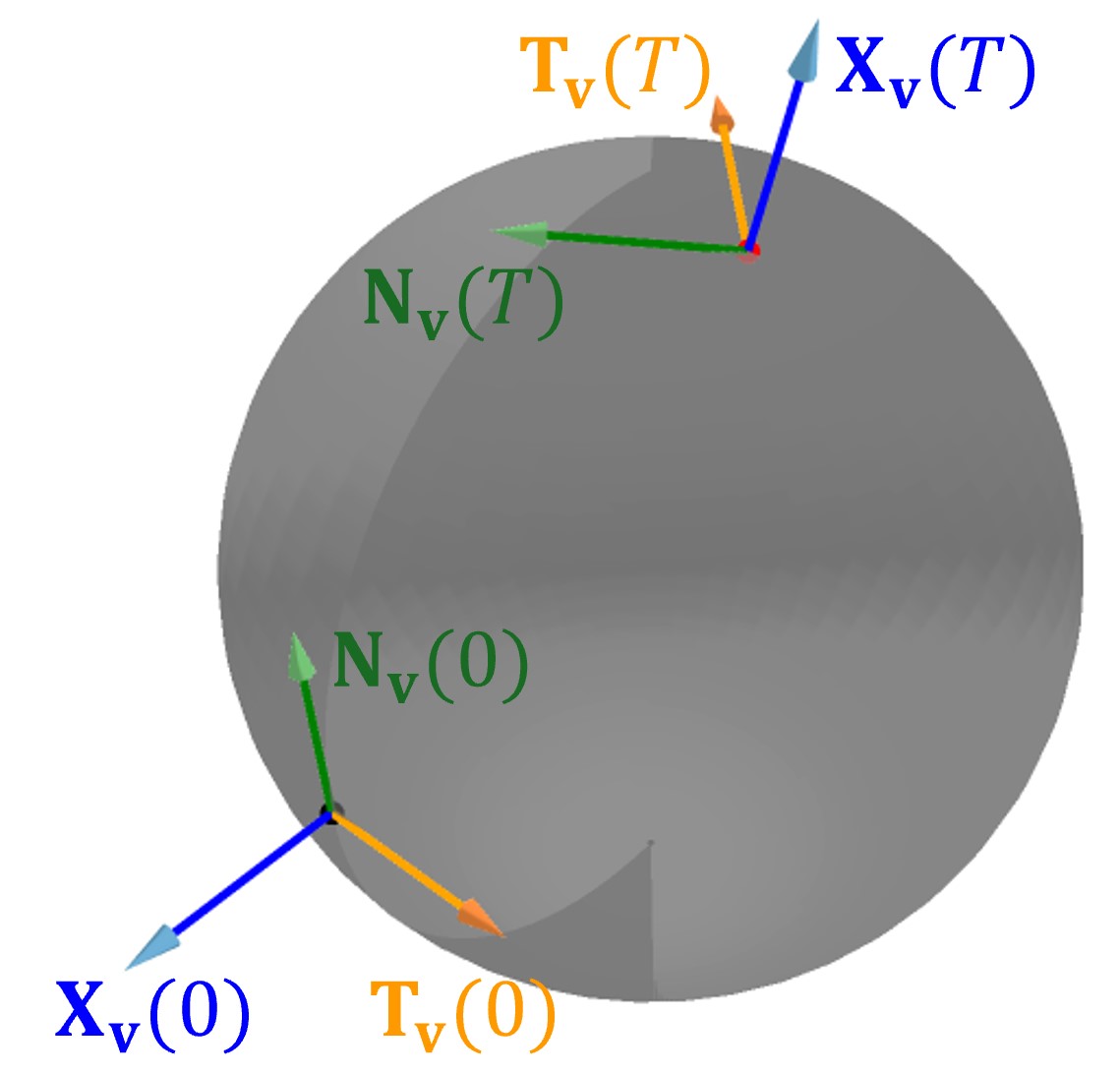}}
    \caption{Configurations on a sphere}
    \label{fig:problem_form}
    \end{figure}
    \begin{figure}[!t]
    \centering
    \subfigure[fixed $v$, different constant $u_g$]{\includegraphics[width=0.20\textwidth]{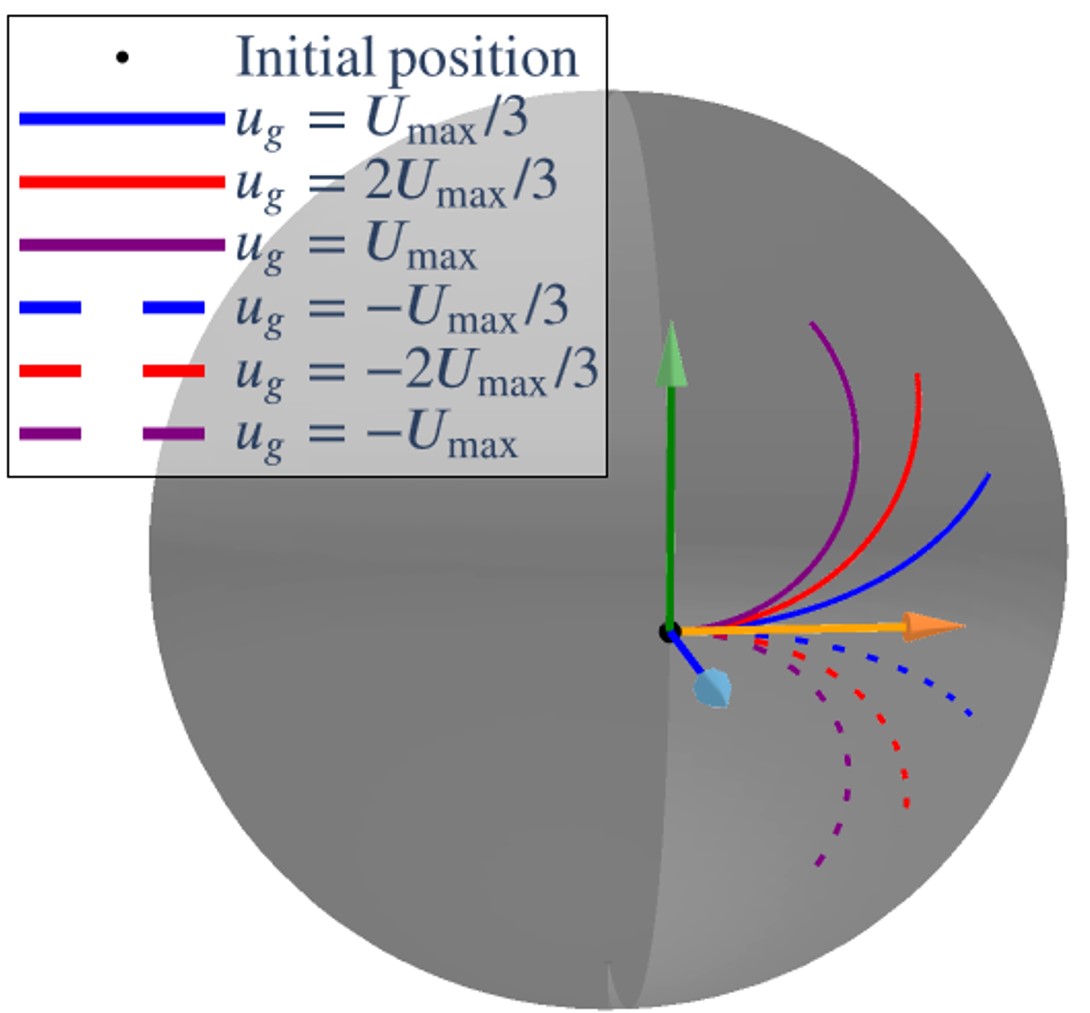}}
    \hspace{0.04\textwidth}
    \subfigure[fixed $u_g$, different constant $v$]{\includegraphics[width=0.186\textwidth]{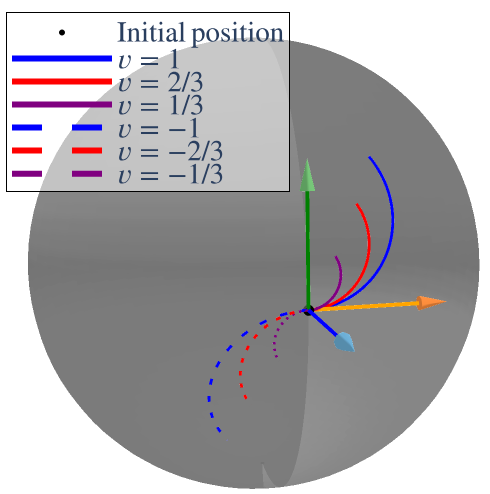}}
    \caption{Illustration of $u_g$ and $v$}
    \label{fig:vary}
    \end{figure}

In this paper, we extend this line of research to the convexified Reeds–Shepp (CRS)\footnote{In this paper, we use the term CRS to denote the variant of the classical Reeds-Shepp model with the convexified constraint $v\in[-1,1]$. The term “convexified’’ does not imply an approximation or the existence of smooth (non-bang-bang) optimal solutions.} vehicle \cite{sussmann1991shortest}, a variant of the classical Reeds–Shepp (RS) model \cite{reeds1990optimal}, and study the corresponding time-optimal\footnote{The notion of optimality is understood with respect to the CRS model, which is subject to the same ideal kinematic assumptions as the classical Dubins and RS models.} path planning problem on the sphere. Unlike the planar case, the intrinsic curvature of the sphere further couples translational and rotational motion, so that even locally straight motion induces changes in orientation, fundamentally altering the structure of optimal paths. As depicted in Fig. \ref{fig:problem_form}, the configuration of a CRS vehicle on a sphere is represented by a collection of vectors $\bold{X_v}(t)$, $\bold{T_v}(t)$, and $\bold{N_v}(t)$, which denote the CRS vehicle's position, heading direction, and lateral direction, respectively. The spherical CRS problem involves finding a path that starts from an initial configuration  $[\bold{X_v}(0), \bold{T_v}(0), \bold{N_v}(0)]$ at $t=0$, and reaches a desired terminal configuration  $[\bold{X_v}(T), \bold{T_v}(T), \bold{N_v}(T)]$ at $t=T$, while minimizing $T$. The CRS vehicle is subjected to the input constraints $u_g(t)\in[-U_{max},U_{max}]$ and $v(t)\in[-1,1]$. Here, $u_g$ represents the turning rate, dictating the CRS vehicle's ability to turn away from its heading direction $\bold{T_v}(t)$, while $v$ denotes the vehicle's speed, determining its ability to move along $\bold{T_v}(t)$. Fig. \ref{fig:vary}(a) shows fixed-time paths with a fixed speed $v=1$ and various constant values of $u_g$. Since $v$ is fixed, the arc lengths of the paths are identical, but the change in heading increases as $|u_g|$ increases. Fig. \ref{fig:vary}(b) shows fixed-time paths with a fixed turning rate $u_g=U_{max}$ and various constant values of $v$. Since $u_g$ is fixed, the heading change remains the same across paths, while the arc length increases as $|v|$ increases. It is worth noting that there is no loss of generality in considering a unit maximum $|v|$ and a unit sphere since the distance and time can be scaled.

 The spherical CRS problem is motivated by three primary types of real-world applications: 
 
1) \textit{Optimal attitude control of underactuated satellites:} In scenarios where actuators fail, a satellite could become underactuated \cite{yoshimura2011position}; research from \cite{tsiotras2000control, Krishnan,crouch1984spacecraft} shows that only two control inputs suffice to control the pose of a satellite. The spherical CRS model is equivalent to the satellite model featuring two reaction wheels\footnote{Reaction wheels are internal satellite actuators that control attitude by exchanging angular momentum, allowing reorientation without external thrust.} \cite{bullo1995control,murray2017mathematical}.  This equivalence will be shown in more detail in Section \ref{sec2}. This work distinguishes from earlier work in the following two aspects: (1) the objective is to minimize the time to change pose with limited control effort, and (2) it is assumed that the reaction wheels' angular velocities are controlled instantaneously; therefore, this study can be considered as the development of a planner that produces a time-optimal trajectory, including references for both the satellite's pose and higher-level control signals, which can subsequently be tracked by a lower-level controller.

2) \textit{Optimal motion planning for spherical rolling robots \cite{diouf2024spherical} with an internal drive unit (IDU).} These robots possess a spherical outer shell with the IDU placed at the bottom inside. The IDU typically includes a unicycle maintaining continuous contact with the inner surface of the shell \cite{halme1996motion,bicchi1997introducing,zhan2011design}. By controlling wheel rotation and direction, the IDU induces the shell to roll. Fig. \ref{fig:sphere robot} illustrates such a robot schematically. The kinematics of the shell's pose relative to the IDU \cite{bicchi1997introducing,zhan2011design} is equivalent to the spherical CRS model, which describes the model of a CRS vehicle traveling on a unit sphere. This equivalence will be shown in more detail in Section \ref{sec2}. The model helps to perform motion planning tasks. For example, in reconnaissance, a shutter is incorporated into the shell, which allows sensors to extend outward \cite{bujvnak2022spherical}. The robot typically needs to reorient to position the shutter on top before extending the sensors. Assuming the wheels do not slip inside the shell and the ground friction is low, the shell-ground contact point will eventually be where the IDU comes to a stop on the shell. Hence, re-adjusting the shutter position can be achieved by directing the IDU to the shutter's antipodal point, as shown in Fig. \ref{fig:sphere robot plan}. The dashed line from the shutter to the IDU's target runs through the sphere's center, with the red line showing the IDU's optimal path on the sphere.

\begin{figure}[h]
    \centering
    {\includegraphics[width=0.15\textwidth]{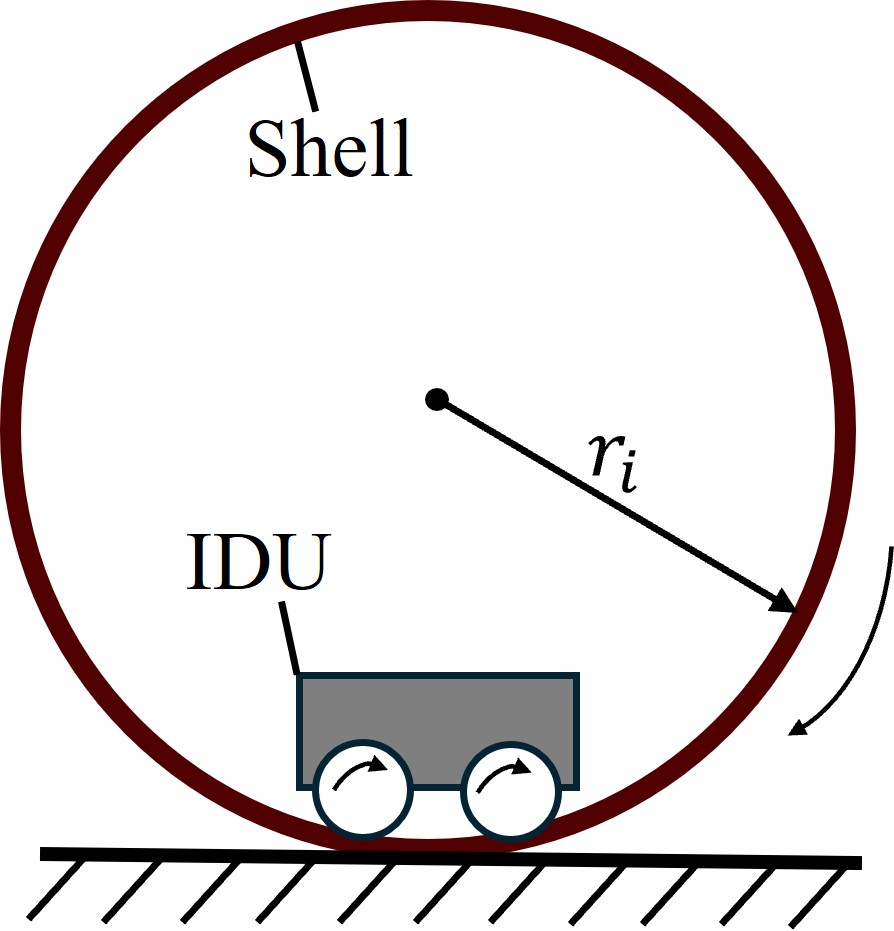}}
    \caption{Schematic plot of the spherical rolling robot}
    \label{fig:sphere robot}
    \end{figure}

    \begin{figure}[h]
    \centering
    \subfigure[Initial configuration]{\includegraphics[width=0.155\textwidth]{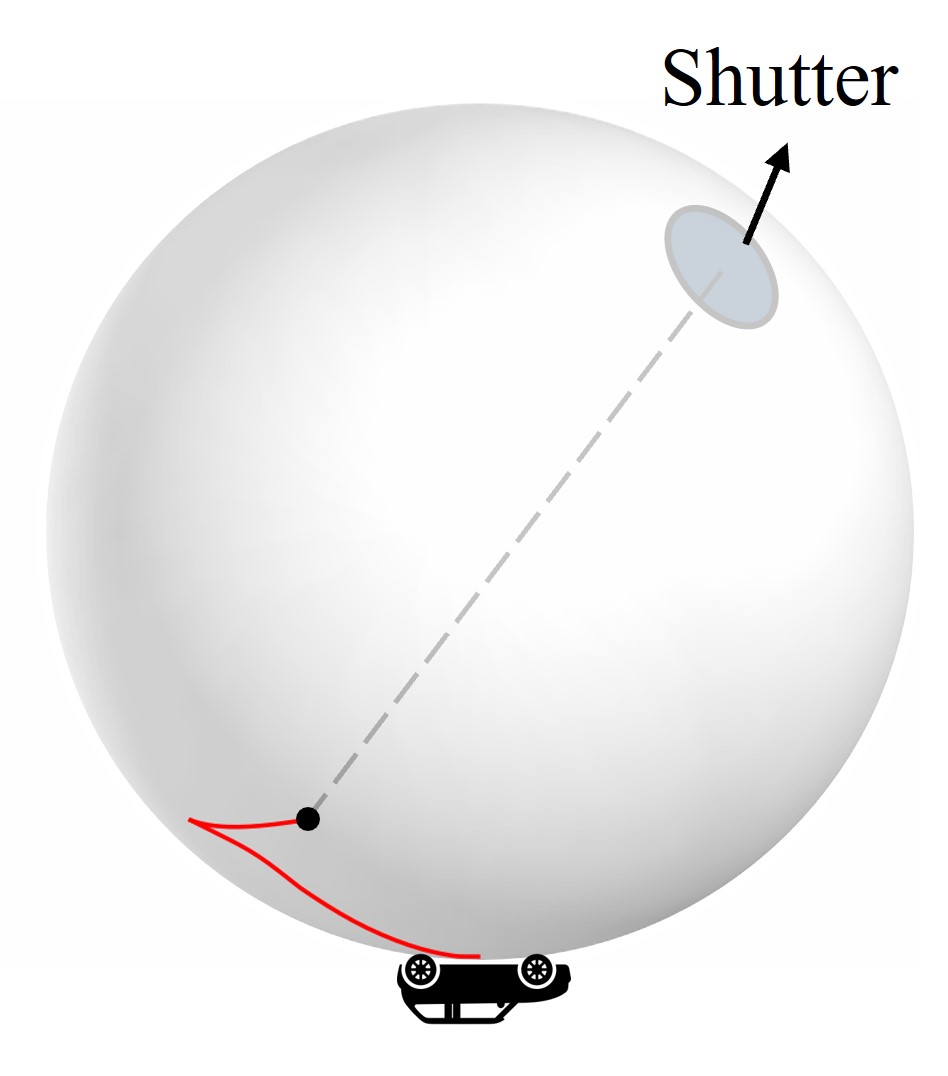}}
    \hspace{0.04\textwidth}
    \subfigure[Terminal configuration]{\includegraphics[width=0.15\textwidth]{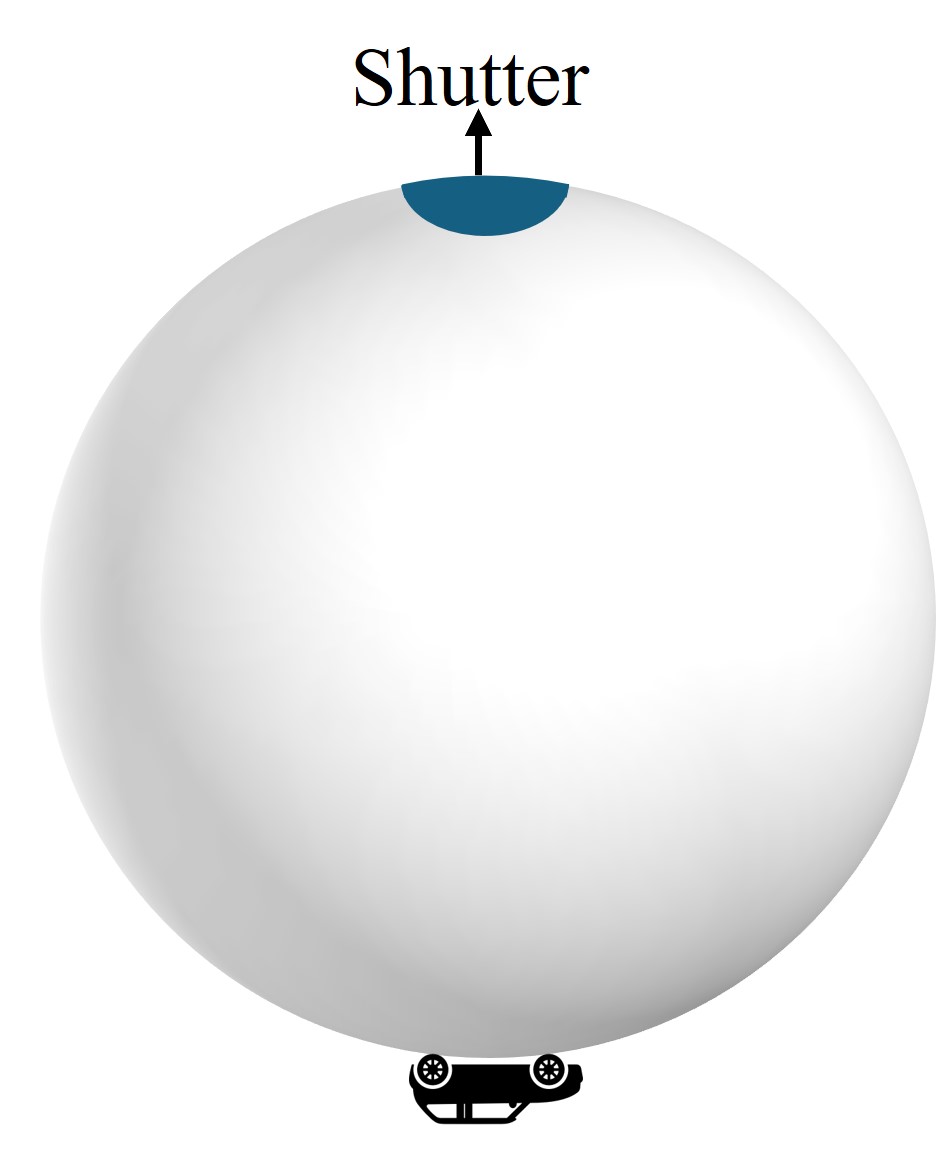}}
    \caption{Repositioning of the shutter}
    \label{fig:sphere robot plan}
    \end{figure}

3) \textit{Optimal path planning for a CRS-like robot (e.g., with differential drive) on spherical surfaces or uneven terrains.} The application on spherical surfaces is straightforward and is utilized in inspection tasks (e.g., for spherical gas tanks \cite{li2022weld,okamoto2012development}). As for uneven terrains, the proposed formulation applies particularly to terrains that can be locally approximated by a spherical patch through curvature matching, indicating that the local terrain has nearly constant or gradually varying non-negative Gaussian curvature. Such terrain characteristics are seen in “gently rolling terrain”, as illustrated, for example, in Figure 2 of \cite{luca2011geomorphological} and Figure 1 of \cite{caruso2012compositional}. Fig. \ref{fig:uneven terrain} demonstrates such local approximation with a spherical patch, where the red line shows the optimal CRS path on the sphere. For a given start/terminal configuration pair, the set of all admissible paths on the spherical patch is contained within the set of all admissible paths on the entire sphere. Consequently, the path in Fig. \ref{fig:uneven terrain} is also optimal within the patch. The difference between the actual terrain and its spherical patch approximation can be treated as a disturbance input to the lower-level controller. 

\begin{figure}[h]
    \centering
    \setlength{\abovecaptionskip}{0pt}
    \includegraphics[width=0.27\textwidth]{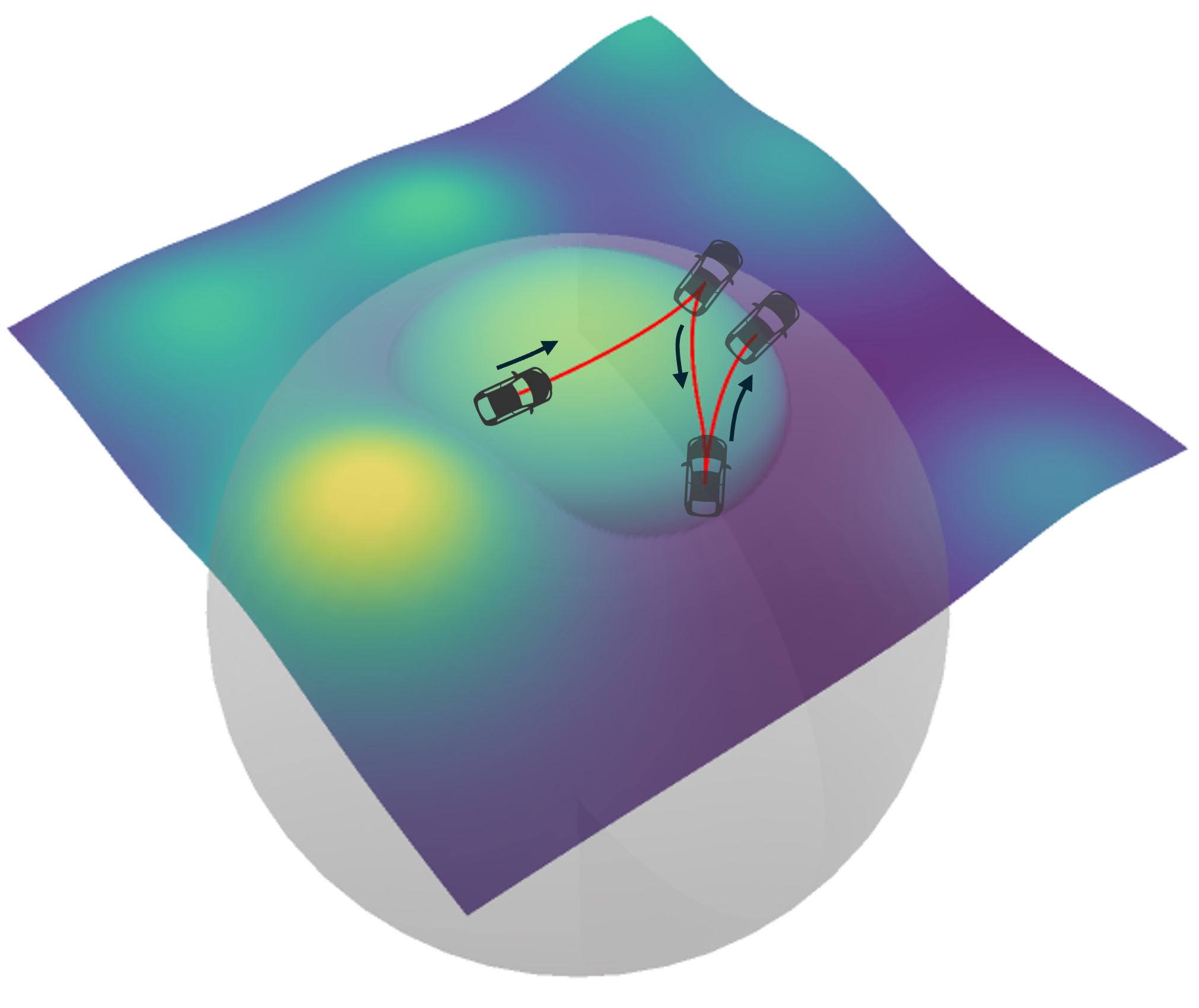}
    \caption{Approximating local uneven terrain with a spherical patch}
    \label{fig:uneven terrain}
\end{figure}

These applications show that a wide range of systems share the same underlying kinematics captured by the spherical CRS model. Solving the corresponding time-optimal planning problem therefore provides fundamental insight, including a lower bound on the fastest achievable motion and the qualitative structure of optimal paths (e.g., switching behavior), and serves as a canonical benchmark for evaluation. In practice, the resulting paths can be used as high-level references, while smoothness and dynamical feasibility are handled by downstream tracking controllers, since the bang-bang nature of kinematic time-optimal solutions precludes perfect execution on dynamically constrained systems.

Another motivating factor is the exploration of mechanism design \cite{balkcom2002time,balkcom2006time}. Expanding the set of vehicles with known optimal paths enables direct comparisons between different actuation models. For example, in a robotic task that requires traveling between two configurations on a sphere in minimum time, we can evaluate whether a Dubins-type vehicle or a CRS-type vehicle\footnote{A Dubins-type vehicle has a single independent control input (the turning rate), with forward speed assumed constant, whereas a CRS-type vehicle admits two independent control inputs (turning rate and forward speed).} is more efficient, while also considering trade-offs between travel time, actuation complexity, and implementation cost.

The main contributions of this paper are as follows: 

1) Formulation of the spherical CRS problem.

2) Deriving necessary conditions for time-optimal paths using the Pontryagin Maximum Principle (PMP) and phase portraits. 

3) Characterization of a sufficient list of the time-optimal path types for $U_{max}\geq1$ (without loss of generality) by proving that some paths satisfying the necessary conditions are non-optimal or redundant\footnote{The threshold 1 in the condition $U_{max}\geq1$ is stated under the adopted normalization (unit sphere, unit maximum speed) and reflects the relative scaling between the bounds on turning rate and speed.}. The complementary case $U_{\max}<1$ can be rewritten into an equivalent formulation with
a new bound $\tilde U_{\max}>1$, for which the results of this paper directly apply;
see Appendix \ref{append equivalence} in the supplementary materials.

4) Derivation of closed-form expressions for the angles
of each path in the sufficient list, given an initial configuration, a desired terminal configuration, and $U_{max}$. 

5) Full release of the source code for the research community  to utilize, enhance, and build upon.

The remainder of this article is structured as follows. Section \ref{sec related} reviews prior results. Section \ref{sec2} presents the problem formulation and utilizes PMP to obtain some basic properties. Section \ref{sec3} further characterizes the optimal paths and demonstrates that the problem can be divided into three distinct cases. These cases are analyzed in detail in Sections \ref{sec 3.1}, \ref{sec3.2}, and \ref{sec3.3}, respectively, obtaining a sufficient list of optimal path types. Section \ref{sec generate} includes a discussion on path generation, along with numerical experiments. Conclusions are drawn in Section \ref{sec5}.

The main symbols used throughout the paper are summarized in Table \ref{tab: key-symbols} for convenience.

\begin{table}[h]
\centering
\caption{Key symbols}
\label{tab: key-symbols}
\setlength{\tabcolsep}{3pt}
\begin{tabular}{ll}
\hline
Symbol & Description \\
\hline
\\[-6pt]$\mathbf{X_v}$ & Position of the CRS vehicle on the unit sphere \\[4pt]
$\mathbf{T_v}$ & Heading (tangent) direction of the vehicle \\[4pt]
$\mathbf{N_v}$ & Lateral direction, completing the moving frame \\[4pt]
$\mathbf{R}$ & Configuration matrix $[\mathbf{X_v}, \mathbf{T_v}, \mathbf{N_v}] \in SO(3)$ \\[4pt]

$v$ & Signed linear speed, $v \in [-1,1]$ \\[4pt]
$u_g$ & Turning rate control, $u_g \in [-U_{max}, U_{max}]$ \\[4pt]
$U_{max}$ & Maximum allowable turning rate \\[4pt]

$h_1, h_2, H_{12}$ & Hamiltonians associated with left-invariant vector fields \\[4pt]
$\mathcal{A}, \mathcal{B}, \mathcal{C}$ &
\makecell[l]{Reparameterized variables defined by $\mathcal{A}:=H_{12}$, $\mathcal{B}:=h_2$,\\
  $\mathcal{C}:=-h_1$} \\[8pt]

$g$ & Invariant constant $g\equiv\mathcal{A}^2+\mathcal{B}^2+\mathcal{C}^2$ along an extremal \\[4pt]

$C$ & Tight turn segment ($|v|=1$, $|u_g|=U_{max}$) \\[4pt]
$G$ & Great-circle arc ($|v|=1$, $u_g=0$) \\[4pt]
$T$ & Turn-in-place segment ($v=0$, $|u_g|=U_{max}$) \\[4pt]
$|$ & Cusp (switch in the sign of $v$) \\[4pt]

$L, R$ & \makecell[l]{Action labels indicating saturated turning rate:
$L$ for $u_g=$\\ $U_{max}$ and
$R$ for $u_g=-U_{max}$}  \\[8pt]

$+,0,-$ & \makecell[l]{Superscripts indicating the speed mode of a segment: $+$ for\\
 $v=1$, $0$ for $v=0$,
and $-$ for $v=-1$}\\[8pt]

$r$ & Radius of a tight turn, $r=\frac{1}{\sqrt{1+U_{max}^2}}$ \\[4pt]

$c(\cdot), s(\cdot)$ & Shorthand notations for $\cos(\cdot)$ and $\sin(\cdot)$, respectively \\[4pt]

$\alpha$, $\sigma$, $\rho$ & Angles with values $\alpha>0$, $\sigma\geq0$, and $\rho\geq0$ \\[4pt]

$\beta$ & Angle with value $\beta=\arctan(\frac{1}{\sqrt{U_{max}^4-1}})+\frac{\pi}{2}$ \\[4pt]

$\mu$ & Angle with value $0<\mu<\arctan(\frac{1}{\sqrt{U_{max}^4-1}})+\frac{\pi}{2}$ \\[4pt]

$\epsilon$ & Angle with value $0<\epsilon<2\mu$ \\[4pt]

\hline
\end{tabular}
\end{table}

\section{Related Work}
\label{sec related}
To place the present work in context, we next review prior results on time-optimal path planning in planar, three-dimensional (3D), and spherical settings.

\subsection{Time-optimal Path Planning in a Plane}

In \cite{dubins1957curves}, the time-optimal path problem was solved for a forward-moving vehicle with constant speed $v=1$, and bounded turning rate\footnote{With a constant speed, a bounded turning rate corresponds to a minimum turning radius.} $u\in[-U_{max},U_{max}]$, which is known as the Dubins vehicle. The authors proved that the optimal path must be of types $CCC$, $CSC$, or their degenerate forms, where $C$ represents a left or right turn with a maximum turning rate and $S$ denotes a straight line. In \cite{reeds1990optimal}, the problem was extended to the RS vehicle, which moves both forward and backward with $v\in\{-1,1\}$. It was proven that the optimal path must be one of the following types:  $CSC,$ $C|C|C,$ $CC|C,$ $C|CC,$ $CC|CC,$ $C|CC|C,$ $C|CSC,$ $CSC|C,$ $C|CSC|C$, or its degenerate form; here, $``|"$ represents a cusp. 

The aforementioned pioneering works relied on geometry and differential calculus. Subsequent improvements to the RS results were made using PMP \cite{pontryagin2018mathematical}, as seen in \cite{boissonnat1994shortest} and \cite{sussmann1991shortest}. In particular, as a proof technique, \cite{sussmann1991shortest} addressed the RS problem by relaxing the constraint $v\in\{-1,1\}$ to $v\in[-1,1]$. The authors showed that, in a plane, the sufficient list of time-optimal paths for the relaxed problem is also admissible for the RS problem, indicating that the sufficient lists for both problems are identical. The resulting model with the convexified constraint $v\in[-1,1]$ is known as the CRS vehicle. Later, in \cite{chitsaz2009minimum}, it was proved that the sufficient lists for the time-optimal CRS problem and the minimum wheel-rotation differential-drive problem also coincide. A phase portrait approach was recently used in \cite{kaya2017markov} to obtain a simplified proof utilizing PMP. A similar phase-portrait approach was used to solve the weighted Dubins problem using PMP in \cite{kumar2023weighted}. Time-optimal paths of differential drive vehicles, omni-directional vehicles, and car-like mobile robots in a plane were studied using PMP in \cite{balkcom2002time}, \cite{balkcom2006time}, and \cite{ben2021time}, respectively.

\subsection{Time-optimal Path Planning in 3D and on a sphere}
Although optimal path planning in 3D has posed difficulties, several significant studies exist. In \cite{sussmann1995shortest}, the problem of a 3D Dubins path with constraints on total curvature was investigated, demonstrating that the time-optimal path is either a helicoidal arc or is composed of up to three segments. However, this study does not take into account the full configuration of the vehicle in 3D space. In \cite{chitsaz2007time}, the Dubins vehicle model was extended to incorporate altitude, allowing for the modeling of airplanes. The study determined the time-optimal paths for achieving final altitudes categorized as low, medium, and high. Nevertheless, the full configuration space was not comprehensively considered, and the challenges for control synthesis remained open. For surfaces of non-negative curvature, the existence conditions of Dubins paths were given in \cite{chitour2005dubins} without addressing the optimality of the paths. 

Exploring optimal paths on a sphere is another valuable area, particularly due to its relevance in path planning on uneven terrain and planetary surfaces, as well as in attitude control. In \cite{thomas2025formulating}, a linear time-varying formulation of a unicycle on the sphere was proposed to generate smooth paths that are advantageous for direct execution. However, the formulation is aimed at stabilization rather than time-optimal planning. In \cite{monroy1998non}, it was shown that the results of planar Dubins generalize to a unit sphere for the specific value $r=\frac{1}{\sqrt{2}}$ (or $U_{max}=1$), where $r=\frac{1}{\sqrt{1+U_{max}^2}}$ denotes the radius of the tight turn. In \cite{darbha2023optimal}, the spherical Dubins problem was studied for $r\leq\frac{1}{2}$ (or $U_{max}\geq\sqrt{3}$) by modeling the vehicle using a Sabban frame and constructing a geodesic curvature-constrained time-optimal path problem. The authors proved that for $r\leq\frac{1}{2}$, the optimal path is of types $CCC$, $CGC$, or their degenerate forms, where $C$ represents a left or right turn with the maximum absolute geodesic curvature, and $G$ denotes a great circular arc on the sphere---a geodesic. In \cite{kumar2024generalization}, the spherical Dubins problem was extended to free terminal orientation using PMP. The authors proved that for $r\leq\frac{\sqrt{3}}{2}$ (or $U_{max}\geq\frac{1}{\sqrt{3}}$), the time-optimal paths must be of types $CG,$ $CC$, or their degenerate forms.

Clearly, the CRS problem on a sphere has not been addressed in the literature. In this work, we address this gap by proposing an appropriate model, deriving a finite sufficient list of optimal path types, and providing the analytical constructions of candidate paths within this list.

\section{Problem Formulation \label{sec2}}

In this section, the problem formulation is proposed, and PMP is utilized to obtain some basic properties of the optimal path. 

\subsection{The spherical CRS problem}
In this paper, the time-optimal path of a CRS vehicle on a sphere is considered. The problem is modeled based on the spherical Dubins vehicle model using a Sabban frame\footnote{The Sabban frame is the orthonormal moving frame associated with a curve on the unit sphere.}, as proposed in \cite{darbha2023optimal}:

\begin{align}
\label{DU model}
&\frac{d\bold{X}}{ds} = \bold{T}(s),~ ~
\frac{d\bold{T}}{ds} = -\bold{X}(s) + u_g(s)\bold{N}(s), \notag \\
&\frac{d\bold{N}}{ds} = -u_g(s)\bold{T}(s),
\end{align}
where $\bold{X}, \bold{T}, \bold{N}$ denote the position vector, tangent vector, and the tangent-normal vector, respectively, and form an orthonormal basis that describes the location and orientation of a Dubins vehicle on a sphere (see Fig. 1 of \cite{darbha2023optimal}). Furthermore, $s$ represents the arc length traversed by the Dubins vehicle and $u_g$ denotes the geodesic curvature.


We first derive the spherical RS model. When considering an RS path, a new control variable $v\in\{-1,1\}$ for the velocity needs to be introduced to describe the forward/backward movement of an RS vehicle. To distinguish between the forward/backward movements along the spherical curve $\bold{X}(s)$, we define a new set of state vectors, $\bold{X_v}(s), \bold{T_v}(s)$ and $\bold{N_v}(s)$. Here, $\bold{X_v}(s)= \bold{X}(s)$ represents the position of the RS vehicle on a sphere, $\bold{T_v}(s)= v(s)\bold{T}(s)$ represents the direction that the RS vehicle is facing, and $\bold{N_v}(s)=\bold{X_v}(s)\wedge\bold{T_v}(s)=v(s)\bold{N}(s)$. The relation between $[\bold{X_v}(s), \bold{T_v}(s),\bold{N_v}(s)]$ and $[\bold{X}(s), \bold{T}(s),\bold{N}(s)]$ is shown in Fig. \ref{fig:basis relation}. Similar to \cite{sussmann1991shortest}, we first derive the spherical RS model (with $v \in \{-1,1\}),$ and then relax $v$ to lie in $[-1,1]$ to obtain the spherical CRS model.

Noting that for the RS model, $v\in\{-1,1\}$ is piecewise constant and $\frac{ds}{dt}=|v|=1$, the new state equations re-parameterized with respect to time\footnote{For Dubins and RS vehicles, time and arc length are equivalent since $|v|=1$. In contrast, $|v|$ varies for a CRS vehicle, necessitating a re-parameterization.} almost everywhere\footnote{The equations hold almost everywhere because $v\in\{-1,1\}$ is piecewise constant and may switch at isolated time instants where derivatives involving 
$v$ (e.g., $\frac{d(v\mathbf{T})}{dt}$) are undefined.} are obtained as $\frac{d\bold{X_v}}{dt} =\frac{d\bold{X}}{dt}= \bold{T}(t)=\bold{T_v}(t)/v(t)=v(t)\bold{T_v}(t)$ (since $|v(t)|=1$)\footnotemark\footnotetext{We place $v(t)$ in the numerator because this form will also be used later for the CRS model, where $v(t)=0$ is allowed. In that case, $v(t)=0$ in the numerator correctly yields $\frac{d\bold{X_v}}{dt}=\mathbf 0$ (no motion), whereas placing $v(t)$ in the denominator would make $\frac{d\bold{X_v}}{dt}$ undefined.}, $\frac{d\bold{T_v}}{dt} = \frac{d(v\mathbf{T})}{dt}=v(t)\frac{d\bold{T}}{dt}=-v(t)\bold{X}(t) + v(t)u_g(t)\bold{N}(t)=-v(t)\bold{X_v}(t) + u_g(t)\bold{N_v}(t)$, and $\frac{d\bold{N_v}}{dt}=\frac{d(v\mathbf{N})}{dt}=v(t)\frac{d\bold{N}}{dt} = -v(t)u_g(t)\bold{T}(t)=-u_g(t)\bold{T_v}(t)$.

    \begin{figure}[h]
    \centering
    \subfigure[$v=1$]{\includegraphics[width=0.22\textwidth]{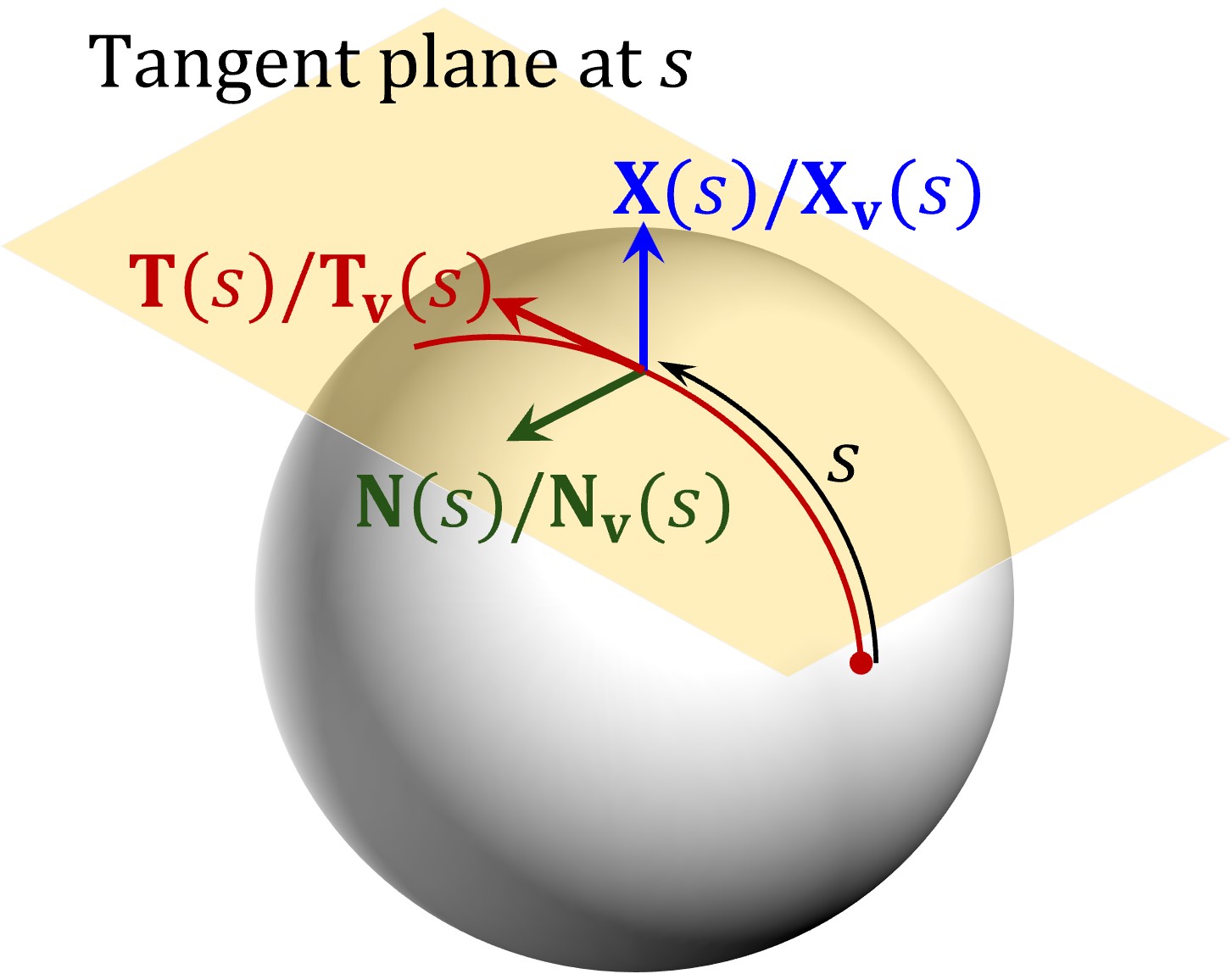}}
    \subfigure[$v=-1$]{\includegraphics[width=0.22\textwidth]{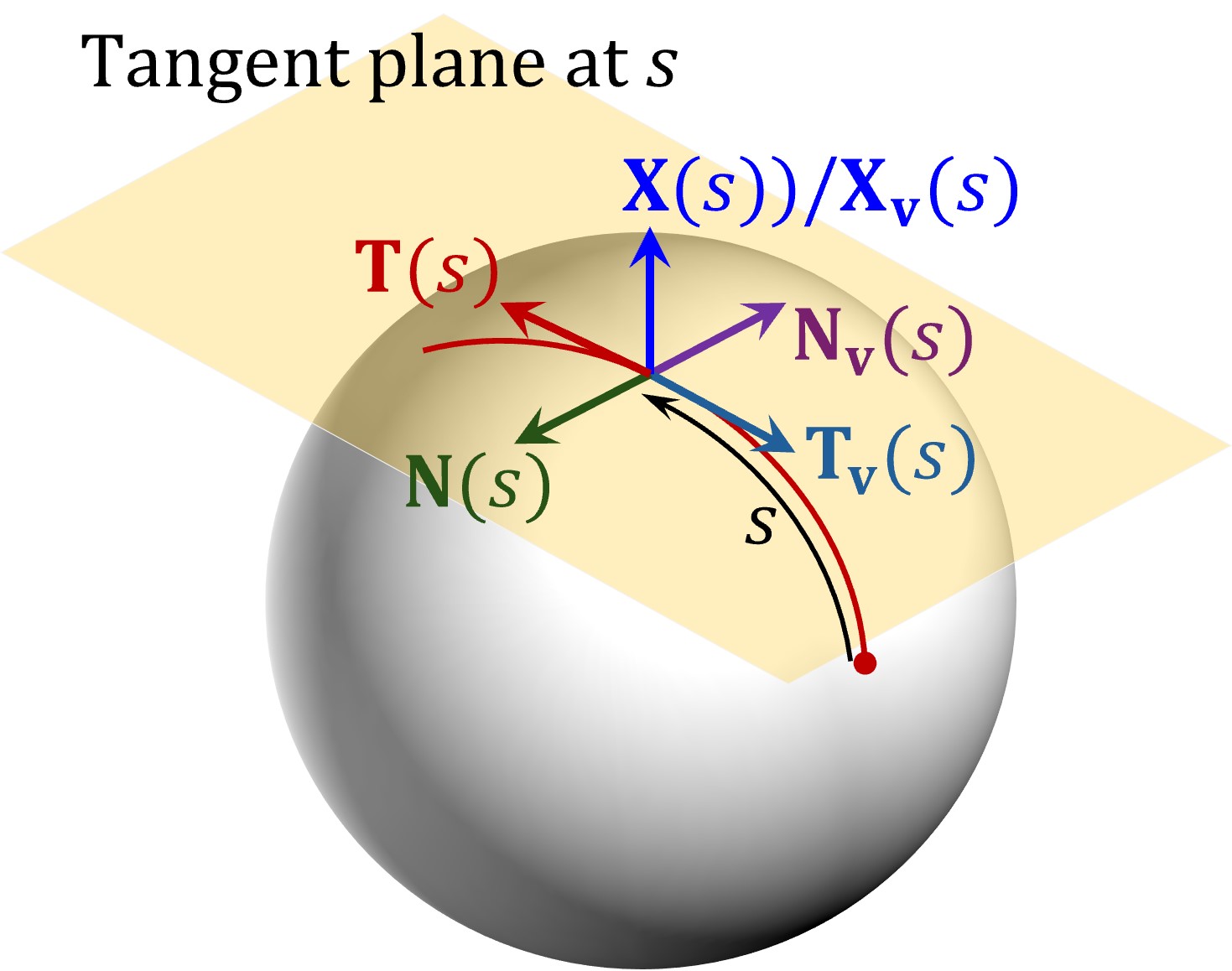}}
    \caption{Relation between $[\bold{X_v}(s), \bold{T_v}(s),\bold{N_v}(s)]$ and $[\bold{X}(s), \bold{T}(s),\bold{N}(s)]$}
    \label{fig:basis relation}
    \end{figure}


Analogous to \cite{sussmann1991shortest}, we derive the CRS model by expanding the admissible set of $v$ to $[-1,1]$, resulting in the formal formulation of the time-optimal spherical CRS problem:
\begin{equation}
\label{eq: time cost}
    J^* = \min \int_0^T 1 \; dt
\end{equation}
subject to
\begin{align}
\label{RS model1}
&\frac{d\bold{X_v}}{dt} = v(t)\bold{T_v}(t), \\
\label{RS model2}
&\frac{d\bold{T_v}}{dt} = -v(t)\bold{X_v}(t) +u_g(t)\bold{N_v}(t), \\
\label{RS model3}
&\frac{d\bold{N_v}}{dt} = -u_g(t)\bold{T_v}(t), \\
\label{terminal conditions}
&\bold{R}(0) =\bold{I_3}, ~\bold{R}(T) =\bold{R_f}, 
\end{align}
where $v \in[-1\text{,}1]$ and $u_g \in [-U_{max},U_{max}]$, $\bold{R}(t)=[\bold{X_v}(t), \bold{T_v}(t), \bold{N_v}(t)]\in SO(3)$, $\bold{I_3}$ is the identity matrix, and $\bold{R_f}$ represents the desired terminal configuration. 

Note that the state equations, (\ref{RS model1})-(\ref{RS model3}), are equivalent to
\begin{align}
\label{RS model}
    \frac{d\bold{R}(t)}{dt}
    = \bold{R}(t) \underbrace{\left( \begin{array}{ccc}
0 & -v(t) & 0 \\
v(t) & 0 & -u_g(t) \\
0 & u_g(t) & 0 \\
\end{array} \right)}_{\Omega(t)}.
\end{align}

\begin{rem}
    In the above model, $u_g$ no longer represents the geodesic curvature. This formulation is analogous to the model in \cite{sussmann1991shortest} for a CRS vehicle in a plane, where $u_g$ specifies the turning rate and allows for turn-in-place motions when $v=0$. 
\end{rem}



\subsection{Examples: Underactuated Satellite and Spherical Rolling Robot}
To highlight the generality of the CRS model, we note that underactuated satellites and spherical rolling robots can be cast into the same form.

\vspace*{0.1in} 

\noindent{\bf Example: Underactuated Satellite:}
    For a satellite with two reaction wheels, following \cite{bullo1995control} and neglecting its initial angular momentum \cite{murray2017mathematical}, the satellite's kinematics is obtained as: 
    \begin{equation}
        J\omega_s = e_1v_1 + e_3v_2,
    \end{equation}
    where $J=diag(J_1,J_2,J_3)$ denotes the inertia matrix, $\omega_s$ denotes the angular velocity vector of the satellite relative to its body frame, $e_1 = [1, 0, 0]^T$, $e_3 = [0, 0, 1]^T$, and $v_i\in[-v_{i,max},v_{i,max}]$ represent the scaled reaction wheel velocities. Using the above equation and mapping the angular velocity vector to its corresponding skew-symmetric matrix:
    \begin{equation}
        \frac{d\bold{R}(t)}{dt}
    = \bold{R}(t) \left(\begin{array}{ccc}
0 & -\frac{v_2(t)}{J_3} & 0 \\
\frac{v_2(t)}{J_3} & 0 & -\frac{v_1(t)}{J_1} \\
0 & \frac{v_1(t)}{J_1} & 0 \\
\end{array} \right),
    \end{equation}
    where in this case, $\bold{R}(t)\in SO(3)$ represents the pose of the satellite. Redefining new control variables $u_g = \frac{v_1J_3}{v_{2,max}J_1}$, $v=\frac{v_2}{v_{2,max}}$, and scaling time with $\tau=\frac{v_{2,max}}{J_3}t$, it is obtained that 
    \begin{equation}
    \label{satellite equiv}
        \frac{d\bold{R}(\tau)}{d\tau}
    = \bold{R}(\tau) \Omega(\tau),
    \end{equation}
    and $v\in[-1,1]$, $u_g\in[-\frac{v_{1,max}J_3}{v_{2,max}J_1},\frac{v_{1,max}J_3}{v_{2,max}J_1}]$. It is clear that (\ref{RS model}) and (\ref{satellite equiv}) are equivalent.


\vspace*{0.1in}

\noindent{\bf Example: Spherical Rolling Robot:} 
    For a spherical rolling robot, following \cite{zhan2011design}, the kinematics describing the robot's pose is as follows:
    \begin{equation}
    \left\{
    \begin{array}{l}
        \dot{x} = \frac{u_1}{R} \frac{\sin(z - \theta_v)}{\cos y} \\
    \dot{y} = \frac{u_1}{R} \cos(z - \theta_v) \\
    \dot{z} = \frac{u_1}{R} \tan y \sin(z - \theta_v). 
    \end{array}
    \right.
    \end{equation}
    where $x, y$, and $z$ denote the Euler angles of the spherical shell relative to a ground-fixed inertial frame, according to the ZYX (or Tait-Bryan) convention. $u_1$ and $\theta_v$ represent the velocity and heading angle of the IDU related to the ground-fixed inertial frame, respectively. Also, $\dot{\theta}_v=u_2$, with $u_2$ indicating the IDU's yaw rate. $R$ denotes the spherical shell's radius.

    Following \cite{zhan2011design}, we assume the projection of geometrical center of the IDU on the ground is always coincident with the contact point of the robot with the groud. To describe the pose kinematics of the shell relative to the IDU, we define a new variable $\Bar{z}=z-\theta_v$. It can be observed that $x,y$, and $\Bar{z}$ represent the ZYX Euler angles of the shell relative to a body frame attached to the IDU. By scaling $u_1$ with $v=\frac{u_1}{R}$ and defining $u_g=-u_2$, it is obtained that
\begin{equation}
    \left\{
    \begin{array}{l}
        \dot{x} = v \frac{\sin(\Bar{z})}{\cos y} \\
    \dot{y} = v \cos(\Bar{z}) \\
    \dot{\Bar{z}} = v \tan y \sin(\Bar{z})+u_g. 
    \end{array}
    \right.
    \end{equation}
    Referring to \cite{kumar2024equivalence}, the above equation is equivalent to the model of a vehicle traveling on a unit sphere using geographic coordinates, and its equivalence to the model in $SO(3)$ is shown in \cite{kumar2024equivalence} as well.

After illustrating real-world systems with the same kinematics as the spherical CRS model, we next apply PMP to characterize the optimal control actions for the time-optimal spherical CRS problem.

\subsection{Applying PMP to derive the optimal control actions}
We now apply PMP to the time-optimal problem for~\eqref{RS model}. Similar to \cite{monroy1998non}, PMP is applied for the symplectic formalism \cite{jurdjevic1997geometric}, wherein PMP is applied on the cotangent bundle of the Lie group, $T^*SO(3)$. For details of definitions and derivations, see Appendix \ref{appd PMP} in the supplementary materials.

The Hamiltonian $H_{u_g (t), v (t)}$ can be written as 
\begin{align}
\begin{split}
\label{Halmil0}
    &H_{u_g (t), v (t)} \left(\zeta_0, \zeta \right) = \zeta_0 + v (t) h_1 (\zeta (t)) -  u_g (t) H_{12} (\zeta (t)),
\end{split}
\end{align}
where $\zeta_0\le 0$ is a constant and $\zeta(t)\in T^*SO(3)$ is the integral curve. 
Here $h_1$ and $H_{12}$ are the Hamiltonians associated with control vector fields.

\noindent
\textbf{From PMP, the following hold for an optimal path:}

1) $\zeta_0\leq0$ and is constant in $t$.

2) If $\zeta_0=0$, then $h_1$, $h_2$, and $H_{12}$ cannot all be identically zero.

3) $v$ and $u_g$ pointwise maximize $H_{u_g (t), v (t)}$.

4) $H_{u_g (t), v (t)}\equiv0$.
\vspace{0.5em}

Moreover, the evolution of the Hamiltonians from PMP are given by \cite{monroy1998non} 
\begin{align}
\label{Ham1}
    \frac{d{h}_1}{dt} &= u_g h_2, 
    \\
    \label{Ham2}
    \frac{d{h}_2}{dt} &= -v H_{12} - u_g h_1, \\
    \label{Ham3}
    \frac{d{H}_{12}}{dt} &= v h_2.
\end{align}
Here, $h_2$ is the Hamiltonian associated with a control vector field. The detailed derivation is given in Appendix~\ref{appd PMP}.






Noting that $\zeta_0=0$ corresponds to the abnormal case, we have the following lemma:

\begin{lem}
\label{lem1}
No nontrivial abnormal extremal path exists.
\end{lem}
\begin{proof}
See Appendix \ref{appd lem1} in the supplementary materials.
\end{proof}

\begin{rem}
\label{remark e=1}
    Since  $\zeta_0\leq0$, and $\zeta_0\neq0$ by Lemma \ref{lem1}, hence, $\zeta_0<0$. Furthermore, $\zeta_0<0$ can be normalized to $\zeta_0=-1$ without loss of generality \cite{monroy1998non}. Therefore, hereafter, for a nontrivial extremal path, we treat $\zeta_0$ as $-1$.
\end{rem}

Thus far, the Hamiltonian can be further simplified with Remark \ref{remark e=1}. Furthermore, for the simplicity of symbols and the uniformity of signs of the control variables, we define:
\begin{equation}
    \mathcal{A}:=H_{12}, ~\mathcal{B}:=h_2, ~\mathcal{C}:=-h_1,
\end{equation}
accordingly, by \eqref{Ham1}-\eqref{Ham3},
\begin{equation}
    \label{d_switch}
    \frac{d\mathcal{A}}{dt}=v\mathcal{B}, ~~~~\frac{d\mathcal{B}}{dt}=-v\mathcal{A}+u_g\mathcal{C}, ~~~~\frac{d\mathcal{C}}{dt}=-u_g\mathcal{B}.
\end{equation}

\begin{rem}
\label{rem:Euler-Rodriguez}
The above set of equations can be compactly expressed as 
$$ \frac{d}{dt} \begin{bmatrix} {\mathcal A}(t) \\ {\mathcal B}(t) \\ {\mathcal C}(t) \end{bmatrix} = -\Omega(t) \begin{bmatrix} {\mathcal A} \\ {\mathcal B} \\ {\mathcal C} \end{bmatrix}. $$
When $(v,u_g) \ne (0,0)$ and is constant, one can define $\hat \Omega (t) := \frac{1}{\sqrt{v^2+u_g^2}}\Omega(t)$ and $\phi:= \sqrt{v^2+u_g^2}t$. Since $\hat{\Omega}$ is skew-symmetric, we have $e^{-\phi\hat{\Omega}}=e^{\phi\hat{\Omega}^T}=(e^{\phi\hat{\Omega}})^T$, one may express the solution as:
\begin{align}
\label{eq:A_B_C_sol}
    \begin{bmatrix} {\mathcal A}(t)\\ {\mathcal B}(t)\\ {\mathcal C}(t) \end{bmatrix} &= e^{-\phi\hat{\Omega}}
    \begin{bmatrix} {\mathcal A}(0)\\ {\mathcal B}(0)\\ {\mathcal C}(0) \end{bmatrix}
    ={\mathbf M^T(\cdot)}
    \begin{bmatrix} {\mathcal A}(0)\\ {\mathcal B}(0)\\ {\mathcal C}(0) \end{bmatrix},
\end{align}
where 
\begin{equation}
\label{eq: M_definition}
\mathbf{M}(\cdot):=e^{\phi\hat{\Omega}},
\end{equation}
and its specific calculation will be given in Subsection \ref{sec geometric meaning}
\end{rem}

Therefore, the Hamiltonian in (\ref{Halmil0}) simplifies to:

\begin{equation}
\label{Halmil}
    H_{u_g,v}=-1-v\mathcal{C}-u_g\mathcal{A},
\end{equation}

From PMP, $v$ and $u_g$ pointwise maximize $H_{u_g,v}$, therefore:
\begin{align}
    \label{bangbang1}
    v=\begin{cases} 
    -1 &\text{if } \mathcal{C}>0\\
    1 &\text{if } \mathcal{C}<0
    \end{cases}, ~~~ 
    u_g=\begin{cases} 
    -U_{max} &\text{if } \mathcal{A}>0\\
    U_{max} &\text{if } \mathcal{A}<0
    \end{cases}.
\end{align}

By (\ref{bangbang1}), it is clear that $v$ and $u_g$  switch between extreme values when $\mathcal{C}$ and $\mathcal{A}$ change sign, respectively. Next, we analyze the cases where $\mathcal{C}=0$ or $\mathcal{A}=0$.



\begin{lem}
    \label{Lem axis point}
    On a nontrivial extremal path, the following hold:
    
     (i) ${\mathcal A}$ and ${\mathcal C}$ do not take a value of $0$ simultaneously. More specifically, (a)  $|\mathcal{C}|=1$ if $\mathcal{A}=0$; (b)  $|\mathcal{A}|=\frac{1}{U_{max}}$ if $\mathcal{C}=0$. 
        
        (ii) $u_g\equiv0$ iff $\mathcal{A}\equiv0$.
        
        (iii) $v\equiv0$ iff $\mathcal{C}\equiv0$.
        
\end{lem}
\begin{proof}
We prove the statements separately.

\vspace{6pt}
     \noindent(i) $|\mathcal{C}|=1$ if $\mathcal{A}=0$; and $|\mathcal{A}|=\frac{1}{U_{max}}$ if $\mathcal{C}=0$.
     
     By Condition 4) of PMP and (\ref{Halmil}), $-1+|\mathcal{C}|=0$ if $\mathcal{A}=0$ and $-1+U_{max}|\mathcal{A}|=0$ if $\mathcal{C}=0$. Hence, $|\mathcal{C}|=1$ if $\mathcal{A}=0$ and $|\mathcal{A}|=\frac{1}{U_{max}}$ if $\mathcal{C}=0$.
     
     \vspace{6pt}
     \noindent(ii) $u_g\equiv0$ iff $\mathcal{A}\equiv0$.
     
     (Sufficiency) If $\mathcal{A}\equiv0$, then $|\mathcal{C}|\equiv1$ by (i). Furthermore, by (\ref{d_switch}), $\frac{d\mathcal{A}}{dt}=v\mathcal{B}\equiv0$, which implies $\mathcal{B}\equiv0$ since $v\neq0$ by \eqref{bangbang1}. Hence, $\frac{d\mathcal{B}}{dt}=-v\mathcal{A}+u_g\mathcal{C}\equiv0$. Therefore, $u_g\equiv0$. 
     
     (Necessity) If $u_g\equiv0$, then $\mathcal{A}\equiv0$ since by (\ref{bangbang1}), $u_g\neq0$ if $\mathcal{A}\neq0$. 

      \vspace{6pt}
     \noindent(iii) $v\equiv0$ iff $\mathcal{C}\equiv0$.
     
     (Sufficiency) If $\mathcal{C}\equiv0$, then $|\mathcal{A}|\equiv\frac{1}{U_{max}}$  by (i), and $u_g\neq0$ by (\ref{bangbang1}). Furthermore, by (\ref{d_switch}), $\frac{d\mathcal{C}}{dt}=-u_g\mathcal{B}\equiv0$, which implies $\mathcal{B}\equiv0$. Hence, $\frac{d\mathcal{B}}{dt}=-v\mathcal{A}+u_g\mathcal{C}\equiv0$, which implies that $v\equiv0$. 
     
     (Necessity) If $v\equiv0$, then $\mathcal{C}\equiv0$ since by (\ref{bangbang1}), $v\neq0$ if $\mathcal{C}\neq0$. 
\end{proof}

Therefore, the optimal control actions are given by
\begin{align}
    \label{bangbang2}
    v=\begin{cases} 
    -1 &\text{if } \mathcal{C}>0\\
    0 &\text{if } \mathcal{C}\equiv0\\
    1 &\text{if } \mathcal{C}<0
    \end{cases}, ~~~ 
    u_g=\begin{cases} 
    -U_{max} &\text{if } \mathcal{A}>0\\
    0 &\text{if } \mathcal{A}\equiv0\\
    U_{max} &\text{if } \mathcal{A}<0
    \end{cases}.
\end{align}

\subsection{Geometric meaning of the optimal control actions}
\label{sec geometric meaning}
Because the optimal controls are piecewise constant and take only extreme/zero values, the optimal path is a concatenation of a small set of motion primitives, which result from different combinations of the values of $v$ and $u_g$. 

For constant $v$ and $u_g$, by differentiating (\ref{RS model2}), substituting (\ref{RS model1}) and (\ref{RS model3}), and simplifying terms, we obtain the second derivative of $\bold{T_v}(t)$:
\begin{equation}
\label{Eq. ddTv}
    \frac{d^2\bold{T_v}}{dt^2}=-(v^2(t)+u_g^2(t))\bold{T_v}(t),
\end{equation}
which can be seen as a spring-mass system. When $v=V$ and $u_g=U$ are constants, the solution of (\ref{Eq. ddTv}) is a periodic function $\bold{T_v}(t)=\mathbf {T_v}(0)\cos(\omega t)
+
\frac{\dot{\mathbf T}_\mathbf{v}(0)}{\omega}\sin(\omega t),
\qquad
\omega=\sqrt{V^2+U^2}.$ with angular frequency $\omega=\sqrt{V^2+U^2}$. Hence, from (\ref{RS model1}), $\bold{X_v}(t)=\bold{X_v}(0)+
\frac{V}{\omega}\,\mathbf{T_v}(0)\sin(\omega t)
+
\frac{V}{\omega^2}\,\dot{\mathbf T}_\mathbf{v}(0)\bigl(1-\cos(\omega t)\bigr)$ is also periodic with angular frequency $\omega$. When $|v|=1$, the radius of the periodic motion equals $\frac{|v|}{\omega}=\frac{|V|}{\sqrt{V^2+U^2}}=\frac{1}{\sqrt{1+U^2}}$, therefore, defining $r:=\frac{1}{\sqrt{1+U_{max}^2}}$, we have
\begin{itemize}
    \item $u_g=\pm U_{max}$, $v=\pm1$ corresponds to a tight turn with radius $r$, denoted as a ``$C$" segment.
    \item $u_g=0$, $v=\pm1$ corresponds to a great circular arc with radius 1, denoted as a ``$G$" segment. 
    \item $u_g=\pm U_{max}$, $v=0$ corresponds to a turn-in-place motion with a turning rate of $\sqrt{0 + U_{max}^2} = U_{max}$, denoted as a ``$T$" segment.
\end{itemize}



    If specific actions need to be declared, we will use ``$L$" and ``$R$" to denote $u_g=U_{max}$ and $u_g=-U_{max}$, respectively, and superscripts ``$+$", ``$0$", and ``$-$" to represent $v=1$, $v=0$ and $v=-1$, respectively. All possible segment types are visualized in Fig. \ref{fig:segments}. Furthermore, we will use subscripts to denote the arc angle of a segment when they need to be specified. For instance, ``$L^+_\alpha G^+_\beta R^-_\gamma R^0_\sigma$" represents a path that is a concatenation of a tight turn segment (with $u_g=U_{max}$, $v=1$) with an angle of $\alpha$; an arc segment of a great circle (with $u_g=0$, $v=1$) with an angle of $\beta$; a tight turn segment (with $u_g=-U_{max}$, $v=-1$) with an angle of $\gamma$; and a segment of a turn-in-place motion (with $u_g=-U_{max}$, $v=0$) with an angle of $\sigma$.
    \begin{figure}[h]
    \centering
    \subfigure[$L^+$, $R^+$ and $G^+$]{\includegraphics[width=0.24\textwidth]{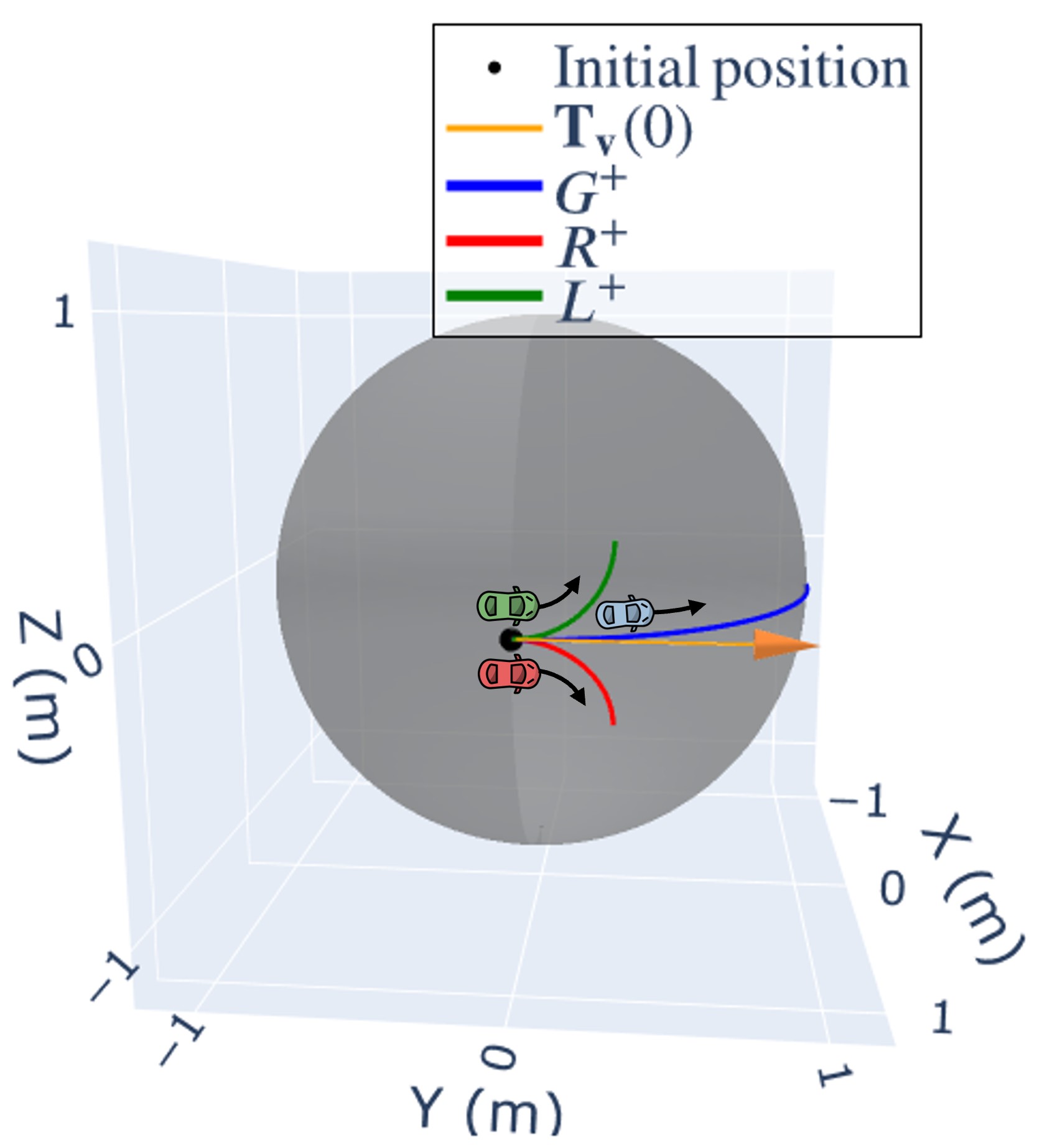}}
    \subfigure[$L^-$, $R^-$ and $G^-$]{\includegraphics[width=0.24\textwidth]{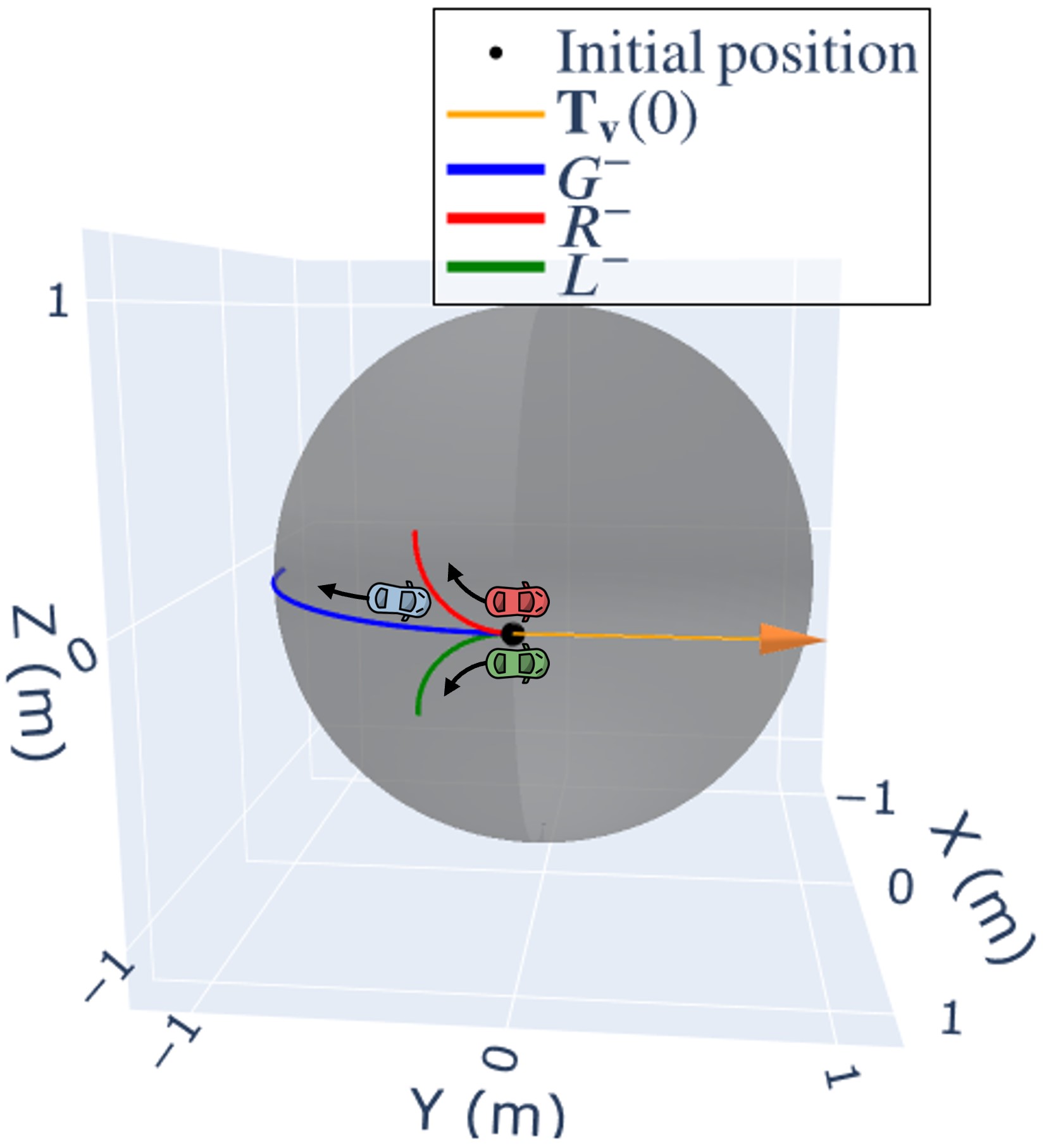}}
    \centering
    \subfigure[$R^0$]{\includegraphics[width=0.24\textwidth]{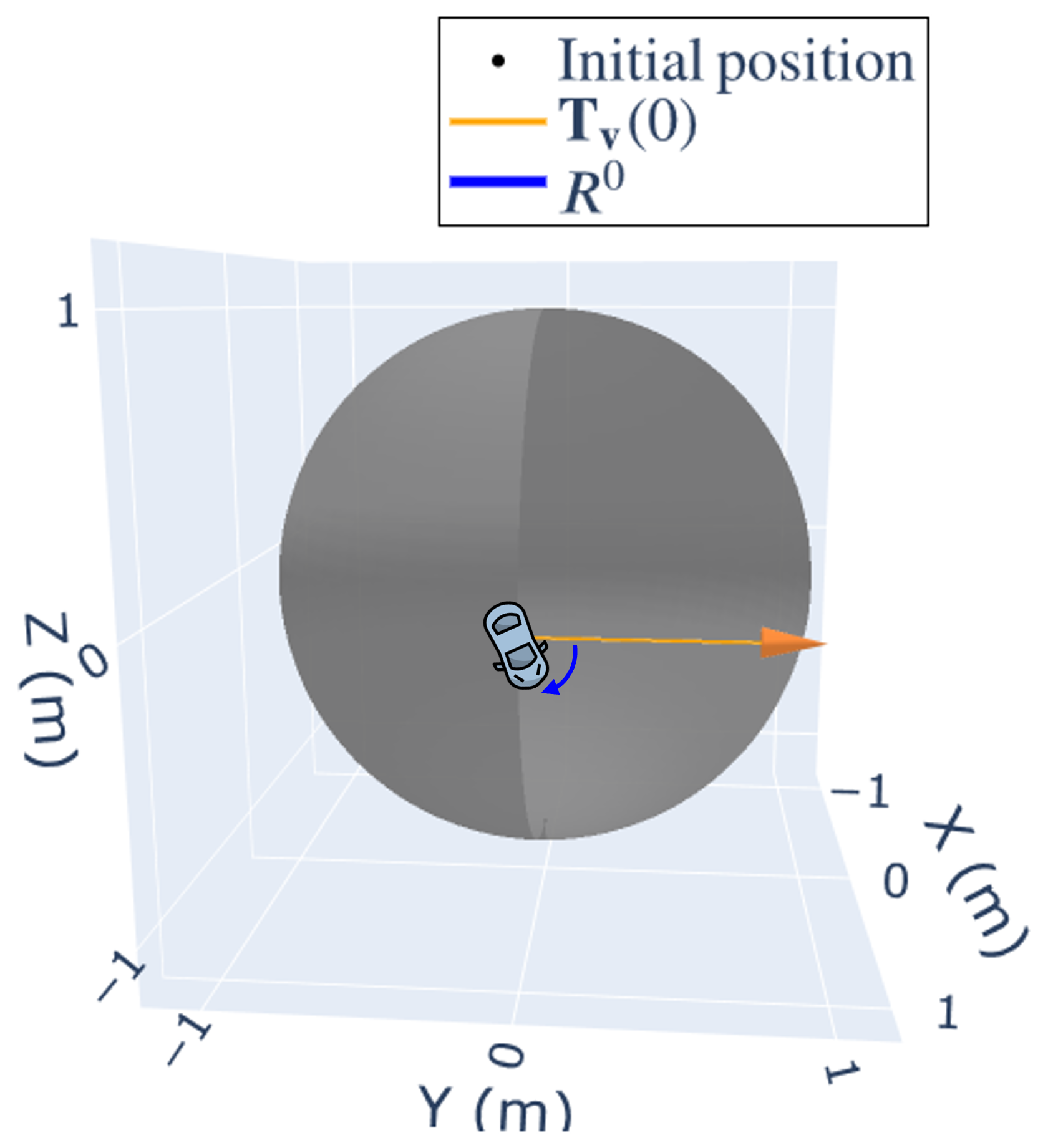}}
    \centering
    \subfigure[$L^0$]{\includegraphics[width=0.24\textwidth]{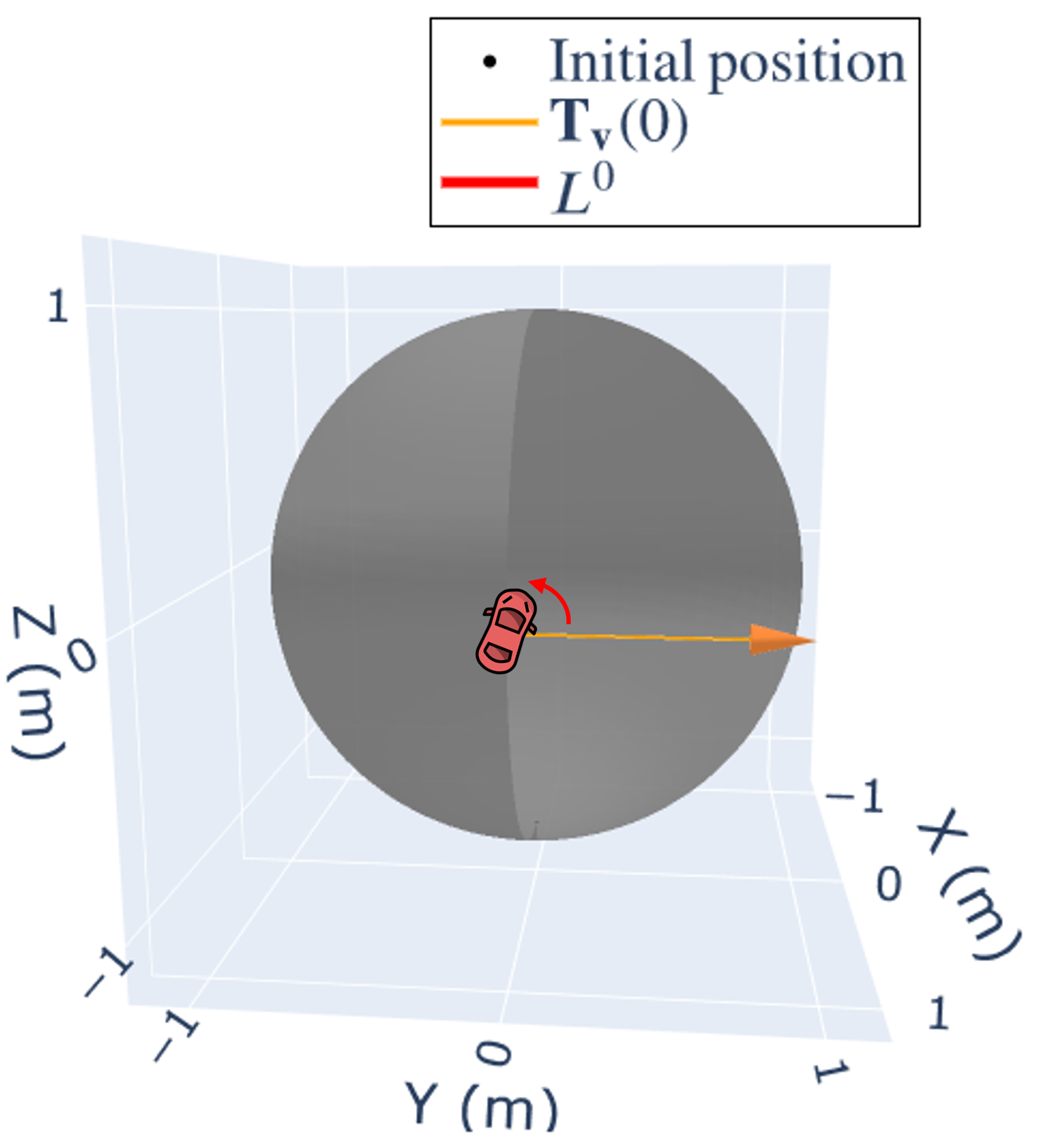}}
    \caption{Possible segment types on an extremal path}
    \label{fig:segments}
    \end{figure}

The above notation provides a compact way to describe any optimal path as a concatenation of the segments. We now leverage this piecewise-constant structure to obtain explicit closed-form state transitions for each segment type

    Recall the differential equation in (\ref{RS model}), since $v$ and $u_g$ remain constant on each segment, its solution on each segment is
\begin{equation}
\label{R sol}
    \bold{R}(t)=\bold{R}(0)\mathbf{M}(\cdot),
\end{equation}
where $\mathbf M(\cdot)$ is as defined in \eqref{eq: M_definition} and can be calculated using the Euler-Rodriguez formula \cite{darbha2023optimal}:
\begin{equation}
\label{eq: Euler-Rodriguez}
    \mathbf M(\cdot)=[I + \hat \Omega \sin \phi + \hat \Omega^2(1-\cos\phi) ]^T.
\end{equation}
Here, $\hat{\Omega}$ contains $u_g$ and $v$, and $\phi$ corresponds to the angle of the segment. We substitute specific values of $v$ and $u_g$ for each type of segment and calculate the corresponding $\bold{M}(\cdot)$ using \eqref{eq: Euler-Rodriguez}. Note that here and throughout the paper, we use the shorthand notations $c(\cdot):=\cos(\cdot)$, and $s(\cdot):=\sin(\cdot)$. The following is obtained
\begin{align}
\label{M_G+}
    \bold{M}_{G^+}(\phi) &= \left(\begin{array}{ccc}
c(\phi) & -s(\phi) & 0 \\
s(\phi) & c(\phi) & 0 \\
0 & 0 & 1 \\
\end{array} \right),\\
\label{M_L+}
    \bold{M}_{L^+}(r, \phi) &= \left( \begin{array}{ccc}
\eta_{11} & -rs(\phi) & \eta_{13} \\
rs(\phi) & c(\phi) & -\eta_{23} \\
\eta_{13} & \eta_{23} & \eta_{33} \\
\end{array} \right),\\
\label{M_R+}
    \bold{M}_{R^+}(r, \phi) &= \left( \begin{array}{ccc}
\eta_{11} & -rs(\phi) & -\eta_{13} \\
rs(\phi) & c(\phi) & \eta_{23} \\
-\eta_{13} & -\eta_{23} & \eta_{33} \\
\end{array} \right),\\
\label{M_L0}
    \bold{M}_{L^0}(\phi)& = \begin{pmatrix}
1 & 0 & 0 \\
0 & c(\phi) & -s(\phi) \\
0 & s(\phi) & c(\phi)
\end{pmatrix},\\
\label{M_G-}
    \bold{M}_{G^-}(\phi) &= \bold{M}^T_{G^+}(\phi),\\
\label{M_L-}
    \bold{M}_{L^-}(r, \phi)& = \bold{M}^T_{R^+}(r, \phi),\\
\label{M_R-}
    \bold{M}_{R^-}(r, \phi)& = \bold{M}^T_{L^+}(r, \phi),\\
\label{M_R0}
    \bold{M}_{R^0}(\phi)& = \bold{M}^T_{L^0}(\phi),
\end{align}
where $\eta_{11} = 1 - (1 - c(\phi))r^2, \quad \eta_{13} = (1 - c(\phi))r\sqrt{1 - r^2}, \eta_{23} = s(\phi)\sqrt{1 - r^2}, \quad \eta_{33} = c(\phi) + (1 - c(\phi))r^2.$

    Another notation that is important for the later phase portrait analysis is a cusp, which corresponds to a change in the value of $v$ between $1$ and $-1$, denoted as ``$|$". Fig. \ref{fig:cusp} illustrates a cusp on a $C|C$ path. For instance, ``$GC|CT$" represents a path starting as a segment of a great circle arc, followed by two tight turns joined by a cusp, and concluding with a turn-in-place motion.

    \begin{figure}[h]
    \centering
    \setlength{\abovecaptionskip}{0pt}
    \includegraphics[width=0.24\textwidth]{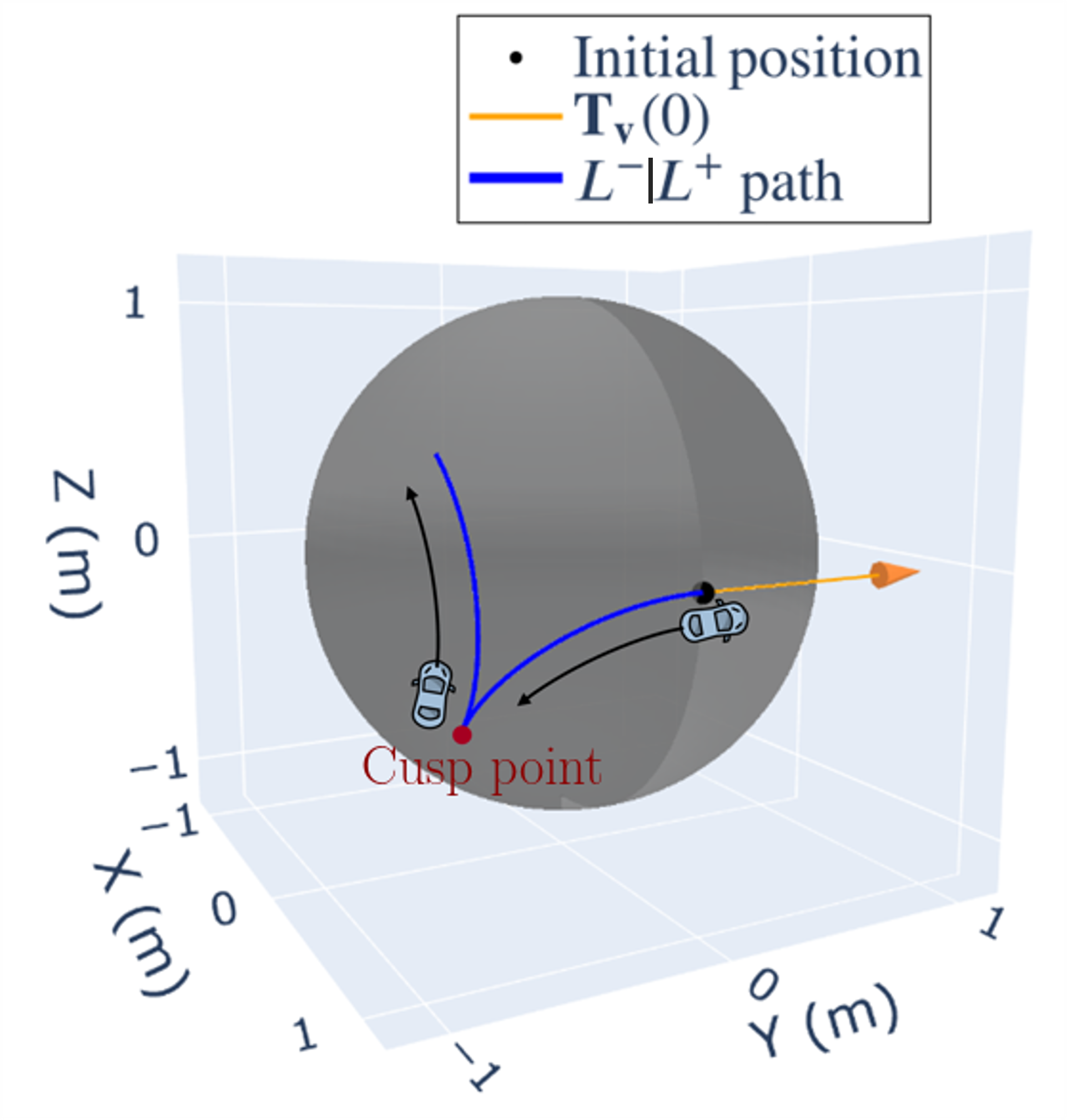}
    \caption{Cusp on a $C|C$ path}
    \label{fig:cusp}
\end{figure}

\section{Brief Analysis of the Optimal Paths \label{sec3}}
In this section, we adopt a phase portrait based approach to characterize optimal path types.

Since the control actions $v$ and $u_g$ depend on $\mathcal{C}$ and $\mathcal{A}$, respectively, the optimal path can be characterized by the evolution of $\mathcal{C}$ and $\mathcal{A}$. By (\ref{Halmil}) and further with $H_{u_g,v}\equiv0$ (from PMP), a clear $\mathcal{A}$-$\mathcal{C}$ relation is obtained:
\begin{equation}
\label{A-C relation}
    1+v\mathcal{C}+u_g\mathcal{A}\equiv0.
\end{equation}
Together with (\ref{bangbang2}), the shape of the $\mathcal{A}-\mathcal{C}$ phase portrait can be shown to form a diamond-like figure as in Fig. \ref{fig:A-C};  the evolution direction of $\mathcal{A}$ and $\mathcal{C}$ on the phase portrait needs further analysis.

\begin{figure}[h]
    \centering
    \setlength{\abovecaptionskip}{0pt}
    \includegraphics[width=0.26\textwidth]{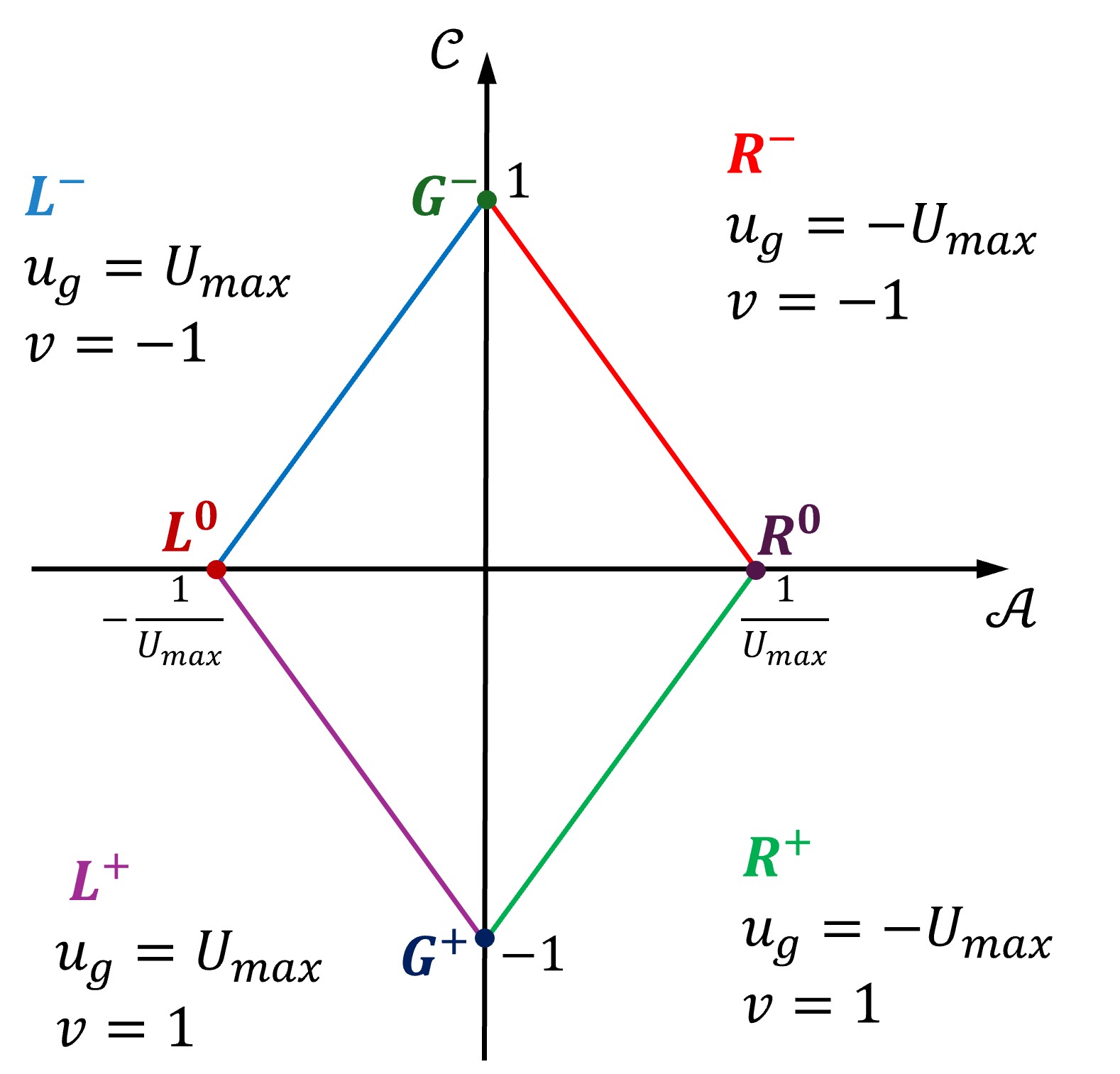}
    \caption{Phase portrait of $\mathcal{A}-\mathcal{C}$}
    \label{fig:A-C}
\end{figure}

By Lemma \ref{Lem axis point}, $\mathcal{A}$ and $\mathcal{C}$ cannot have any common zero. This property is also illustrated in Fig.~\ref{fig:A-C}, where the phase portrait does not pass through the origin. Therefore, by (\ref{bangbang2}), $u_g$ and $v$ cannot switch values at the same time instance.


 A cusp can only occur between two $C$ segments with the value of $u_g$ unchanged. By (\ref{bangbang2}), the points $(-\frac{1}{U_{max}},0)$ and $(\frac{1}{U_{max}},0)$ in the phase portrait correspond to either cusps or $T$ segments, and $(0,-1)$ and $(0,1)$ correspond to either inflection points or $G$ segments. 

By further inspecting Fig. \ref{fig:A-C}, the possible segment transitions on an optimal path are as follows: 
\begin{itemize}
    \item [(i)] between $C$ and $G$ segments ($L^-/R^-\leftrightarrow G^-$ and $L^+/R^+\leftrightarrow G^+$);
    \item[(ii)] between two $C$ segments joined by an inflection point ($L^-\leftrightarrow R^-$ and $L^+\leftrightarrow R^+$);
    \item[(iii)]  between $C$ and $T$ segments ($L^-/L^+\leftrightarrow L^0$ and $R^-/R^+\leftrightarrow R^0$); and 
    \item[(iv)]between two $C$ segments joined by a cusp ($L^-\leftrightarrow L^+$ and $R^-\leftrightarrow R^+$).
\end{itemize}

\begin{lem}
\label{clock/counter B}
    A path evolves clockwise in the $\mathcal{A}-\mathcal{C}$ portrait if $\mathcal{B}<0$, and counter-clockwise if $\mathcal{B}>0$.
\end{lem}
\begin{proof}
    Recall from (\ref{d_switch}) the relationships between $\mathcal{A}$, $\mathcal{B}$, and $\mathcal{C}$. Suppose $\mathcal{B}<0$. Since $u_g=-U_{max}$ within the right half plane of Fig. \ref{fig:A-C}, by (\ref{d_switch}), $\frac{d\mathcal{C}}{dt}=-u_g\mathcal{B}<0$ within the right half plane. Similarly, it can be obtained that $\frac{d\mathcal{C}}{dt}>0$ within the left half plane. At point $(0,1)$, since $v=-1$, by (\ref{d_switch}), $\frac{d\mathcal{A}}{dt}=v\mathcal{B}>0$. Similarly, it can be obtained that $\frac{d\mathcal{A}}{dt}<0$ at $(0,-1)$. Therefore, the path evolves clockwise.

    With a similar proof for $\mathcal{B}>0$, the lemma is proved.
\end{proof}


%

It is also important to note that, by (\ref{d_switch}), $\frac{d(\mathcal{A}^2(t)+\mathcal{B}^2(t)+\mathcal{C}^2(t))}{dt}=2\big(\mathcal{A}(t)\frac{d\mathcal{A}(t)}{dt}+\mathcal{B}(t)\frac{d\mathcal{B}(t)}{dt}+\mathcal{C}(t)\frac{d\mathcal{C}(t)}{dt}\big)\equiv 0$. Hence, we define
\begin{equation}
\label{ABC constant}
    g:=\mathcal{A}^2(t)+\mathcal{B}^2(t)+\mathcal{C}^2(t),
\end{equation}
which is a non-negative constant.



From the phase portrait in Fig.~\ref{fig:A-C}, the role of the parameter $g$ can be understood as follows. Since $\mathcal{A}^2(t)+\mathcal{B}^2(t)+\mathcal{C}^2(t)\equiv g$ is invariant, the evolution of $(\mathcal{A},\mathcal{C})$ in the phase portrait is constrained by $\mathcal{A}^2(t)+\mathcal{C}^2(t)\le g$. When $g<1$, this implies that the points $(\mathcal{A},\mathcal{C})=(0,\pm1)$ cannot be reached, and consequently the evolution in the phase portrait is confined to either the left or the right half plane. In contrast, when $g\ge 1$, these boundary points become reachable, allowing transitions between the two half planes. Therefore, the value $g=1$ serves as a critical threshold, leading to different admissible switching structures. This motivates the classification of time-optimal paths into three cases: (1) $g=1$; (2) $g>1$; (3) $g<1$. These three cases will be analyzed in detail in Sections \ref{sec 3.1}, \ref{sec3.2}, and \ref{sec3.3}, respectively. We provide an overview of the connections between the upcoming lemmas, corollaries, and propositions with Theorem \ref{thm1}, in Fig. \ref{fig:g=1 diagram_all}, along with summaries of the key results.
\begin{figure}[h]
    \centering
    \setlength{\abovecaptionskip}{0pt}
    \includegraphics[width=0.48\textwidth]{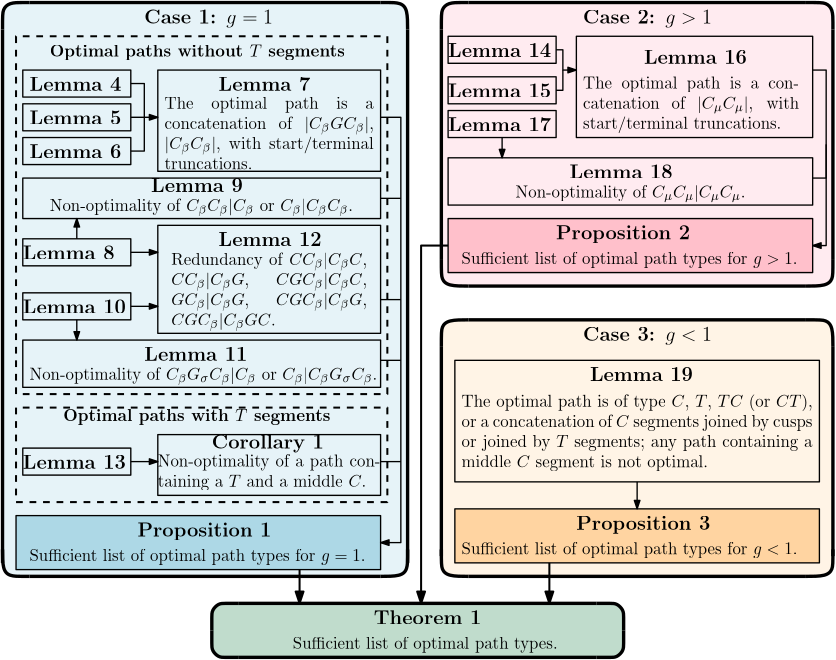}
    \caption{Connections between  proofs}
    \label{fig:g=1 diagram_all}
\end{figure}

Moreover, the analysis throughout the paper relies heavily on path substitution arguments to eliminate candidate paths from the sufficient list of time-optimal path types. A candidate path is eliminated if it is shown to be either non-optimal or redundant\footnote{If a path is redundant, it means that there is always an alternative path of the same cost, belonging to the sufficient list, with fewer segments.}. Non-optimality is established by showing that either (a) there exists an alternative path with strictly shorter travel time, or (b) there exists an alternative path with equal travel time that is itself known to be non-optimal, implying that the candidate cannot be optimal. Redundancy is established by showing that an alternative path with equal travel time but fewer segments exists. The path substitution arguments used throughout the paper are summarized in Table~\ref{tab:substitutions} to provide a concise reference for the eliminations. Recall the value of angles are given in Table \ref{tab: key-symbols}.

\begin{table}[t]
\centering
\caption{Summary of path substitutions used to eliminate non-optimal or redundant path candidates.}
\label{tab:substitutions}
\setlength{\tabcolsep}{5.7pt}
\renewcommand{\arraystretch}{1.05}
\begin{tabular}{c p{0.26\linewidth} p{0.29\linewidth} c}
\hline
\textbf{Lemma} & \textbf{Eliminated path} & \textbf{Alternative path} & \textbf{Comparison} \\
\hline
\ref{GC|C|C non optimal} & $C_{\pi+\alpha}$ & $C_{\pi-\alpha}$ & shorter \\[2pt]
\ref{CC|C non optimal} & $C_\beta C_\beta|C_\beta$ & $ C_{2\beta}|C_\beta$ & equal  \\[2pt]
\ref{lem CGCC non optimal} & $C_\beta G_\sigma C_\beta|C_\beta$ & $C_{2\beta}|C_\beta G_\sigma$ & equal  \\[2pt]
\ref{lem redundent} & $CC_\beta|C_\beta C$ & $C|C$ & equal \\[2pt]
\ref{lem redundent} & $CC_\beta|C_\beta G$ & $C|C_\beta G$  & equal \\[2pt]
\ref{lem redundent} & $CGC_\beta|C_\beta C$ & $CGC_\beta|C$ & equal \\[2pt]
\ref{lem redundent} & $GC_\beta|C_\beta G$ & $C|C_\beta G$ & equal \\[2pt]
\ref{lem redundent} & $CGC_\beta|C_\beta G$ & $C|C_\beta G$ & equal \\[2pt]
\ref{lem redundent} & $CGC_\beta|C_\beta GC$ & $CGC_\beta|C$ & equal \\[2pt]
\ref{g=1 TC non opt} & $T_\rho C_\pi$ ($U_{max}=1)$ & $C_\pi G_\rho$ & equal  \\[2pt]
\ref{lem g>1 CC|C nonoptimal} & $C_\mu C_\mu| C_\mu C_\mu$ & $C_{\epsilon-\mu}C_\mu|C_\mu C_{\epsilon-\mu}$ or $C_{\mu-\epsilon}|C_\mu|C_\mu|C_{\mu-\epsilon}$ & shorter \\[2pt]
\hline
\end{tabular}
\end{table}

\subsection{Overview of the case $g=1$}
We first outline the proof of the main result for the case $g=1$ by explicitly listing below the key results on which it is built; the proofs of these results will follow in detail in the following section.


\begin{defn}
   Throughout this paper, we define the angle $\beta=\arctan(\frac{1}{\sqrt{U_{max}^4-1}})+\frac{\pi}{2}$, which will be utilized extensively.
\end{defn}

\begin{defn}
   The starting and terminal segments of any candidate optimal path with at least $n \ge 3$ segments will sometimes be referred to as boundary segments, and the other $n-2$ segments will be referred to as middle segments. For example, in the path $CC_\beta|C$, the middle segment is $C_\beta$ and in  $C|C_\beta GC_\beta|C$, the middle segments are $C_\beta, G$. 
\end{defn}

\noindent\textbf{Lemma \ref{lem G inflection}.} \textit{For $g=1$ and $U_{max}\geq1$ (or $r\leq\frac{1}{\sqrt{2}}$), an optimal path without $T$ segments is a concatenation of $|C_\beta GC_\beta|$ and $|C_\beta C_\beta|$ sub-paths, where the start/terminal sub-paths can be truncated\footnote{The start sub-paths may be truncated to $CGC_\beta|$, $GC_\beta|$, $C|$, or $CC_\beta|$, and the terminal sub-paths may be truncated to $|C_\beta GC$, $|C_\beta G$, $|C$, or $|C_\beta C$.}.}

\noindent\textbf{Lemma \ref{CC|C non optimal}.} \textit{A $C_\beta C_\beta|C_\beta$ (or $C_\beta|C_\beta C_\beta$) path is not optimal.}

\vspace{4pt}
\noindent\textbf{Lemma \ref{lem CGCC non optimal}.} \textit{For $\sigma\geq0$, a path of type $C_\beta G_\sigma C_\beta|C_\beta$ (or $C_\beta|C_\beta G_\sigma C_\beta$) is not optimal.}

\vspace{4pt}
\noindent\textbf{Lemma \ref{lem redundent}.} 
\textit{Paths of type $CC_\beta|C_\beta C$, $CC_\beta|C_\beta G$, 
$CGC_\beta|C_\beta C$, $GC_\beta|C_\beta G$, $CGC_\beta|C_\beta G$, 
and $CGC_\beta|C_\beta GC$ are redundant. In particular, each path can be replaced as follows:
\begin{enumerate}
  \item $CC_\beta|C_\beta C \;\rightarrow\; C|C$
  \item $CC_\beta|C_\beta G \;\rightarrow\; C|C_\beta G$
  \item $CGC_\beta|C_\beta C \;\rightarrow\; CGC_\beta|C$
  \item $GC_\beta|C_\beta G \;\rightarrow\; C|C_\beta G$
  \item $CGC_\beta|C_\beta G \;\rightarrow\; C|C_\beta G$
  \item $CGC_\beta|C_\beta GC \;\rightarrow\; CGC_\beta|C$
\end{enumerate}
}

\vspace{4pt}
\noindent\textbf{Corollary \ref{cor TC}.} \textit{For $g=1$ and $U_{max}\geq1$ (or $r\leq\frac{1}{\sqrt{2}}$), a path containing a $T$ segment and a middle $C$ segment is not optimal.}
\vspace{4pt}

Given these key results, we now prove how the main result for $g=1$, stated in Proposition \ref{prop2}, can be obtained.

\begin{prop}
\label{prop2}
    For $g=1$ and $U_{max}\geq1$ (or $r\leq\frac{1}{\sqrt{2}}$), the optimal path may be restricted to the following types, together with their symmetric forms\footnote{By the term ``symmetric form", we mean the reversal of a string, for example, the symmetric form of the path type $C|C_\beta G$ is $GC_\beta|C$.}: 
    
    $C$,\; $G$,\; $T$,\; $CC$,\; $GC$,\; $C|C$,\; $TC$,
    
    $CC_\beta|C$,\; $CGC$,\; $C|C_\beta G$,\; $CTC$,
    
    $C|C_\beta C_\beta|C$,\; $CGC_\beta|C$,\; 
    $C|C_\beta GC_\beta|C$,  
    
    where $\beta=\arctan(\frac{1}{\sqrt{U_{max}^4-1}})+\frac{\pi}{2}$.
\end{prop}
\begin{proof}
    We divide the set of optimal path types into two disjoint sets:

\noindent\underline{Set 1: Paths without $T$ segments}

    By Lemma \ref{lem G inflection}, any optimal path without $T$ segments must be a concatenation of sub-paths of the form:

    (i) $|C_\beta GC_\beta|$ or (ii) $|C_\beta C_\beta|$,

    \noindent where the start sub-paths may be truncated to be of type $CGC_\beta|$, $GC_\beta|$, $C|$, or $CC_\beta|$; terminal sub-paths may similarly be truncated to be of type $|C_\beta GC,$  $|C_\beta G$, $|C,$ or $|C_\beta C$.

    However:
    \begin{itemize}
        \item  Lemma \ref{CC|C non optimal} states that $C_\beta C_\beta|C_\beta$ (or $C_\beta|C_\beta C_\beta$) is not optimal. This eliminates paths containing two adjacent middle sub-paths with at least one of type (ii), i.e., (i,ii), (ii,i), or (ii,ii). For example, (i,ii) corresponds to a $|C_\beta G C_\beta| C_\beta C_\beta|$ path, which is non-optimal by Lemma~\ref{CC|C non optimal}.
        \item Lemma \ref{lem CGCC non optimal} states that $C_\beta G_\sigma C_\beta|C_\beta$ (or $C_\beta|C_\beta G_\sigma C_\beta$) is not optimal, thereby eliminating paths containing two adjacent middle sub-paths of type (i), i.e., (i,i).
    \end{itemize}

    Thus, it suffices to consider paths with at most one middle sub-path. The paths with one middle sub-path have the form
    
   (a) $P_i|C_\beta GC_\beta|P_f$ or (b)  $P_i|C_\beta C_\beta|P_f$,
   
   \noindent where $P_i \in \{C, GC_\beta, CC_\beta, CGC_\beta\}$ and $P_f \in \{C, C_\beta G, C_\beta C,  C_\beta GC\}$. If $P_i \in \{GC_\beta, CC_\beta, CGC_\beta\}$ or $P_f \in \{C_\beta G, C_\beta C,  C_\beta GC\},$ then forms (a) and (b) contain patterns such as:
   \[
   C_\beta G_\sigma C_\beta|C_\beta \text{~or~} C_\beta| C_\beta G_\sigma C_\beta \text{~or~} C_\beta C_\beta|C_\beta \text{~or~} C_\beta| C_\beta C_\beta, 
   \]
   \noindent which are non-optimal by Lemmas \ref{CC|C non optimal} and \ref{lem CGCC non optimal}. Hence, only $P_i = P_f = C$ remains.

   If $P_i=P_f=C$, then forms (a) and (b) reduce to:
   \[
   C|C_\beta GC_\beta |C \text{~and~} C|C_\beta C_\beta|C, 
   \]
   \noindent which are included in the sufficient list. 
   
   Finally, consider paths with no middle sub-path (i.e., only start and terminal sub-paths). These include the forms: 
   \[
     P_i|P_f \text{~and~} P_s,
     \]
     \noindent where $P_i$ and $P_f$ are as defined earlier, and $P_s\in\{C,G,GC,CG,CC,CGC\}$. 
     
     Applying Lemma \ref{lem redundent}, we replace the following redundant paths:     
     $CC_\beta|C_\beta G$, $CGC_\beta|C_\beta C$, $GC_\beta|C_\beta G$, $CGC_\beta|C_\beta G$, and $CGC_\beta|C_\beta GC.$ 
     
     After these eliminations, the remaining paths in Set 1 are exactly those without $T$ segments listed in Proposition 1.

\vspace{0.5em}
    \noindent \underline{Set 2: Pats containing $T$ segments}

From the phase portrait (Fig. \ref{fig:A-C}), an $L^0$ or $R^0$ segment (a $T$ segment) can only appear between two $C$ segments. Therefore, any optimal path with $T$ segments must have the form:
\[
P_lTP_r, 
\]
\noindent where $P_l$ and $P_r$ are either empty or single $C$ segments. If either side contains more than a $C$ segment, the path includes a middle C segment and is non-optimal by Corollary \ref{cor TC}. 

Therefore, the optimal path types with $T$ segments are: 
\[
\{T, TC, CT, CTC\}.
\]
\vspace{0.5em}
\noindent \underline{Conclusion}

Combining Sets 1 and 2, and including symmetric forms (string reversal), yields the finite list stated in Proposition \ref{prop2}.

\end{proof}

\section{Optimal Paths for $g=1$ \label{sec 3.1}}
This section offers the detailed proofs of the key results for $g=1$, including Lemmas \ref{lem G inflection}, \ref{CC|C non optimal}, \ref{lem CGCC non optimal}, \ref{lem redundent}, and Corollary \ref{cor TC}, which are utilized for Proposition \ref{prop2}, together with their supporting lemmas illustrated in Fig. \ref{fig:g=1 diagram_all}.

Notice that $\tan^{-1}(\frac{y}{x})$ is not defined at $y=\pm 1~\&~ x=0$. In this article, we define $\arctan(\cdot)$ as
\begin{equation}
    \arctan(\frac{y}{x})=\begin{cases}
        \frac{\pi}{2} &~~\text{if} ~y=1~\&~x=0\\
        -\frac{\pi}{2} &~~\text{if} ~y=-1~\&~x=0\\
        \tan^{-1}(\frac{y}{x})&~~\text{otherwise}
    \end{cases}.
\end{equation}

The following Lemmas \ref{lem G/inflection C angle}-\ref{GC|C|C non optimal} will be useful throughout this section.
\begin{lem}
\label{lem G/inflection C angle}
    For $g=1$, the angle of a $C$ segment that is completely traversed\footnote{A $C$ segment is completely traversed if $|\mathcal{C}|$ monotonically increases from 0 to 1 or monotonically decreases from 1 to 0 along the $C$ segment.} clockwise/counter-clockwise is exactly $\arctan(\frac{1}{\sqrt{U_{max}^4-1}})+\frac{\pi}{2}$ radians.
\end{lem}
\begin{proof}
    Consider a $R^-$ segment that is completely traversed clockwise in the phase portrait. Let ${t}_1$ denote the time spent on the $R^-$ segment. From Fig. \ref{fig:A-C}, on the $R^-$ segment, 
    \begin{equation}
    \mathcal{A}(0)=0,~  \mathcal{C}(0)=1,~  \mathcal{A}({t}_1)=\frac{1}{U_{max}},~  \mathcal{C}({t}_1)=0. 
    \end{equation}
    
    By (\ref{ABC constant}), $\mathcal{B}(0)=\pm\sqrt{g-1}$ and $\mathcal{B}({t}_1)=\pm\sqrt{g-\frac{1}{U_{max}^2}}$. Since the segment is traversed clockwise, $\frac{d\mathcal{C}}{dt}<0$ from Fig. \ref{fig:A-C}. Furthermore, since $u_g=-U_{max}$ on a $R^-$ segment, $\mathcal{B}<0$ by (\ref{d_switch}). Hence, 
    \begin{equation}
    \mathcal{B}(0)=-\sqrt{g-1}, ~\mathcal{B}({t}_1)=-\sqrt{g-\frac{1}{U_{max}^2}}. 
    \end{equation}

    Therefore, by \eqref{eq:A_B_C_sol} and noting that the specific $\mathbf{M}(\cdot)$ matrix of a $R^-$ segment is given by \eqref{M_R-}, from the above end-point conditions we have
    \begin{equation}
    \begin{aligned}
        \begin{bmatrix}
            \frac{1}{U_{max}} \\[4pt]
            -\sqrt{g-\frac{1}{U_{max}^2}} \\[4pt]
            0
        \end{bmatrix}=\bold{M}^T_{R^-}(r, \phi)
        \begin{bmatrix}
            0 \\[4pt]
            -\sqrt{g-1} \\[4pt]
            1
        \end{bmatrix}.
        \end{aligned}
    \end{equation}

    Expanding the last two rows of the above equation and rearranging yields:
    \begin{align}
        -\sqrt{g-\frac{1}{U_{max}^2}}
&=
-\sqrt{g-1}\,c(\phi)-\sqrt{1-r^2}s(\phi), \\
-\frac{r^2}{\sqrt{1-r^2}}
&=
-\sqrt{g-1}\,s(\phi)+\sqrt{1-r^2}\,c(\phi).
    \end{align}
    Define $\delta$ by $\cos\delta=\frac{\sqrt{g-1}}{\sqrt{g-r^2}}$ and $\sin\delta=\frac{\sqrt{1-r^2}}{\sqrt{g-r^2}}$, and using trigonometric identities, the above two equations respectively gives
    \begin{align}
\sqrt{g-\frac{1}{U_{max}^2}}
&=\sqrt{g-r^2}\,\cos(\phi-\delta),\\[4pt]
\frac{r^2}{\sqrt{1-r^2}}
&=\sqrt{g-r^2}\,\sin(\phi-\delta),\\
\implies \tan(\phi-\delta)
&=
\frac{r^2}{\sqrt{1-r^2}\,\sqrt{g-\dfrac{1}{U_{max}^2}}}.
\end{align}
Substituting the definitions of $r$ and $\delta$, we obtain
\begin{align}
\label{alpha1 alpha2}
\phi
=&
\arctan\left(
\frac{1}{U_{\max}\sqrt{(1+U_{max}^2)(g-\frac{1}{U_{\max}^2})}}
\right) \notag\\
&+
\arctan\left(
\frac{U_{\max}}{\sqrt{(1+U_{max}^2)(g-1)}}
\right).
\end{align}
When $g=1$, $\phi=\arctan(\frac{1}{\sqrt{U_{max}^4-1}}) +\frac{\pi}{2}$.

Using similar results for a $R^-$ segment completely traversed counter-clockwise, and segments $R^+$, $L^-$, and $L^+$, the lemma is proved.
\end{proof}

According to Lemma \ref{clock/counter B}, the sign of $\mathcal{B}$ reveals the evolution direction of an optimal path on the phase portrait. Consequently, in the subsequent lemma, we examine the locations on the phase portrait where $\mathcal{B}$ may reach $0$, indicating where an optimal path can switch its evolution direction on the portrait.

\begin{lem}
\label{g=1 B=0}
    For $g=1$, 
    \begin{enumerate}
        \item For $U_{max}>1$ (or $r<\frac{1}{\sqrt{2}}$), $\mathcal{B}$ can be equal to 0 only when $\mathcal{A}=0$. 
        \item For $U_{max}=1$ (or $r=\frac{1}{\sqrt{2}}$), $\mathcal{B}$ can be equal to 0 only when $\mathcal{A}=0,\pm \frac{1}{U_{max}}$
    \end{enumerate}
\end{lem}
\begin{proof}

Since $\mathcal{A}^2+\mathcal{B}^2+\mathcal{C}^2 = 1$, 
\[\mathcal{B}=0 \iff \mathcal{A}^2+\mathcal{C}^2 = 1.\] 
Within the left half plane of the phase portrait (Fig. \ref{fig:A-C}), from \eqref{A-C relation}, $\mathcal{C}^2 = (1+U_{max}\mathcal{A})^2$; in other words: 
\begin{align} \label{eq: f(A)_definition}
\mathcal{B} = 0 \iff \underbrace{\mathcal{A}^2 + (1+U_{max}\mathcal{A})^2}_{:=f(\mathcal{A})} = 1.
\end{align}
$f(\mathcal{A})$ is convex and its minimum is
attained between $\mathcal{A}=-\frac{1}{U_{max}}$ and $\mathcal{A} = 0.$
Hence, its maximum over $[-1/U_{max}, 0]$ is the maximum over the two endpoints. Note that in the left half plane of the phase portrait (Fig.~\ref{fig:A-C}), $\mathcal{A} \in [-1/U_{max}, 0];$ hence, it suffices to consider this range for $f(\mathcal{A}).$ Since $0\le \mathcal{B}^2 \le 1$ and $0\le f(\mathcal{A}) \le 1$, $\mathcal{A}=0$ corresponds to a maximizer for $f(\mathcal{A})$ since $f(0)=1$, and $\mathcal{B}=0$ correspondingly. 

In the special case of $U_{max}=1$, $\mathcal{A}=-\frac{1}{U_{max}}$ also corresponds to a maximizer since $f(-\frac{1}{U_{max}})=1$, and $\mathcal{B}=0$ correspondingly.

    With a similar analysis for the right half plane of Fig. \ref{fig:A-C}, the lemma is proved. 
\end{proof}

\begin{rem}
\label{rem3}
    By Lemma~\ref{clock/counter B}, a path in the $\mathcal{A}$--$\mathcal{C}$ portrait evolves clockwise if $\mathcal{B}<0$, and counterclockwise if $\mathcal{B}>0$. Since $\mathcal{B}$ is continuous, Lemma~\ref{g=1 B=0} implies that for $g=1$ and $U_{\max}>1$, the path can only reverse its direction of evolution (i.e., switch between clockwise and counterclockwise motion) at the points $(0, \pm 1)$. At these locations---which correspond to $G$ segments or inflection points---we have $\mathcal{B}=0$, so both
\[
\frac{d\mathcal{A}}{dt}=v\mathcal{B}=0 \qquad \text{and} \qquad \frac{d\mathcal{C}}{dt}=-u_g\mathcal{B}=0.
\]
Therefore, the path may remain stationary at these points in the portrait and enter a $G$ segment.

Similarly, in the special case where $U_{\max}=1$, the path may switch direction at $(0, \pm 1)$ or $(\pm \frac{1}{U_{\max}}, 0)$, and may enter either $G$ or $T$ segments at these points.
\end{rem}

\begin{rem}
\label{T U=1}
    Lemma \ref{g=1 B=0} and Remark \ref{rem3} demonstrate that, for $g=1$, a $T$ segment can appear in an optimal path solely when $U_{max}=1$. This observation simplifies the characterization of optimal path types by dividing them into two groups: those lacking $T$ segments and those containing them. This will be elaborated upon later in this section.
\end{rem}

\begin{rem}
    Please note that, when we focus on paths not containing $T$ segments, all such paths adhere to $|v|=1$, so comparing their lengths is effectively the same as comparing their times.
\end{rem}

The subsequent lemma will be employed frequently throughout this paper.

\begin{lem}
\label{GC|C|C non optimal}
    For $\alpha>0$, a $C_{\pi+\alpha}$ path is not optimal.
\end{lem}
\begin{proof}
   We show that for $\alpha>0$, the terminal configuration produced by $C_{\pi+ \alpha}$ coincides with that of a $C_{\pi- \alpha}$ of opposite direction; since $\pi-\alpha<\pi+\alpha$, the former is non-optimal.


    Consider a $L^-_{\pi+\alpha}$ path, where $\alpha>0$. If $\alpha\geq\pi$, the path is obviously not optimal, as the initial and terminal configurations coincide for a $C$ path with angle of $2\pi$. Otherwise, if $0<\alpha<\pi$, it is claimed that there exists an alternate $R^+_{\pi-\alpha}$ path that is shorter. Since $0<\alpha<\pi$, $\pi-\alpha<\pi+\alpha$, the alternate path is shorter, hence, it suffices to show that $\bold{M}_{L^-}(r,\pi+\alpha)=\bold{M}_{R^+}(r,\pi-\alpha)$. It can be obtained that $\sin(\pi+\alpha)=-\sin(\alpha)$, $\cos(\pi+\alpha)=-\cos(\alpha)$, $\sin(\pi-\alpha)=\sin(\alpha)$, and $\cos(\pi-\alpha)=-\cos(\alpha)$. Substituting these values into (\ref{M_L-}) and (\ref{M_R+}), it is obtained that
    \begin{align}
        \bold{M}_{L^-}(r,\pi+\alpha)&=\bold{M}_{R^+}(r,\pi-\alpha) 
        \notag\\
        &=\begin{pmatrix}
        \xi_1&-r s(\alpha)&\xi_2 \\
        r s(\alpha)&-c(\alpha)&\xi_3 \\
        \xi_2&-\xi_3&\xi_4
        \end{pmatrix},
    \end{align}
    where $\xi_1=(-c(\alpha) - 1)r^2 + 1$, $\xi_2=-r\sqrt{1 - r^2}(c(\alpha) + 1)$, $\xi_3=s(\alpha)\sqrt{1 - r^2}$, $\xi_4=(r^2 - 1)(c(\alpha) + 1) + 1$.

    The non-optimality of $L^-_{\pi+\alpha}$ is visualized in Fig. \ref{fig:L_neg_non} for $\alpha=\frac{\pi}{2}$ and $U_{max}=3$ (or $r=\frac{1}{\sqrt{10}}$).
    \begin{figure}[h]
    \centering
    \setlength{\abovecaptionskip}{0pt}
     \includegraphics[width=0.24\textwidth]{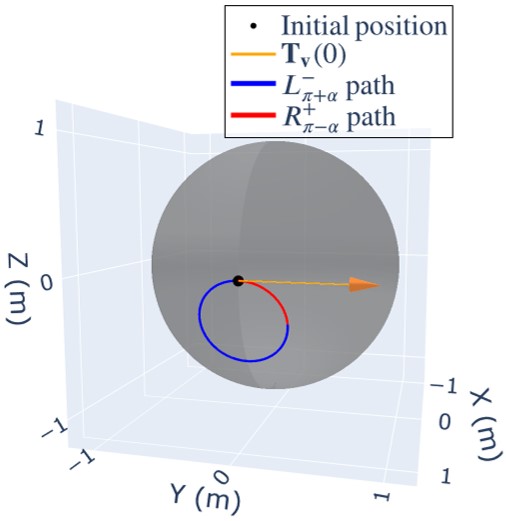}
    \caption{Non-optimality of $L^-_{\pi+\alpha}$}
    \label{fig:L_neg_non}
    \end{figure}

    With similar proofs for path types $L^+_{\pi+\alpha}$, $R^+_{\pi+\alpha}$, and $R^-_{\pi+\alpha}$, the lemma is proved.
\end{proof}

As noted in Remark \ref{T U=1}, when $g=1$, an optimal path can contain $T$ segments only if $U_{max}=1$ (or $r=\frac{1}{\sqrt{2}}$). Therefore, we analyze the case where $g=1$ by distinguishing between paths without $T$ segments and those that contain them.

\subsection{Optimal Paths Without $T$ Segments}

To characterize the optimal paths without $T$ segments, we first utilize Lemmas \ref{lem G/inflection C angle}-\ref{GC|C|C non optimal} to derive the key result Lemma \ref{lem G inflection}.

\begin{lem}
\label{lem G inflection}
    For $g=1$ and $U_{max}\geq1$ (or $r\leq\frac{1}{\sqrt{2}}$), an optimal path without $T$ segments is a concatenation of $|C_\beta GC_\beta|$ and $|C_\beta C_\beta|$ sub-paths, where the start/terminal sub-paths can be truncated.
\end{lem}
\begin{proof}
    The proof first shows that an optimal path cannot contain a middle $C$ segment that reverses its evolution direction (clockwise/counter-clockwise) in the phase portrait. This is because such a reversal forces a double traversal of the same quadrant and yields a $C$ segment with arc angle larger than $\pi$ by Lemma \ref{lem G/inflection C angle}, which is non-optimal by Lemma \ref{GC|C|C non optimal}. 
Having ruled out direction reversals, we can enumerate all possible middle sub-path types and conclude that they are of the form $|C_\beta GC_\beta|$ or $|C_\beta C_\beta|$.


    We now formalize the above argument. As discussed in Remark \ref{rem3}, a path may switch its evolution direction (i.e., clockwise/counter-clockwise) at points $(0,\pm 1)$ and $(\pm \frac{1}{U_{max}},0)$. We first show that for $g=1$ and $U_{max}\geq1$, any path containing a middle $C$ segment with a switched evolution direction (that is, with the sign of $\mathcal{B}$ changed and the portrait switches direction) is not optimal.

      Consider a middle $L^-$ segment with a switched evolution direction, as shown in Fig.~\ref{fig:g=1 twice traversed}. 
The $L^-$ segment can occur in one of the following two cases: 
(1) it is entered at $(0,1)$ from a $G^-$ or $R^-$ segment, traverses the second quadrant counterclockwise, switches evolution direction at $(-\tfrac{1}{U_{\max}},0)$, and then traverses the second quadrant again clockwise before exiting at $(0,1)$ into a $G^-$ or $R^-$ segment; 
or (2) it is entered at $(-\tfrac{1}{U_{\max}},0)$ from a $L^+$ segment, traverses the second quadrant clockwise, switches evolution direction at $(0,1)$, and then traverses the second quadrant again counterclockwise before exiting at $(-\tfrac{1}{U_{\max}},0)$ into a $L^+$ segment. Notice that in either case, the path remains within the second quadrant in the portrait and represents a $L^-$ segment since the signs of $\mathcal{A}$ and $\mathcal{C}$ do not change. It is clear that the $L^-$ segment is completely traversed twice from opposite directions, therefore its angle is $2\arctan(\frac{1}{\sqrt{U_{max}^4-1}})+\pi$ according to Lemma \ref{lem G/inflection C angle}. Such a $L^-$ segment is visualized in Fig. \ref{fig:L- twice traversed} for $U_{max}=1.1$ (or $r=\sqrt{2.21}$). However, it is not optimal by Lemma \ref{GC|C|C non optimal}. The same conclusion holds for middle $L^+$ and $R$ segments.
    \begin{figure}[h]
    \centering
    \setlength{\abovecaptionskip}{0pt}
    \includegraphics[width=0.22\textwidth]{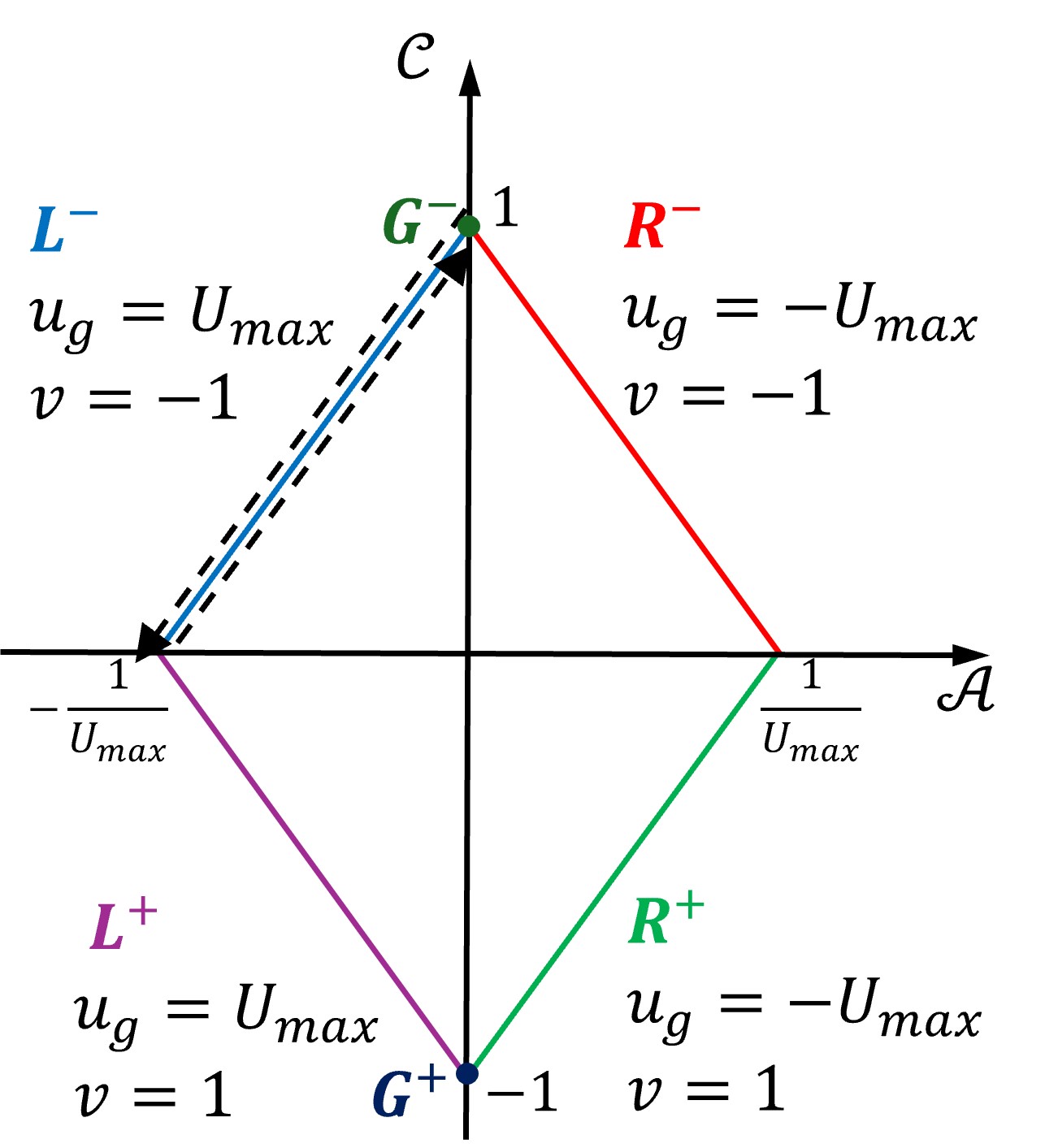}
    \caption{Illustration of a middle $L^-$ segment with switched evolution direction. The dashed arrows indicate the evolution directions: the segment is entered at $(0,1)$ (or $(-\tfrac{1}{U_{\max}},0)$), traverses the second quadrant to $(-\tfrac{1}{U_{\max}},0)$ (or $(0,1)$), switches evolution direction there, and then traverses the second quadrant again in the reverse direction before exiting into the adjacent segment.}
    \label{fig:g=1 twice traversed}
    \end{figure}

    \begin{figure}[h]
    \centering
    \setlength{\abovecaptionskip}{0pt}
    \includegraphics[width=0.24\textwidth]{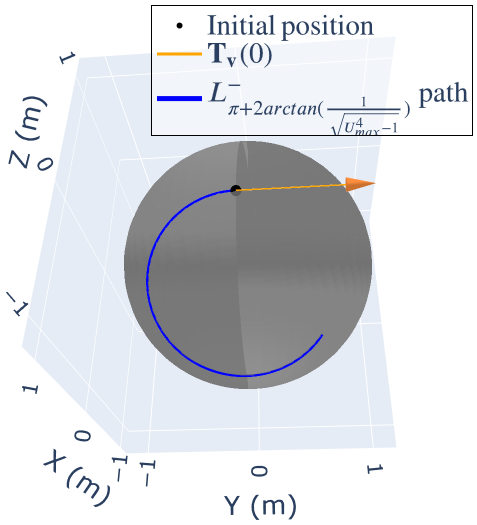}
    \caption{A middle $L^-$ segment with switched evolution directions on a sphere, for $U_{max}=1.1$ (or $r=\frac{1}{\sqrt{2.21}}$)}
    \label{fig:L- twice traversed}
    \end{figure}

    Now consider a middle sub-path that is entered clockwise at the point $(-\frac{1}{U_{max}},0)$ from a cusp. Since it has been shown that the middle $C$ segments must not switch their evolution direction on an optimal path, the middle sub-path can only evolve clockwise in the second quadrant, followed by 3 possible cases: (1) entering a $G^-$ segment and then evolving into the right half plane; (2) directly evolving into the right half plane; (3) entering a $G^-$ segment and then evolving back into the left half plane. In any of the cases, the sub-path again completely traverses a $C$ segment until entering a new sub-path at a cusp. The 3 cases are illustrated in Fig. \ref{fig:g=1 cases}, where case (2) does not enter the $G^-$ segment in Fig. \ref{fig:g=1 cases}(a). It is clear that cases (1)-(3) correspond to middle sub-paths of type $|L^-G^-R^-|$, $|L^-R^-|$, and $|L^-G^-L^-|$, respectively. The angle of the $C$ segments is $\beta=\arctan(\frac{1}{\sqrt{U_{max}^4-1}})+\frac{\pi}{2}$ by Lemma \ref{lem G/inflection C angle}.

    \begin{figure}[h]
    \centering
    \subfigure[Case (1) and case (2)]{\includegraphics[width=0.22\textwidth]{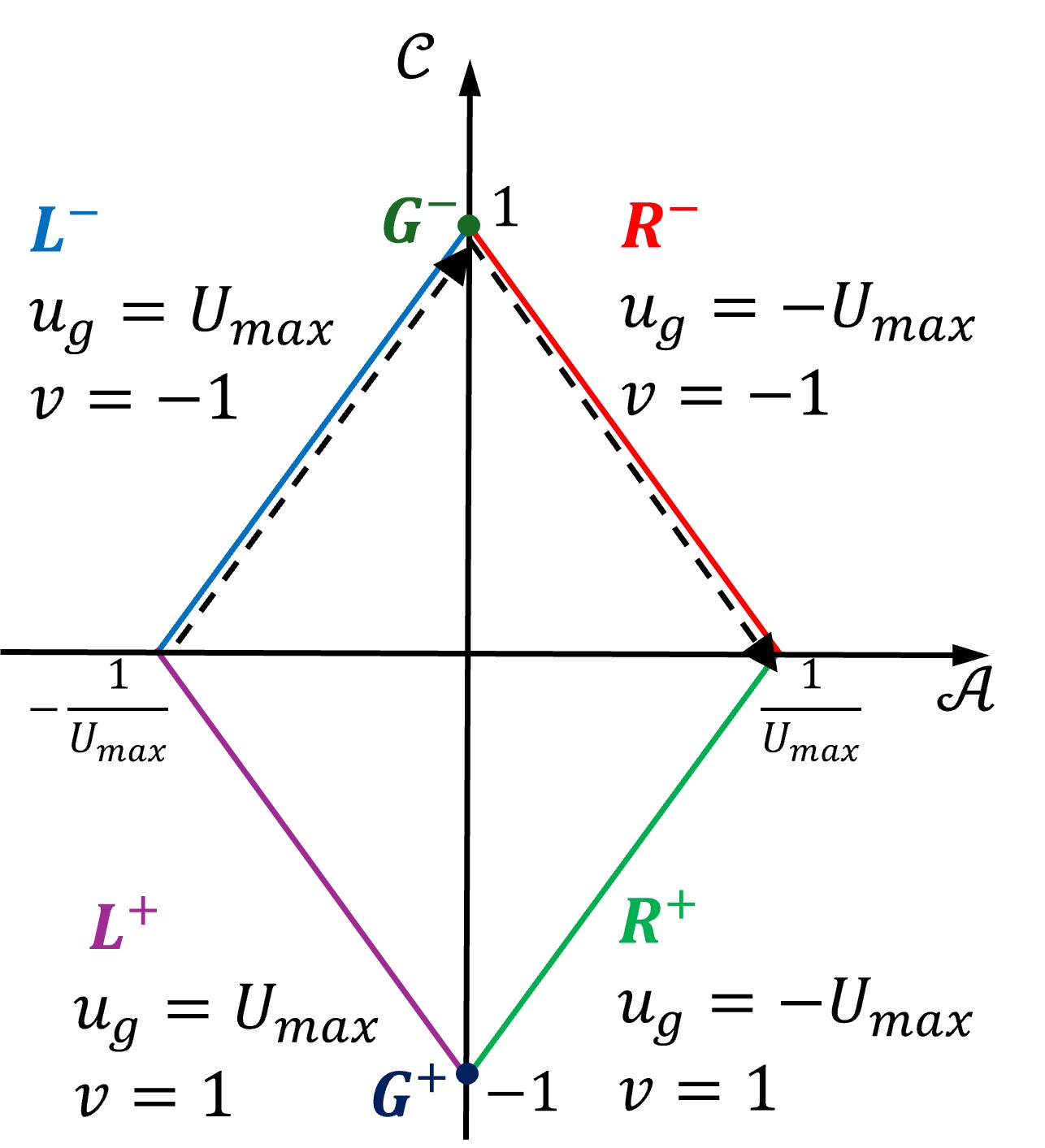}}
    \subfigure[Case (3)]{\includegraphics[width=0.22\textwidth]{Figure/g=1_Case2_3.jpg}}
    \caption{Illustrative evolution of $\mathcal{A}$--$\mathcal{C}$ for $g=1$. The dashed arrows indicate the evolution directions. The sub-path enters an $L^-$ segment at $(-\tfrac{1}{U_{\max}},0)$ and traverses the second quadrant clockwise. In case (1), it enters a $G^-$ segment at $(0,1)$ before entering an $R^-$ segment in the first quadrant; in case (2), it enters an $R^-$ segment directly from $(0,1)$; in case (3), it enters a $G^-$ segment at $(0,1)$ before entering another $L^-$ segment in the second quadrant.}
    \label{fig:g=1 cases}
    \end{figure}
    
    Using similar proofs for the middle sub-path entered at $(-\frac{1}{U_{max}},0)$ counter-clockwise, and the sub-paths entered at $(\frac{1}{U_{max}},0)$ in both clockwise and counter-clockwise directions, it can be proved that the middle sub-paths are of type $|C_\beta GC_\beta|$ or $|C_\beta C_\beta|$. For starting/terminal sub-paths that do not start/terminate at cusps, the starting sub-paths may be truncated to $CGC_\beta|$, $GC_\beta|$, $C_\beta|$, or $CC_\beta|$, with corresponding symmetric forms for the terminal sub-paths. 
\end{proof}


Observing that the key result stated in Lemma \ref{lem G inflection} leads to an infinite list of potentially optimal path types, we begin by demonstrating the non-optimality and redundancy of certain paths fulfilling these conditions. This allows us to limit the optimal path types to a finite list of sufficient ones.

The following lemma states that $L_\beta|L_\beta$ and $R_\beta|R_\beta$ paths can replace each other, that is, they can connect the same initial and terminal configurations.
\begin{lem}
    \label{lem full cusp equivalence}
    A $L^+_\beta|L^-_\beta$ path can be replaced by a $R^+_\beta|R^-_\beta$ path, and a $L^-_\beta|L^+_\beta$ path can be replaced by a $R^-_\beta|R^+_\beta$ path, and vice versa.
\end{lem}
\begin{proof}
    See Appendix \ref{appnd B} in the supplementary materials.
\end{proof}

Using Lemma \ref{lem full cusp equivalence}, we now prove the non-optimality of a $C_\beta C_\beta|C_\beta$ (or $C_\beta|C_\beta C_\beta$) path.

\begin{lem}
    \label{CC|C non optimal}
    A $C_\beta C_\beta|C_\beta$ (or $C_\beta|C_\beta C_\beta$) path is not optimal.
\end{lem}
\begin{proof}
    Consider a $L^-_\beta R^-_\beta|R^+_\beta$ path, where $\beta=\arctan(\frac{1}{\sqrt{U_{max}^4-1}})+\frac{\pi}{2}$. By Lemma \ref{lem full cusp equivalence}, the $R^-_\beta|R^+_\beta$ sub-path can be replaced by $L^-_\beta|L^+_\beta$. Hence, an alternative $L^-_{2\beta}|L^+_\beta$ path of the same length as the original $L^-_\beta R^-_\beta|R^+_\beta$ path is constructed. By Lemma \ref{GC|C|C non optimal}, the alternate path is not optimal; hence, the $L^-_\beta R^-_\beta|R^+_\beta$ path is not optimal.

    The non-optimality of $L^-_\beta R^-_\beta|R^+_\beta$ is visualized in Fig. \ref{fig:L_neg_R_neg_R_pos_non} for $U_{max}=3$ (or $r=\frac{1}{\sqrt{10}}$).

    \begin{figure}[h]
    \centering
    \setlength{\abovecaptionskip}{0pt}
     \includegraphics[width=0.24\textwidth]{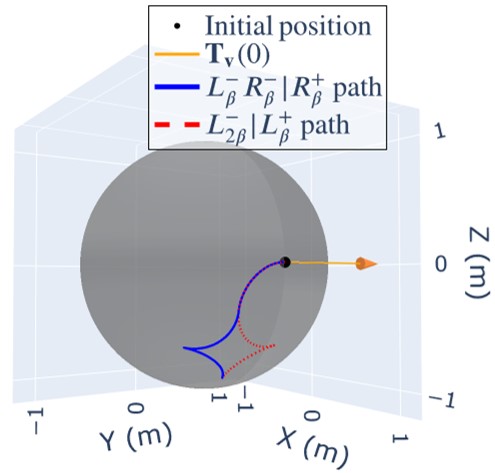}
    \caption{Non-optimality of $L^-_\beta R^-_\beta|R^+_\beta$}
    \label{fig:L_neg_R_neg_R_pos_non}
    \end{figure}

    Using similar proofs for paths $L^+_\beta R^+_\beta|R^-_\beta$, $R^-_\beta L^-_\beta|L^+_\beta$, and $R^+_\beta L^+_\beta|L^-_\beta$, the lemma is proved.
\end{proof}

The following lemma states that with any non-negative $\sigma$, $G_\sigma C_\beta|C_\beta$ and $C_\beta|C_\beta G_\sigma$ paths can replace each other.

\begin{lem}
\label{lem GC|CG replace}
    A path of type $G_\sigma C_\beta|C_\beta $ can be replaced by a path of type $C_\beta|C_\beta G_{\sigma}$, and vice versa.
\end{lem}
\begin{proof}
   See Appendix \ref{append C} in the supplementary materials.
\end{proof}

Using Lemma \ref{lem GC|CG replace}, we now prove the non-optimality of a $C_\beta G_\sigma C_\beta|C_\beta$ (or $C_\beta|C_\beta G_\sigma C_\beta$) path.

\begin{lem}
\label{lem CGCC non optimal}
    For $\sigma\geq0$, a path of type $C_\beta G_\sigma C_\beta|C_\beta$ (or $C_\beta|C_\beta G_\sigma C_\beta$) is not optimal.
\end{lem}
\begin{proof}
    Consider a $R^-_\beta G^-_\sigma L^-_\beta| L^+_\beta$ path, where $\beta=\arctan(\frac{1}{\sqrt{U_{max}^4-1}})+\frac{\pi}{2}$ and $\sigma\geq0$. By Lemma \ref{lem full cusp equivalence}, the $L^-_\beta| L^+_\beta$ sub-path can be replaced by a $R^-_\beta| R^+_\beta$ sub-path. Therefore, a new path of type $R^-_\beta G^-_\sigma R^-_\beta| R^+_\beta$ can be constructed, which has the same length as the initial path. Furthermore, by Lemma \ref{lem GC|CG replace}, the $G^-_\sigma R^-_\beta| R^+_\beta$ sub-path can be replaced by a $ R^-_\beta| R^+_\beta G^+_\sigma$ sub-path; hence, an alternate $R^-_{2\beta}| R^+_\beta G^+_\sigma$ path with the same length is constructed. By Lemma \ref{GC|C|C non optimal}, the alternate path is not optimal; hence, the initial $R^-G^-L^-|L^+$ path is not optimal.

    The non-optimality of $R^-_\beta G^-_\sigma L^-_\beta| L^+_\beta$ is visualized in Fig. \ref{fig:R_neg_G_neg_L_neg_L_pos_non} for $U_{max}=3$ (or $r=\frac{1}{\sqrt{10}}$) and $\sigma=\frac{\pi}{6}$.

    \begin{figure}[h]
    \centering
    \setlength{\abovecaptionskip}{0pt}
     \includegraphics[width=0.24\textwidth]{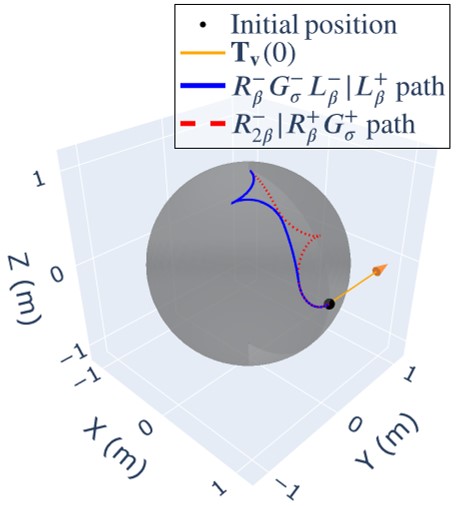}
    \caption{Non-optimality of $R^-_\beta G^-_\sigma L^-_\beta| L^+_\beta$}
    \label{fig:R_neg_G_neg_L_neg_L_pos_non}
    \end{figure}

    Using similar proofs for the path types $R^+G^+L^+| L^-$, $LGL|L$, $RGR| R $, $LGR|R$, and $C|CGC$, the lemma is proved.
\end{proof}

With Lemmas \ref{lem G inflection}, \ref{CC|C non optimal} and \ref{lem CGCC non optimal}, a sufficient list of the optimal path types without $T$ segments for $g=1$ can be obtained. However, it may be further restricted by showing that some of the path types are redundant (i.e., they may be replaced by path types with equivalent lengths but fewer segments).
    
    
    
    

\begin{lem}
\label{lem redundent}

    Paths of type $CC_\beta|C_\beta C$, $CC_\beta|C_\beta G$, 
$CGC_\beta|C_\beta C$, $GC_\beta|C_\beta G$, $CGC_\beta|C_\beta G$, 
and $CGC_\beta|C_\beta GC$ are redundant. In particular, each path can be replaced as follows:
\begin{enumerate}
  \item $CC_\beta|C_\beta C \;\rightarrow\; C|C$
  \item $CC_\beta|C_\beta G \;\rightarrow\; C|C_\beta G$
  \item $CGC_\beta|C_\beta C \;\rightarrow\; CGC_\beta|C$
  \item $GC_\beta|C_\beta G \;\rightarrow\; C|C_\beta G$
  \item $CGC_\beta|C_\beta G \;\rightarrow\; C|C_\beta G$
  \item $CGC_\beta|C_\beta GC \;\rightarrow\; CGC_\beta|C$
\end{enumerate}
\end{lem}
\begin{proof}
    The replacements of 1)-3) follow directly from Lemma \ref{lem full cusp equivalence}; 4) follows directly from Lemma \ref{lem GC|CG replace}, and 5) and 6) follow directly from Lemmas \ref{lem full cusp equivalence} and \ref{lem GC|CG replace}.
\end{proof}

\subsection{Optimal Paths Containing $T$ Segments}

As shown in the phase portrait in Fig. \ref{fig:A-C}, on an optimal path, an $L^0$ (or $R^0$) segment must be preceded and followed by an $L$ (or $R$) segment. By Lemma \ref{g=1 B=0}, these $L$ (or $R$) segments must be completely traversed if they are middle segments, with arc angles of $\beta$, by Lemma \ref{lem G/inflection C angle}. Therefore, the optimal path types containing $T$ segments can be easily restricted to a finite sufficient list by showing a path consisting of a $T$ segment and a completely traversed $C$ segment is non-optimal.

As highlighted in Remark \ref{T U=1}, for $g=1$, an optimal path may include $T$ segments solely in the special case of $U_{max}=1$ (or $r=\frac{1}{\sqrt{2}}$). Hence, $\beta=\arctan(\frac{1}{\sqrt{1-1}})+\frac{\pi}{2}=\pi$, we have the following lemma.

\begin{lem}
    \label{g=1 TC non opt}
    For $\rho\geq0$ and $U_{max}=1$ (or $r=\frac{1}{\sqrt{2}}$), a $T_\rho C_\pi$ (or $C_\pi T_\rho$) path is not optimal.
\end{lem}
\begin{proof}
    Consider a $L^0_\rho L^-_\pi$ path, where $\rho\geq 0$. It is claimed that when $U_{max}=1$, the $L^0_\rho L^-_\pi$ path can be replaced by an alternate $R^+_\pi G^-_\rho$ path of equivalent time. From the phase portrait in Fig. \ref{fig:A-C}, an $R^+$ segment can only be followed by a $G^+$, $R^0$, $L^+$, or $R^-$ segment on an optimal path. Hence, the alternate $R^+_\pi G^-_\rho$ path is not optimal, indicating the non-optimality of the $L^0_\rho L^-_\pi$ path.
    
    As one of the paths considered contains a $T$ segment, we must now compare the paths based on time rather than length. We first show that the two paths are of equivalent time. Recall that the time spent on a segment equals $\frac{\phi}{\omega}$, where $\phi$ and $\omega=\sqrt{v^2+u_g^2}$ denote the angle and angular frequency of the segment, respectively. Therefore, the time duration spent for a $L^0_\rho L^-_\pi$ path and a $R^+_\pi G^-_\rho$ path equal $\frac{\rho}{\sqrt{U_{max}^2}}+\frac{\pi}{\sqrt{1+U_{max}^2}}$ and $\frac{\pi}{\sqrt{1+U_{max}^2}}+\rho$, respectively. These time durations are equivalent when $U_{max}=1$.

    With the equivalence of time shown, it suffices to show that $\bold{M}_{L^0}(\rho)\bold{M}_{L^-}(\frac{1}{\sqrt{2}},\pi)=\bold{M}_{R^+}(\frac{1}{\sqrt{2}},\pi)\bold{M}_{G^-}(\rho)$. By (\ref{M_R+})-(\ref{M_L-}), we obtain
    \begin{align}
    \bold{M}_{L^0}(\rho)\bold{M}_{L^-}(\frac{1}{\sqrt{2}},\pi)&=\bold{M}_{R^+}(\frac{1}{\sqrt{2}},\pi)\bold{M}_{G^-}(\rho) \notag \\
    &=\begin{pmatrix}
0 & 0 &  -1 \\
s(\rho) & -c(\rho)  & 0 \\
-c(\rho) & -s(\rho)  & 0
\end{pmatrix}.
    \end{align}

    The non-optimality of $L^0_\rho L^-_\pi$ is visualized in Fig. \ref{fig:L_zero_L_neg_non} for $\rho=\frac{\pi}{4}$.

    \begin{figure}[h]
    \centering
    \setlength{\abovecaptionskip}{0pt}
     \includegraphics[width=0.24\textwidth]{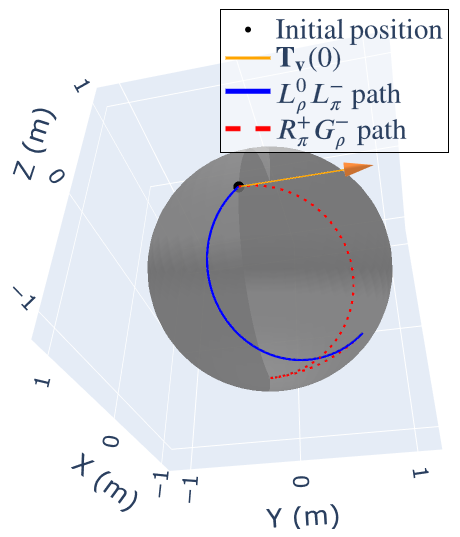}
    \caption{Non-optimality of $L^0_\rho L^-_\pi$}
    \label{fig:L_zero_L_neg_non}
    \end{figure}

    Using similar proofs for paths $L^0_\rho L^+_\pi$, $R^0_\rho R^-_\pi$, and $R^0_\rho R^+_\pi$, the lemma is proved.
\end{proof}

Since for $g=1$ and $U_{max}\geq1$, an optimal path may include $T$ segments solely when $U_{max}=1$ (by Remark \ref{T U=1}), and a $C_\pi$ segment corresponds to a middle $C$ segment when $U_{max}=1$, the following corollary directly follows from Lemma \ref{g=1 TC non opt}. 

\begin{cor}
\label{cor TC}
    For $g=1$ and $U_{max}\geq1$ (or $r\leq\frac{1}{\sqrt{2}}$), a path containing a $T$ segment and a middle $C$ segment is not optimal.
\end{cor}

Up to this point, all the key results (that is, Lemmas \ref{lem G inflection}, \ref{CC|C non optimal}, \ref{lem CGCC non optimal}, \ref{lem redundent}, and Corollary \ref{cor TC}) are proved, along with detailed proofs of the supporting lemmas.

\section{Optimal Paths for $g >1$ \label{sec3.2}}


For $g > 1$, as illustrated in Fig. \ref{fig:g=1 diagram_all}, Lemma \ref{lem g>1 type angle} is the key result that determines the optimal paths as a concatenation. Meanwhile, Lemma \ref{lem g>1 CC|C nonoptimal} serves as the key result that confines the concatenation to a finite list. Below, we enumerate these key results, and their detailed proofs will be provided later in this section.

\begin{defn}
   Throughout this paper, we define the angle $0<\mu<\arctan(\frac{1}{\sqrt{U_{max}^4-1}})+\frac{\pi}{2}$, which will be utilized extensively.
\end{defn}

\noindent\textbf{Lemma \ref{lem g>1 type angle}.}
    \textit{For $g>1$ and $U_{max}\geq1$ (or $r\leq\frac{1}{\sqrt{2}}$), the optimal path is a concatenation of $|C_\mu C_\mu|$ sub-paths where the start/terminal sub-paths can be truncated.}

\vspace{4pt}
\noindent\textbf{Lemma \ref{lem g>1 CC|C nonoptimal}.}
    \textit{For $U_{max}\geq1$ (or $r\leq\frac{1}{\sqrt{2}}$), a path of type $C_\mu C_\mu| C_\mu C_\mu$ is not optimal.}
    \vspace{4pt}
    
Given the above key results, we now prove the main result for $g>1$, as stated in Proposition \ref{prop3}.
\begin{prop}
\label{prop3}
    For $g>1$ and $U_{max}\geq1$ (or $r\leq\frac{1}{\sqrt{2}}$), the optimal path may be restricted to the following types, together with their symmetric forms: 
    
    $C$,\; $CC$,\; $C|C$,\; 
    $C|C_\mu C$,\; $C|C_\mu C_\mu|C$,\;  $CC_\mu |C_\mu C$, 
    
    $C|C_\mu C_\mu |C_\mu C$,\; $CC_\mu|C_\mu C_\mu|C_\mu C$  
    
    where $0<\mu<\arctan(\frac{1}{\sqrt{U_{max}^4-1}})+\frac{\pi}{2}$.
\end{prop}
\begin{proof}
    For $g>1$, Lemma \ref{lem g>1 type angle} dictates that the optimal paths can only be a concatenation of $|C_\mu C_\mu|$ with the start and terminal sub-paths, say $P_i$ and $P_f$, potentially being truncated. Lemma \ref{lem g>1 CC|C nonoptimal} rules out the optimality of paths containing two or more contiguous middle sub-paths of the form $|C_\mu C_\mu|$. Therefore, the sufficient list of optimal path types consists of the following forms: (a) paths with one full middle sub-path and two truncated start/terminal sub-paths, that is, $P_i|C_\mu C_\mu|P_f$, where $P_i \in \{C, CC_\mu\}$ and $P_f \in \{C, C_\mu C\}$; (b) paths with two truncated start/terminal sub-paths, that is, $P_i|P_f$; and (c) paths with one truncated sub-path, that is, $P_s\in\{C,CC\}$. The union of forms (a), (b), and (c) aligns with the sufficient list provided in Proposition \ref{prop3}.
\end{proof}

In the remainder of this section, we present comprehensive proofs of the key results utilized above, along with the lemmas that underpin them. The following two lemmas (\ref{lem g>1 angle} and \ref{lem g>1 B!=0}) build the foundation for the proof of the key result in Lemma \ref{lem g>1 type angle}.
\begin{lem}
\label{lem g>1 angle}
    For $g>1$, the angle of a $C$ segment that is completely traversed clockwise/counter-clockwise is less than $\arctan(\frac{1}{\sqrt{U_{max}^4-1}})+\frac{\pi}{2}$ radians.
\end{lem}
\begin{proof}
    Similarly to the proof of Lemma \ref{lem G/inflection C angle}, (\ref{alpha1 alpha2}) can be obtained, which shows that the angle of a completely traversed $C$ segment monotonically decreases as $g$ increases given a fixed $U_{max}$. Therefore, the angle when $g>1$ is less than the one when $g=1$, hence, less than $\arctan(\frac{1}{\sqrt{U_{max}^4-1}})+\frac{\pi}{2}$.
\end{proof}

\begin{lem}
\label{lem g>1 B!=0}
    For $g>1$ and $U_{max}\geq1$ (or $r\leq\frac{1}{\sqrt{2}}$), $\mathcal{B}\neq0$.
\end{lem}
\begin{proof}
   Similarly to the proof of Lemma \ref{g=1 B=0}, $f(\mathcal{A})$ can be defined as in (\ref{eq: f(A)_definition}), which by definition equals $\mathcal{A}^2 + \mathcal{C}^2$. Using the same argument as in Lemma \ref{g=1 B=0}, we can show that $f(\mathcal{A})$ has a maximum value of 1. Since $\mathcal{A}^2+\mathcal{B}^2+\mathcal{C}^2\equiv g>1$, $\mathcal{B}^2$ is always greater than zero, hence, $\mathcal{B}\neq0$.
\end{proof}

Using Lemmas \ref{lem g>1 angle} and \ref{lem g>1 B!=0}, we proceed to prove the key result of Lemma \ref{lem g>1 type angle}.

\begin{lem}
    \label{lem g>1 type angle}
    For $g>1$ and $U_{max}\geq1$ (or $r\leq\frac{1}{\sqrt{2}}$), the optimal path is a concatenation of $|C_\mu C_\mu|$ sub-paths where the start/terminal sub-paths can be truncated.
\end{lem}
\begin{proof}
    By Lemma \ref{lem g>1 B!=0}, either $\mathcal{B}>0$ or $\mathcal{B}<0$ for an optimal path; consequently, $\frac{d\mathcal{A}}{dt} \ne 0$ when $v \ne 0$ and $\frac{d\mathcal{C}}{dt} \ne 0$ when $u_g \ne 0$. Therefore the path evolves either clockwise or counter-clockwise without entering a $G$ or $T$ segment on the $\mathcal{A}-\mathcal{C}$ portrait by Lemma \ref{clock/counter B}, as shown in Fig. \ref{fig:g>1 cases}. From Fig. \ref{fig:g>1 cases}, it is clear that the optimal path is a concatenation of $|CC|$ sub-paths; for start/terminal sub-paths that do not start/terminate at cusps, they are truncated but the evolution directions remain the same. Furthermore, it is clear that all middle $C$ segments are completely traversed once, either clockwise or counter-clockwise; hence, they all have an angle of $\mu<\arctan(\frac{1}{\sqrt{U_{max}^4-1}})+\frac{\pi}{2}$ according to Lemma \ref{lem g>1 angle}.
    \begin{figure}[h]
    \centering
    \subfigure[Clockwise]{\includegraphics[width=0.24\textwidth]{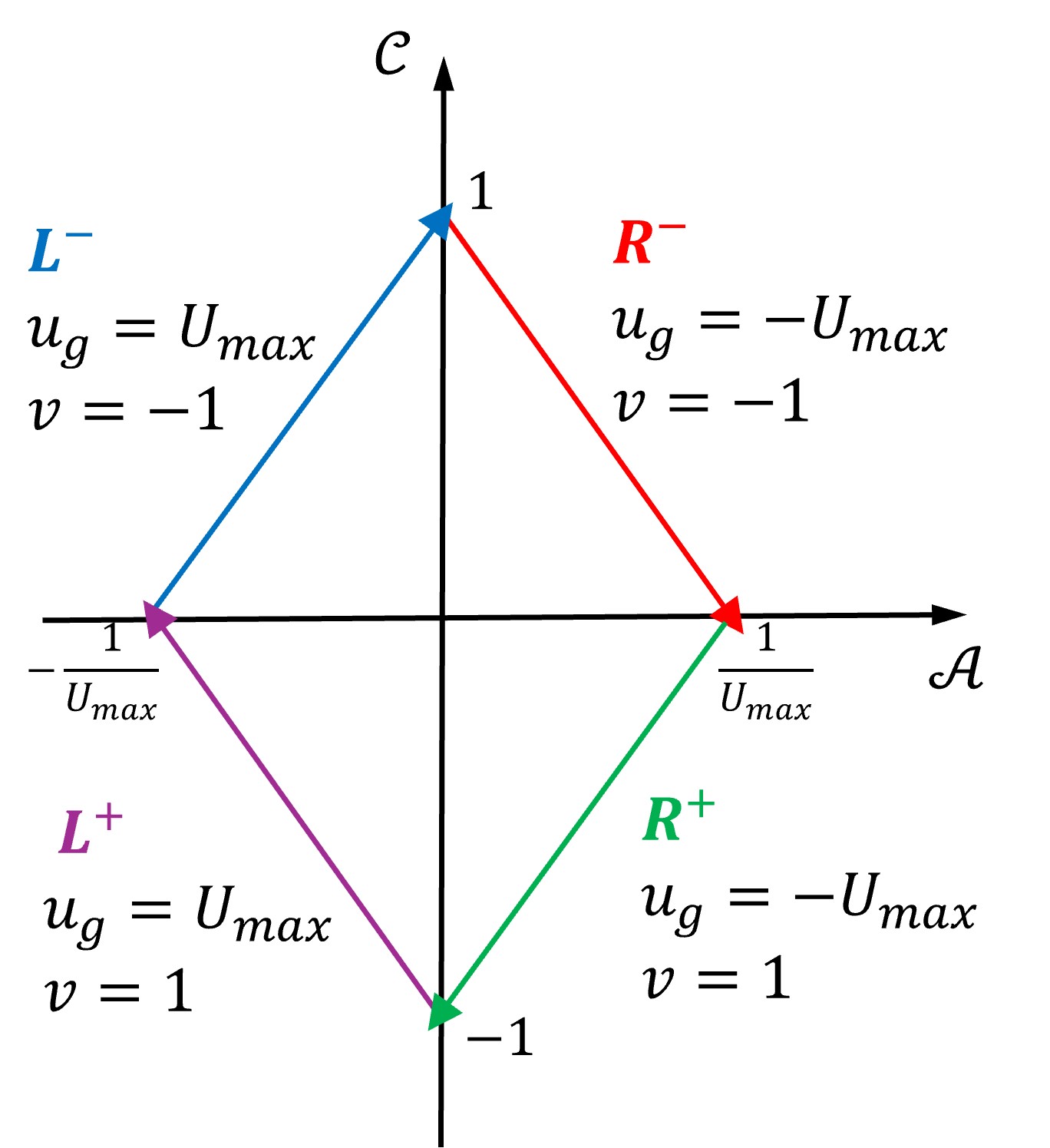}}
    \subfigure[Counter-clockwise]{\includegraphics[width=0.234\textwidth]{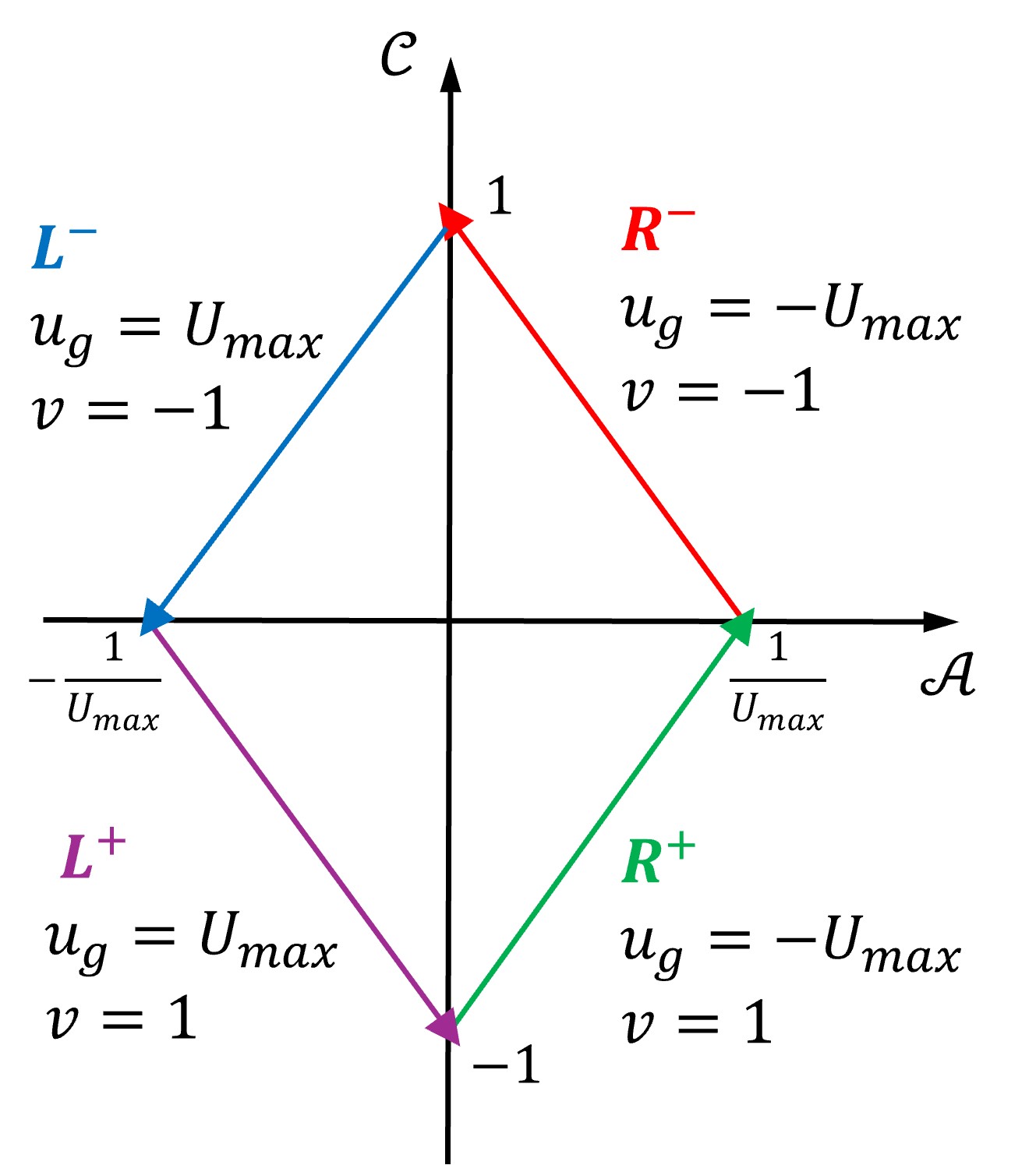}}
    \caption{Evolution of $\mathcal{A}-\mathcal{C}$ for $g>1$}
    \label{fig:g>1 cases}
    \end{figure}
\end{proof}

In the following Lemma \ref{lem CCC2G}, we demonstrate that for $0<\mu<\arctan(\frac{1}{\sqrt{U_{max}^4-1}})+\frac{\pi}{2}$, there is always an $\epsilon$ satisfying $0<\epsilon<2\mu$, which allows the construction of a $C_\mu C_\epsilon C_\mu$ path connecting the same start and terminal configurations as some $G$ path. This result will become useful for proving the key result Lemma \ref{lem g>1 CC|C nonoptimal}.

\begin{lem}
\label{lem CCC2G}
    For $U_{max}\geq1$ (or $r\leq\frac{1}{\sqrt{2}}$), there exists an $\epsilon$ such that $0<\epsilon<2\mu$ and a $C^+_\mu C^+_\epsilon C^+_\mu$ (or $C^-_\mu C^-_\epsilon C^-_\mu$) path can be replaced by a $G^+_\theta$ (or $G^-_\theta$) path.
\end{lem}
\begin{proof}
    See Appendix \ref{appnd D} in the supplementary materials.
\end{proof}

Now, utilizing Lemma \ref{lem CCC2G}, the key result Lemma \ref{lem g>1 CC|C nonoptimal} can be proved.

\begin{lem}
    \label{lem g>1 CC|C nonoptimal}
    For $U_{max}\geq1$ (or $r\leq\frac{1}{\sqrt{2}}$), a path of type $C_\mu C_\mu| C_\mu C_\mu$ is not optimal.
\end{lem}
\begin{proof}
    Consider a $R^+_\mu L^+_\mu|L^-_\mu R^-_\mu$ path where $0<\mu<\arctan(\frac{1}{\sqrt{U_{max}^4-1}})+\frac{\pi}{2}$. It is claimed that there exists an alternate $L^-R^-|R^+ L^+$ or $R^+|R^-|R^+| R^-$ path that is shorter.

    By Lemma \ref{lem CCC2G},
    \begin{equation}
    \label{LRRL1}
    \bold{M}_{L^+}(r,\mu)\bold{M}_{R^+}(r,\epsilon)\bold{M}_{L^+}(r,\mu)=\bold{M}_{G^+}(\theta),
    \end{equation}
    \begin{equation}
    \label{LRRL2}
    \bold{M}_{L^-}(r,\mu)\bold{M}_{R^-}(r,\epsilon)\bold{M}_{L^-}(r,\mu)=\bold{M}_{G^-}(\theta).
    \end{equation}

    From (\ref{M_G+})-(\ref{M_R-}), it is clear that each of them is an orthogonal matrix, and it can be checked that $\bold{M}_{G^+}(\phi)=\bold{M}^{-1}_{G^-}(\phi)$, $\bold{M}_{L^+}(r,\phi)=\bold{M}^{-1}_{R^-}(r,\phi)$, and $\bold{M}_{R^+}(r,\phi)=\bold{M}^{-1}_{L^-}(r,\phi)$, where $\phi$ denotes the angle argument in the rotation matrices. By post-multiplying the two sides of (\ref{LRRL1}) with the two sides of (\ref{LRRL2}), and using the fact that $\bold{M}_{G^+}(\theta)=\bold{M}^{-1}_{G^-}(\theta)$, we obtain: 
    \begin{align}
    &\bold{M}_{L^+}(r,\mu)\bold{M}_{R^+}(r,\epsilon)\bold{M}_{L^+}(r,\mu) \bold{M}_{L^-}(r,\mu)\bold{M}_{R^-}(r,\epsilon) \notag\\ 
    &\cdot\bold{M}_{L^-}(r,\mu)=\bold{M}_{G^+}(\theta)\bold{M}_{G^-}(\theta)=\bold{I_3},
    \end{align}
    where $\bold{I_3}$ is the identity matrix. Further pre-multiplying $\bold{M}_{R^-}(r,\mu)$ and post-multiplying $\bold{M}_{R^+}(r,\mu)\bold{M}_{L^+}(r,\epsilon)$ on both sides:
    \begin{align}
    \label{LRRL3}
    &\bold{M}_{R^+}(r,\epsilon)\bold{M}_{L^+}(r,\mu)\bold{M}_{L^-}(r,\mu) \notag \\
    &=\bold{M}_{R^-}(r,\mu)\bold{M}_{R^+}(r,\mu)\bold{M}_{L^+}(r,\epsilon),
    \end{align}
    this can be rewritten as:
    \begin{align}
    &\bold{M}_{R^+}(r,\epsilon-\mu+\mu)\bold{M}_{L^+}(r,\mu)\bold{M}_{L^-}(r,\mu)
    \notag \\
    &=\bold{M}_{R^-}(r,\mu)\bold{M}_{R^+}(r,\mu)\bold{M}_{L^+}(r,\epsilon-\mu+\mu), 
    \end{align}
    pre-multiplying $\bold{M}_{L^-}(r,\epsilon-\mu)$ and post-multiplying $\bold{M}_{R^-}(r,\mu)$ on both sides of the above equation:
    \begin{align}
    \label{LRRL4}
    &\bold{M}_{R^+}(r,\mu)\bold{M}_{L^+}(r,\mu)\bold{M}_{L^-}(r,\mu)\bold{M}_{R^-}(r,\mu)
    \notag \\
    &=\bold{M}_{L^-}(r,\epsilon-\mu)\bold{M}_{R^-}(r,\mu)\bold{M}_{R^+}(r,\mu)\bold{M}_{L^+}(r,\epsilon-\mu). 
    \end{align}

    From Lemma \ref{lem CCC2G}, $0<\epsilon<2\mu$. If $\epsilon\geq\mu$, then $\mu>\epsilon-\mu\geq0$. Hence, by (\ref{LRRL4}), there exists an alternate $L^-_{\epsilon-\mu}R^-_{\mu}|R^+_\mu L^+_{\epsilon-\mu}$ (or $R^-_{\mu}|R^+_\mu$ if $\epsilon=\mu$) path that is shorter than the $R^+_\mu L^+_\mu|L^-_\mu R^-_\mu$ path.

    If $\epsilon<\mu$, then $0>\epsilon-\mu>-\mu$. Noting that $\bold{M}_{L^-}(r,\epsilon-\mu)=\bold{M}_{R^+}(r,\mu-\epsilon)$ and $\bold{M}_{L^+}(r,\epsilon-\mu)=\bold{M}_{R^-}(r,\mu-\epsilon)$, (\ref{LRRL4}) can be rewritten as:
    \begin{align}
    \label{LRRL5}
    &\bold{M}_{R^+}(r,\mu)\bold{M}_{L^+}(r,\mu)\bold{M}_{L^-}(r,\mu)\bold{M}_{R^-}(r,\mu)
     \notag \\
    &=
    \bold{M}_{R^+}(r,\mu-\epsilon)\bold{M}_{R^-}(r,\mu)\bold{M}_{R^+}(r,\mu)\bold{M}_{R^-}(r,\mu-\epsilon),
    \end{align}
    where $\mu>\mu-\epsilon>0$. Hence, there exists an alternate $R^+_{\mu-\epsilon}|R^-_{\mu}|R^+_\mu| R^-_{\mu-\epsilon}$ path that is shorter than the $R^+_\mu L^+_\mu|L^-_\mu R^-_\mu$ path.

    The non-optimality of $R^+_\mu L^+_\mu|L^-_\mu R^-_\mu$ is visualized in Fig. \ref{fig:R_pos_L_pos_L_neg_R_neg} for two cases: (a) $\mu=\frac{\pi}{2}$ and $U_{max}=3$ (or $r=\frac{1}{\sqrt{10}}$); (b) $\mu=\frac{3}{4}\pi$ and $U_{max}=1.1$ (or $r=\frac{1}{\sqrt{2.21}}$). In these cases, the alternate shorter paths are $L^-_{\epsilon-\mu}R^-_{\mu}|R^+_\mu L^+_{\epsilon-\mu}$ and $R^+_{\mu-\epsilon}|R^-_{\mu}|R^+_\mu| R^-_{\mu-\epsilon}$, respectively.

    \begin{figure}[h]
    \centering
    \subfigure[$\mu=\frac{\pi}{2}$ and $U_{max}=3$]{\includegraphics[width=0.24\textwidth]{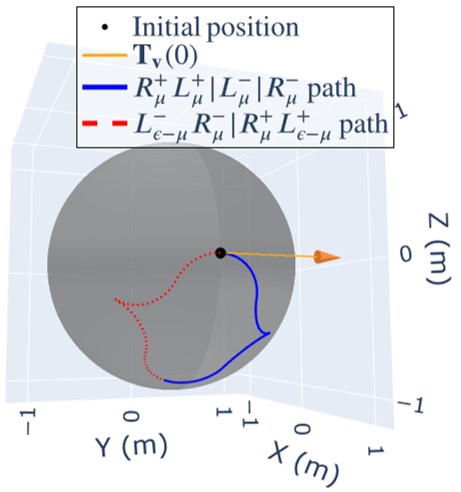}}
    \subfigure[$\mu=\frac{3}{4}\pi$ and $U_{max}=1.1$]{\includegraphics[width=0.24\textwidth]{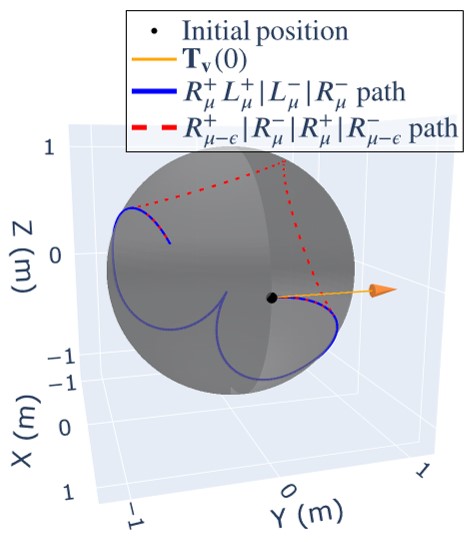}}
    \caption{Non-optimality of $R^+_\mu L^+_\mu|L^-_\mu R^-_\mu$}
    \label{fig:R_pos_L_pos_L_neg_R_neg}
    \end{figure}

    With similar proofs for $R^-_\mu L^-_\mu|L^+_\mu R^+_\mu$, $L^+_\mu R^+_\mu|R^-_\mu L^-_\mu$, and $L^-_\mu R^-_\mu|R^+_\mu L^+_\mu$ paths, the lemma is proved.
\end{proof}



With the key results, Lemmas \ref{lem g>1 type angle} and \ref{lem g>1 CC|C nonoptimal}, the finite sufficient list for $g>1$ listed in Proposition \ref{prop3} can be obtained.

\section{Optimal Paths for $g < 1$ \label{sec3.3}}
Referring to Fig. \ref{fig:g=1 diagram_all}, for $g<1$, Lemma \ref{lem ...C|C|C|C...} alone adequately indicates the main result.

For $\mathcal{A}^2+\mathcal{B}^2+\mathcal{C}^2\equiv g< 1$, an extremal path never reaches $(0,1)$ or $(0,-1)$ in Fig. \ref{fig:A-C}. Hence, it only evolves within the left or right half plane in Fig. \ref{fig:A-C}, and $u_g\equiv\pm U_{max}$ since $\mathcal{A}$ does not change sign. From (\ref{d_switch}), $\frac{d\mathcal{C}}{dt}=-U_{max}\mathcal{B}$. By squaring both sides and utilizing (\ref{ABC constant}), we obtain $\left(\frac{d\mathcal{C}}{dt}\right)^2=U_{max}^2(g^2-\mathcal{A}^2-\mathcal{C}^2)$. Substituting $\mathcal{A}^2$ by means of (\ref{A-C relation}), wherein $\mathcal{A}=\frac{1-|\mathcal{C}|}{\pm U_{max}}$, results in $\left(\frac{d\mathcal{C}}{dt}\right)^2=U_{max}^2\left(g^2-\frac{1+|\mathcal{C}|^2-2|\mathcal{C}|}{U_{max}^2}-\mathcal{C}^2\right)$. By reorganizing the equation, the relationship between $\mathcal{C}(t)$ and $\frac{d\mathcal{C}(t)}{dt}$ is established as:
\begin{equation}
\label{pahse portrait C-dC}
    \left(|\mathcal{C}(t)|-\frac{1}{1+U_{max}^2}\right)^2+\frac{1}{1+U_{max}^2}\left(\frac{d\mathcal{C}(t)}{dt}\right)^2=\lambda_\mathcal{C}^2,
\end{equation}
where $\lambda_\mathcal{C}^2=\frac{U_{max}^2g^2-1}{1+U_{max}^2}+\frac{1}{(1+U_{max}^2)^2}$. Using the above equation and (\ref{bangbang2}), the phase portrait of $\mathcal{C}(t)$ is shown in Fig. \ref{fig:C-dC}. Note that $\lambda_\mathcal{C}\geq 0$, and since an extremal path never reaches $(0,1)$ or $(0,-1)$ in Fig. \ref{fig:A-C} for $g<1$, we have $\mathcal{C}(t)<1$ in Fig. \ref{fig:C-dC}; hence, $\lambda_\mathcal{C}<\sqrt{\left(1-\frac{1}{1+U_{max}^2}\right)^2+0}=\frac{U_{max}^2}{1+U_{max}^2}$. Furthermore, for the phase portrait to intersect with the $\frac{d\mathcal{C}}{dt}$ axis, $\lambda_\mathcal{C}\geq\sqrt{\left(0-\frac{1}{1+U_{max}^2}\right)^2+0}=\frac{1}{1+U_{max}^2}$.
\begin{figure}[h]
    \centering
    \setlength{\abovecaptionskip}{0pt}
    \includegraphics[width=0.35\textwidth]{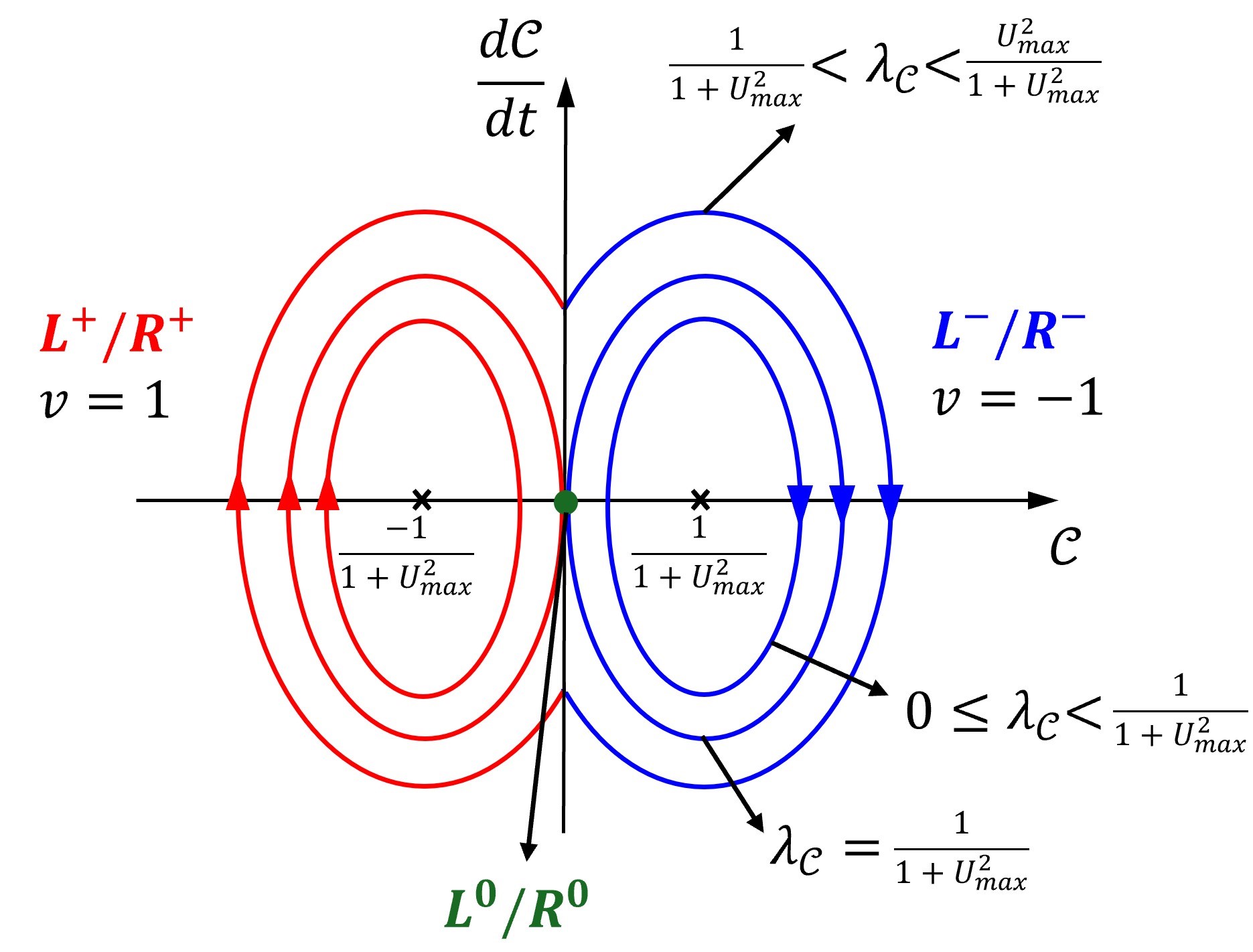}
    \caption{Phase portrait of $\mathcal{C}$-$\frac{d\mathcal{C}}{dt}$ for $g<1$}
    \label{fig:C-dC}
\end{figure}

The key result Lemma \ref{lem ...C|C|C|C...} can be obtained by directly utilizing the phase portrait of $\mathcal{C}(t)$.

\begin{lem}
\label{lem ...C|C|C|C...}
    For $g<1$ and $U_{max}\geq1$ (or $r\leq\frac{1}{\sqrt{2}}$), the optimal path is of type $C$, $T$, $TC$ (or $CT)$, or a concatenation of $C$ segments joined by cusps or joined by $T$ segments, and any path containing a middle $C$ segment is not optimal.
\end{lem}
\begin{proof}
    From the phase portrait shown by Fig. \ref{fig:C-dC}, it is clear that for $0\leq\lambda_\mathcal{C}<\frac{1}{1+U_{max}^2}$, a nontrivial extremal path is a $C$ path. 
    
    For $\frac{1}{1+U_{max}^2}<\lambda_\mathcal{C}<\frac{U_{max}^2}{1+U_{max}^2}$, it is a concatenation of $C$ segments joined by cusps, with each middle $C$ segment traversing the left or right half plane exactly once. 
    
    For $\lambda_\mathcal{C}=\frac{1}{1+U_{max}^2}$, it is either a $T$ path, a $TC$ (or $CT$) path, or a concatenation of $C$ segments. Between two consecutive $C$ segments, if a path remains at the origin in Fig. \ref{fig:C-dC} for some duration, the $C$ segments are joined by a $T$ segment; otherwise, they are joined by a cusp. Furthermore, the middle $C$ segments traverse the left or right half plane at least once (if a middle $C$ segment passes through the origin instantly without transitioning into another half plane, it traverses the same half plane at least twice).

    Now we prove that any path containing a middle $C$ segment is not optimal. Given that a middle $C$ segment traverses the left or right half plane at least once, as established in the above discussion, it suffices to demonstrate the non-optimality of a middle $C$ segment that traverses the left or right half plane just once. Consider such a segment of type $R^-$ and let $t_2$ denote the time spent on it. From Fig. \ref{fig:C-dC}, on the $R^-$ segment, 
    \begin{equation}
    \mathcal{C}(0)=\mathcal{C}(t_2)=0. 
    \end{equation}
By \eqref{A-C relation}, we have $\mathcal{A}\equiv\frac{-1+\mathcal{C}}{-U_{max}}$, therefore,
    \begin{equation}
    \mathcal{A}(0)=\mathcal{A}(t_2)=\frac{1}{U_{max}}. 
    \end{equation}
Furthermore, by (\ref{pahse portrait C-dC}), $\frac{d\mathcal{C}}{dt}(0)=\sqrt{\frac{(1+U_{max}^2)^2\lambda_\mathcal{C}^2-1}{1+U_{max}^2}}=-\frac{d\mathcal{C}}{dt}(t_2)$, and by (\ref{d_switch}), $\frac{d\mathcal{C}}{dt}=-u_g\mathcal{B}$, we have
\begin{equation}
    \mathcal{B}(0)=\sqrt{\frac{(1+U_{max}^2)^2\lambda_\mathcal{C}^2-1}{U_{max}^2+U_{max}^4}}=-\mathcal{B}(t_2).
\end{equation}

Therefore, by \eqref{eq:A_B_C_sol} and noting that the specific $\mathbf{M}(\cdot)$ matrix of a $R^-$ segment is given by \eqref{M_R-}, from the above end-point conditions we have
    \begin{equation}
    \begin{aligned}
        \begin{bmatrix}
            \frac{1}{U_{max}} \\[4pt]
            -\sqrt{\frac{(1+U_{max}^2)^2\lambda_\mathcal{C}^2-1}{U_{max}^2+U_{max}^4}} \\[4pt]
            0
        \end{bmatrix}=\bold{M}^T_{R^-}(r, \phi)
        \begin{bmatrix}
            \frac{1}{U_{max}} \\[4pt]
            \sqrt{\frac{(1+U_{max}^2)^2\lambda_\mathcal{C}^2-1}{U_{max}^2+U_{max}^4}} \\[4pt]
            0
        \end{bmatrix}.
        \end{aligned}
    \end{equation}

    Expanding the last row of the above equation and rearranging yields:
    \begin{equation}
0
=
(1-c(\phi))\,r\,\frac{1}{U_{max}}
+
s(\phi)\sqrt{\frac{(1+U_{max}^2)^2\lambda_c^2-1}{U_{\max}^2+U_{max}^4}}.
\end{equation}
Note that $r>0$, $U_{\max}>0$. Moreover, $1-c(\phi)> 0$, note that equality does not hold since it corresponds to $\phi= 2k\pi$ for $k=1,2,\cdots$, which is non-optimal by Lemma \ref{GC|C|C non optimal}. Hence, the above identity implies $s(\phi)<0$. 
Because $\phi$ denotes the positive arc angle, $s(\phi)<0$ yields $\phi\in(\pi+2k\pi,2\pi+2k\pi)$ for $k=0,1,\cdots$. Therefore, $\phi>\pi$, which is non-optimal by Lemma~\ref{GC|C|C non optimal}.

With similar proofs for middle segments of types $L^-$ and $C^+$, the lemma is proved.
\end{proof}

With Lemma \ref{lem ...C|C|C|C...}, we can directly prove the main result for $g<1$, as listed in Proposition \ref{prop4} below.

\begin{prop}
\label{prop4}
    For $g<1$ and $U_{max}\geq1$ (or $r\leq\frac{1}{\sqrt{2}}$), the optimal path may be restricted to the following types, together with their symmetric forms: $C$,\; $T$,\; $C|C$,\; $TC$,\; $CTC$. 
\end{prop}
\begin{proof}
    According to Lemma \ref{lem ...C|C|C|C...}, for $g<1$, no path that contains a middle $C$ segment is optimal, thus ruling out paths that contain more than two $C$ segments. Therefore, within the possible forms by Lemma \ref{lem ...C|C|C|C...}, the only paths left are $C$, $T$, $TC$ (or $CT$), and the paths with two $C$ segments joined by a cusp or a $T$ segment, namely, $C|C$ and $CTC$.
\end{proof}
With Propositions \ref{prop2}-\ref{prop4}, the full sufficient list is characterized:
\begin{thm}
\label{thm1}
    For $U_{max}\geq1$ (or $r\leq\frac{1}{\sqrt{2}}$), the optimal path may be restricted to the following types, together with their symmetric forms:

    
    
    
    
\begin{itemize}
  \item \textbf{1--2 segments:}
  $C$,\; $G$,\; $T$,\; $CC$,\; $GC$,\; $C|C$,\; $TC$.
  
  \item \textbf{3 segments:}
  $CC_\psi|C$,\; $CGC$,\; $C|C_\beta G$,\; $CTC$.
  
  \item \textbf{4 segments:}
  $C|C_\psi C_\psi|C$,\; $CGC_\beta|C$,\; $CC_\mu |C_\mu C$.
  
  \item \textbf{5 segments:}
  $C|C_\beta GC_\beta|C$,\; $C|C_\mu C_\mu |C_\mu C$.
   \item \textbf{6 segments:}
   $CC_\mu|C_\mu C_\mu|C_\mu C$.
\end{itemize}

     where $0<\psi\leq\arctan(\frac{1}{\sqrt{U_{max}^4-1}})+\frac{\pi}{2}$, $\beta=\arctan(\frac{1}{\sqrt{U_{max}^4-1}})+\frac{\pi}{2}$, and $0<\mu<\arctan(\frac{1}{\sqrt{U_{max}^4-1}})+\frac{\pi}{2}$.
\end{thm}

\section{Path Generation and Numerical Experiments}
\label{sec generate}
\subsection{Path Generation and Numerical Example}
With the sufficient list of optimal path types, multiple candidate solutions for each path can be generated through inverse kinematics, given an initial configuration, a desired terminal configuration, and an $U_{max}$ (or $r$). To this end, we make use of the rotation matrices in (\ref{M_G+})-(\ref{M_R0}) and their corresponding axial vectors to derive closed-form expressions for the angles of each path in the sufficient list. Specifically, for each specified path type (e.g., $L^+L^-$), the terminal configuration condition is enforced through forward kinematics propagation using rotation matrices. By pre- and post-multiplying the resulting matrix equations with appropriate axial vectors, we extract trigonometric polynomial equations in the segment angles. These equations admit a small number of closed-form solutions, which are enumerated to generate candidate paths. With the candidate solutions generated, the feasible paths can be readily identified by verifying if the candidates connect the  initial and desired terminal configuration through forward kinematics. Subsequently, the optimal path can be found by comparing the time across all feasible paths. A separate note for the generation of candidate paths, along with the source code for solving the time-optimal spherical CRS problem and visualization, is available at \href{https://github.com/sixuli97/Optimal-Spherical-Convexified-Reeds-Shepp-Paths}{https://github.com/sixuli97/Optimal-Spherical-Convexified-Reeds-Shepp-Paths}.

In the following numerical example, the initial configuration is set to be $\bold{R}(0)=\bold{I_3}$, and $U_{max}=3$  (or $r=\frac{1}{\sqrt{10}}$). The desired terminal configuration is randomly generated as
\begin{equation}
    \bold{R}(T)=\begin{pmatrix}
0.804977 & -0.592216 & 0.035944 \\
-0.569461 & -0.754203 & 0.326943 \\
-0.166512 & -0.283650 & -0.944360.
\end{pmatrix}
\end{equation}

The feasible paths are summarized in Table \ref{tab:1}. By comparing the third column, it is clear that the $R^-R^+G^+L^+$ path is time-optimal. The paths are visualized in Fig. \ref{fig:numerical example}.

\begin{table}[h]
\centering
\scriptsize
\setlength{\abovecaptionskip}{0pt}
\caption{Feasible Paths}
\label{tab:1}
    \centering
    \begin{tabular}[t]{|>{\centering\arraybackslash}m{1.8cm}|>{\centering\arraybackslash}m{4.0cm}|>{\centering\arraybackslash}m{1cm}|}
        \hline
        \textbf{\textit{Path type}} & \textbf{\textit{Angles (rad)}} & 
        \textbf{\textit{Time (s)}} \\
        \hline
        $L^-R^-R^+$ & $[0.1122, 1.4896, 1.6238]$ & 1.0200\\
        \hline
        $L^-L^0L^+$ & $[1.2685, 1.3659, 0.9832]$ & 1.1673\\
        \hline
        $L^-R^-R^+L^+$ & $[2.4701, 0.5045, 0.5045, 2.1848]$ & 1.7911 \\
        \hline
        $R^+L^+L^-R^-$ & $[2.5273, 1.5573, 1.5573, 2.8126]$ & 2.6735\\
        \hline
        $R^-R^+G^+L^+$ & $[1.4008, 1.6821, 0.0160, 0.0864]$ & 1.0182 \\
        \hline
    \end{tabular}
\hspace{0.01\linewidth}
\end{table}

\begin{figure}[h]
    \centering
    \setlength{\abovecaptionskip}{0pt}
    \includegraphics[width=0.49\textwidth]{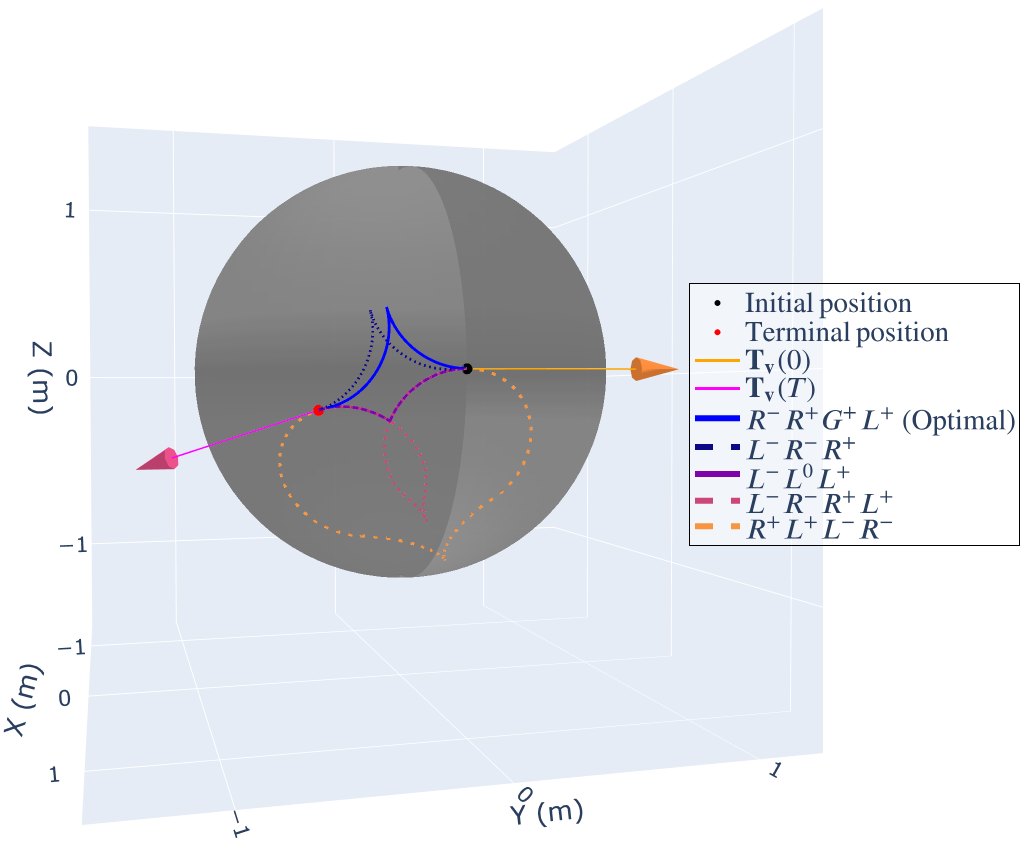}
    \caption{Optimal and feasible paths connecting the initial and desired terminal configurations}
    \label{fig:numerical example}
\end{figure}

\subsection{Numerical Coverage Study via Configuration Sampling}
To assess the practical coverage of the sufficient list over the configuration space, terminal configurations were sampled in a structured and deterministic manner. We sampled $4000$ position vectors on the unit sphere using a Fibonacci lattice \cite{gonzalez2010measurement}, which provides near-uniform coverage, as shown in Fig. \ref{fig:coverage}. For each position, 30 headings were uniformly sampled over $[0,2\pi)$ within the tangent plane, yielding a discrete set of terminal orientations. 

\begin{figure}[h]
    \centering
    \setlength{\abovecaptionskip}{0pt}
    \includegraphics[width=0.15\textwidth]{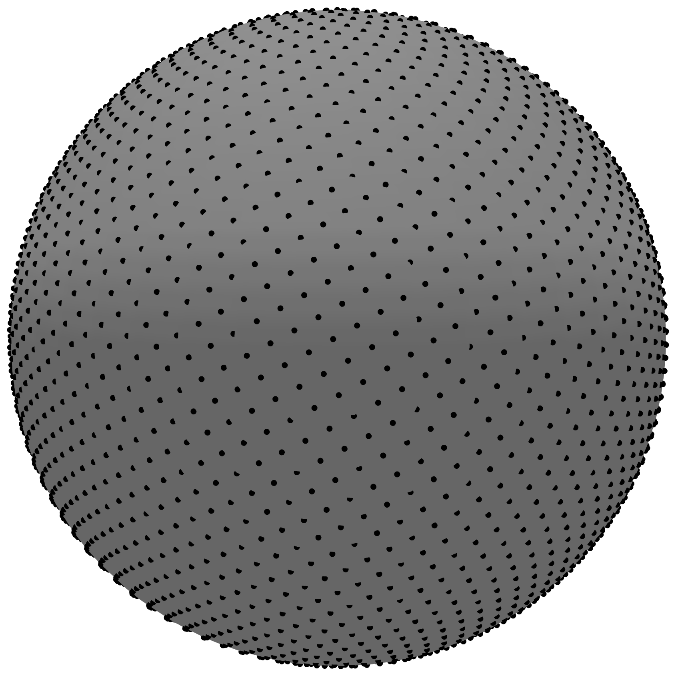}
    \caption{Sampled positions on the sphere using Fibonacci lattice}
    \label{fig:coverage}
\end{figure}

For each sampled terminal configuration, we enumerated the sufficient list to generate candidate paths and selected the time-optimal solution as the feasible candidate with minimum time. Three values of $U_{max}$ (1, 5, and 10) were tested. Across all sampled configurations and all tested values of $U_{max}$, at least one feasible path was found within the sufficient list, providing empirical support for the coverage of the list. In addition, Table~\ref{tab:optimal count} reports the frequency with which each path type (grouped together with its symmetric form) is selected as optimal. While this does not provide a guarantee over the continuous configuration space, it can offer guidance on which path types do not appear to be optimal within the samples (e.g., $C|C_\mu C_\mu|C_\mu C$) and may be candidates for further refinement of the sufficient list.

\begin{table}[h]
\centering
\scriptsize
\setlength{\tabcolsep}{2.1pt}
\renewcommand{\arraystretch}{1.15}
\caption{Percentage of Optimal Paths}
\label{tab:optimal count}

\begin{minipage}{0.98\linewidth}
\setlength{\tabcolsep}{3.6pt}
\centering
\begin{tabular}{>{\raggedright\arraybackslash}p{1.4cm} ccc  cccc}
\toprule
& \multicolumn{3}{c}{1 segment} & \multicolumn{4}{c}{2 segments} 
\\
\cmidrule(lr){2-4}\cmidrule(lr){5-8}
& $C$ & $G$ & $T$ & $CC$ & $GC$ & $C|C$ & $TC$  \\
\midrule
$U_{\max}=1$  & 0.001\% & 0.001\% & 0.001\% & 0.015\% & 0.024\% & 0.033\% & 0.002\%  \\
$U_{\max}=5$ & 0 & 0.001\% & 0.001\% & 0.001\% & 0.038\% & 0.003\% & 0  \\
$U_{\max}=10$ & 0 & 0.001\% &0.001\% & 0 & 0.038\% & 0 & 0 \\
\bottomrule
\end{tabular}
\end{minipage}

\vspace{2pt}

\begin{minipage}{0.98\linewidth}
\setlength{\tabcolsep}{2.61pt}
\centering
\begin{tabular}{>{\raggedright\arraybackslash}p{1.4cm} cccc cc}
\toprule
& \multicolumn{4}{c}{3 segments} & \multicolumn{2}{c}{4 segments} \\
\cmidrule(lr){2-5}\cmidrule(lr){6-7}
& $CC_\psi|C$ & $CGC$ & $C|C_\beta G$ & $CTC$ & $C|C_\psi C_\psi|C$ & $CGC_\beta|C$  \\
\midrule
$U_{\max}=1$ &44.878\% & 24.986\% & 0 & 24.945\% & 2.588\% & 0 \\
$U_{\max}=5$  & 1.645\% & 47.423\% & 0.031\% & 1.168\% & 0.662\% & 47.901\%  \\
$U_{\max}=10$ & 0.411\% & 49.122\% & 0.024\% & 0.293\% & 0.177\% & 49.456\%  \\
\bottomrule
\end{tabular}
\end{minipage}

\vspace{2pt}

\begin{minipage}{0.98\linewidth}
\centering
\begin{tabular}{>{\raggedright\arraybackslash}p{1.4cm} c cc c}
\toprule
& \multicolumn{1}{c}{4 segments}  
& \multicolumn{2}{c}{5 segments} & \multicolumn{1}{c}{6 segments} \\
\cmidrule(lr){2-2}\cmidrule(lr){3-4}\cmidrule(lr){5-5}
& $CC_\mu|C_\mu C$& $C|C_\beta GC_\beta|C$ & $C|C_\mu C_\mu|C_\mu C$ & $CC_\mu|C_\mu C_\mu|C_\mu C$ \\
\midrule
$U_{\max}=1$ & 2.511\%  & 0 & 0 & 0.017\% \\
$U_{\max}=5$ & 0.111\% & 1.017\%\ & 0 & 0 \\
$U_{\max}=10$ & 0.028\%& 0.449\% & 0 & 0 \\
\bottomrule
\end{tabular}
\end{minipage}

\end{table}

\section{Conclusion \label{sec5}}
This paper addresses the time-optimal path problem of the CRS vehicle on a unit sphere. The spherical CRS vehicle is modeled as a variant of the Dubins vehicle in the Sabban frame, with considerations for constrained turning rate and speed. The problem can be used for optimal attitude control of underactuated satellites, optimal motion planning for spherical rolling robots, and optimal path planning for mobile robots on spherical surfaces or uneven terrains. 

To solve the proposed problem, the Pontryagin maximum principle is utilized, along with the analysis of phase portraits of the Hamiltonians. Furthermore, detailed proofs of the non-optimality or redundancy of certain path types are established, ultimately leading to the derivation of a sufficient list of optimal path types for $U_{max}\geq1$ (or $r\leq\frac{1}{\sqrt{2}}$). This sufficient list includes 23 path types, each composed of at most 6 segments, ensuring a thorough exploration of possible optimal paths given the start and desired terminal configurations. 

Furthermore, closed-form expressions for the angles
of each path in the sufficient list are derived, and the source code for solving the time-optimal path problem is released publicly online. The complementary case of $U_{max}<1$ can be rewritten and solved with the same results.



\appendix[]
\subsection{Equivalent Formulation when $U_{max}<1$}
\label{append equivalence}

Recall that the original time-optimal spherical CRS problem is defined by \eqref{eq: time cost}-\eqref{terminal conditions} with $v \in[-1\text{,}1]$ and $u_g \in [-U_{max},U_{max}]$. Moreover, the state equations (\ref{RS model1})-(\ref{RS model3}) are equivalent to \eqref{RS model}. In the main text of the paper, we characterized a sufficient list of optimal path types for $U_{max}\geq 1$; here, we show that when $0<U_{max}<1$, the time-optimal CRS problem can be rewritten in an equivalent form by coordinate transformation, time-scaling, and redefined control inputs $\Tilde{v}$ and $\Tilde{u}_{g}$, with $\Tilde{v}\in[-1,1]$ , $\Tilde{u}_{g}\in[-\Tilde{U}_{max},\Tilde{U}_{max}]$, and $\Tilde{U}_{max}>1$. With this equivalent formulation, the case of $U_{max}<1$ can be solved using the results for $\tilde U_{max}\geq1$ developed in this paper. The derivation is as follows.

 Define $\Tilde{\mathbf{R}}:=\mathbf{Q}^{-1}\mathbf{R}\mathbf{Q}$, where $\mathbf{Q}:=\begin{pmatrix}
    0 & 0 & 1 \\ 0 & 1 & 0 \\ -1 & 0 & 0
\end{pmatrix}$. It can be checked that $\mathbf{Q}\in SO(3)$, and hence $\Tilde{\mathbf{R}}\in SO(3)$, since it is a product of matrices in $SO(3)$ \cite{lynch2017modern}. 

Substituting the relation $\mathbf{R}=\mathbf{Q}\Tilde{\mathbf{R}}\mathbf{Q}^{-1}$ into \eqref{terminal conditions}, we obtain
\begin{align}
    &\mathbf{Q}\Tilde{\bold{R}}(0)\mathbf{Q}^{-1} =\bold{I_3}, ~\mathbf{Q}\Tilde{\bold{R}}(T)\mathbf{Q}^{-1} =\bold{R_f} \notag\\
    \label{terminal conditions new}
    \implies & \Tilde{\bold{R}}(0)=\bold{I_3},~\Tilde{\bold{R}}(T) =\mathbf{Q}^{-1}\bold{R_f}\mathbf{Q}:=\Tilde{\mathbf{R}}_\mathbf{f}.
\end{align}
Similarly, substituting $\mathbf{R}=\mathbf{Q}\Tilde{\mathbf{R}}\mathbf{Q}^{-1}$ into \eqref{RS model}, we obtain
\begin{align}
    \mathbf{Q}\frac{d\Tilde{\bold{R}}(t)}{dt} &\mathbf{Q}^{-1}
    =\mathbf{Q}\Tilde{\mathbf{R}}(t)\mathbf{Q}^{-1} \begin{pmatrix}
0 & -v(t) & 0 \\
v(t) & 0 & -u_g(t) \\
0 & u_g(t) & 0 \\
\end{pmatrix} \notag\\
\implies  \frac{d\Tilde{\bold{R}}(t)}{dt} 
    &=\Tilde{\mathbf{R}}(t)\mathbf{Q}^{-1} \begin{pmatrix}
0 & -v(t) & 0 \\
v(t) & 0 & -u_g(t) \\
0 & u_g(t) & 0 \\
\end{pmatrix} \mathbf{Q} \notag\\
\label{state equation new}
  &=\Tilde{\mathbf{R}}(t) \begin{pmatrix}
0 & -u_g(t) & 0 \\
u_g(t) & 0 & v(t) \\
0 & -v(t) & 0 \\
\end{pmatrix}.
\end{align}

With the scaled time $\tau=U_{max}t$, \eqref{terminal conditions new} becomes
\begin{equation}
\label{terminal constraints scaled}
    \Tilde{\bold{R}}(0)=\bold{I_3},~\Tilde{\bold{R}}(\Tilde{T}) =\Tilde{\mathbf{R}}_\mathbf{f},
\end{equation}
where $\Tilde{T}=U_{max}T$. Similarly, \eqref{state equation new} becomes
\begin{align}
\label{RS model new}
    \frac{d\Tilde{\bold{R}}}{d\tau} =\frac{d\tilde{\mathbf{R}}}{dt}\frac{dt}{d\tau}=\frac{1}{U_{max}}\frac{d\Tilde{\bold{R}}}{dt}=\Tilde{\mathbf{R}}\begin{pmatrix}
0 & -\Tilde{v} & 0 \\
\Tilde{v} & 0 & -\Tilde{u}_g \\
0 & \Tilde{u}_g & 0 \\
\end{pmatrix},
\end{align}
where we define $\Tilde{v}:=\frac{u_g}{U_{max}}$ and $\Tilde{u}_g:=-\frac{v}{U_{max}}$. Note that control channels are exchanged and rescaled. In particular, for the optimal control actions defined in \eqref{bangbang2}, $v=-1,1,0$ correspond to $\Tilde{u}_g=\tilde U_{max},-\tilde U_{max},0$, respectively, where $\tilde U_{max}:=\frac{1}{U_{max}}$; whereas $u_g=-U_{max},U_{max},0$ correspond to $\Tilde{v}=-1,1,0$, respectively. The induced correspondence between segments in the original and transformed formulations is summarized in Table~\ref{tab:segment_mapping_compact}. 
\begin{table}[h]
\centering
\small
\setlength{\tabcolsep}{6pt}
\renewcommand{\arraystretch}{1.15}
\caption{Mapping of primitive segments under $\tilde v=u_g/U_{max}$ and $\tilde u_g=-v/U_{max}$ (with $\tilde U_{max}=1/U_{max}$).}
\label{tab:segment_mapping_compact}
\begin{tabular}{cc}
\toprule
Original segment & Transformed segment \\
\midrule
$L^{+}$ ($v=1$, $u_g=U_{max}$) & $R^{+}$ ($\tilde v=1, \tilde u_g=-\tilde U_{max}$) \\
$R^{+}$ ($v=1$, $u_g=-U_{max}$) & $R^{-}$ ($\tilde v=-1, \tilde u_g=-\tilde U_{max}$) \\
$L^{-}$ ($v=-1$, $u_g=U_{max}$) & $L^{+}$ ($\tilde v=1, \tilde u_g=\tilde U_{max}$) \\
$R^{-}$ ($v=-1$, $u_g=-U_{max}$) & $L^{-}$ ($\tilde v=-1, \tilde u_g=\tilde U_{max}$) \\
$G^{+}$ ($v=1$, $u_g=0$) & $R^{0}$ ($\tilde v=0, \tilde u_g=-\tilde U_{max}$) \\
$G^{-}$ ($v=-1$, $u_g=0$) & $L^{0}$ ($\tilde v=0, \tilde u_g=\tilde U_{max}$) \\
$L^{0}$ ($v=0$, $u_g=U_{max}$) & $G^{+}$ ($\tilde v=1, \tilde u_g=0$) \\
$R^{0}$ ($v=0$, $u_g=-U_{max}$) & $G^{-}$ ($\tilde v=-1, \tilde u_g=0$) \\
\bottomrule
\end{tabular}
\end{table}

In summary, the spherical CRS problem in \eqref{eq: time cost}-\eqref{terminal conditions} can be reformulated as:
\begin{equation}
\label{cost new}
    \tilde J^* =\min\int_0^{\Tilde{T}} 1 \; d\tau 
\end{equation}
\begin{equation}
\label{constraints new}
\text{subject to } \eqref{terminal constraints scaled}, \eqref{RS model new},
\end{equation}
\noindent with $\Tilde{v}\in[-1,1],$ $\Tilde{u}_g\in[-\Tilde{U}_{max}, \Tilde{U}_{max}]$, and $\Tilde{U}_{max}=\frac{1}{U_{max}}$.

We now formally show the equivalence between the above formulation and the original problem.

We first show that the map between the feasible sets of the two problems is a bijection (one-to-one correspondance).
Let $\mathcal{F}$ denote the set of feasible solutions of the original problem, i.e., tuples $(\mathbf{R}(\cdot),v(\cdot),u_g(\cdot),T)$ satisfying \eqref{RS model}, the terminal constraints \eqref{terminal conditions}, and the bounds $v\in[-1,1]$, $u_g\in[-U_{max},U_{max}]$.
Let $\tilde{\mathcal{F}}$ denote the feasible set of \eqref{cost new}--\eqref{constraints new}, i.e., tuples $(\Tilde{\mathbf{R}}(\cdot),\Tilde{v}(\cdot),\Tilde{u}_g(\cdot),\Tilde{T})$ satisfying \eqref{RS model new}, \eqref{terminal constraints scaled}, and the bounds $\Tilde{v}\in[-1,1]$, $\Tilde{u}_g\in[-\Tilde{U}_{max},\Tilde{U}_{max}]$.

Define the mapping $\Phi:\mathcal{F}\to\tilde{\mathcal{F}}$ by
\begin{equation*}
\tau=U_{max}t,\quad \Tilde{T}=U_{max}T,\quad 
\Tilde{\mathbf{R}}(\tau)=\mathbf{Q}^{-1}\mathbf{R}(t)\mathbf{Q},
\end{equation*}
\begin{equation}
\label{eq:Phi_map}
\Tilde{v}(\tau)=\frac{u_g(t)}{U_{max}},\quad
\Tilde{u}_g(\tau)=-\frac{v(t)}{U_{max}},
\end{equation}
with $t=\tau/U_{max}$.
Conversely, define $\Psi:\tilde{\mathcal{F}}\to\mathcal{F}$ by
\begin{equation*}
t=\frac{\tau}{U_{max}},\quad T=\frac{\Tilde{T}}{U_{max}},\quad
\mathbf{R}(t)=\mathbf{Q}\Tilde{\mathbf{R}}(\tau)\mathbf{Q}^{-1},
\end{equation*}
\begin{equation}
\label{eq:Psi_map}
u_g(t)=U_{max}\Tilde{v}(\tau),\quad
v(t)=-U_{max}\Tilde{u}_g(\tau),
\end{equation}
with $\tau=U_{max}t$.
It is straightforward to verify that, for $U_{max}>0$, $\Psi(\Phi(x))=x$ for all $x\in\mathcal{F}$ and $\Phi(\Psi(y))=y$ for all $y\in\tilde{\mathcal{F}}$; hence $\Phi$ is a bijection between $\mathcal{F}$ and $\tilde{\mathcal{F}}$.

Next, we show that the optimizers of the two problems are equivalent.
For any feasible $x:=(\mathbf{R}(\cdot),v(\cdot),u_g(\cdot),T)\in\mathcal{F}$ and its image
$\Phi(x)=(\tilde{\mathbf{R}}(\cdot),\tilde v(\cdot),\tilde u_g(\cdot),\tilde T)\in\tilde{\mathcal{F}}$,
the time scaling $\tau=U_{max}t$ yields
\begin{equation}\label{eq:cost_scaling}
J(x)=\int_{0}^{T}1\,dt=\frac{1}{U_{max}}\int_{0}^{\tilde T}1\,d\tau=\frac{1}{U_{max}}\tilde J(\Phi(x)).
\end{equation}
Since $U_{max}>0$, \eqref{eq:cost_scaling} implies the strict order-preserving relation
\begin{equation}\label{eq:order_preserve}
J(x_1)<J(x_2) \Longleftrightarrow \tilde J(\Phi(x_1))<\tilde J(\Phi(x_2)),\quad \forall x_1,x_2\in\mathcal{F}.
\end{equation}
Together with the fact that $\Phi$ is a bijection between $\mathcal{F}$ and $\tilde{\mathcal{F}}$, it follows that $\Phi(x^*)\in\tilde{\mathcal{F}}$ is an optimizer of \eqref{cost new}-\eqref{constraints new} if and only if $x^*\in\mathcal{F}$ is an optimizer of \eqref{eq: time cost}-\eqref{terminal conditions}.

Note that the problem \eqref{cost new}–\eqref{constraints new} has the same structure as \eqref{eq: time cost}, \eqref{RS model}, \eqref{terminal conditions}, and the new bound $\Tilde{U}_{max}>1$ if $U_{max}<1$. Hence, when $U_{max}<1$, \eqref{cost new}-\eqref{constraints new} can be directly solved with the results developed in this paper for $\tilde U_{max}>1$. Therefore, when $U_{max}<1$, the spherical CRS problem \eqref{eq: time cost}-\eqref{terminal conditions} can be solved with the following steps:

(i) Calculate the transformed desired terminal configuration $\Tilde{\mathbf{R}}_{\mathbf{f}}=\mathbf{Q}^{-1}\bold{R_f}\mathbf{Q}$ and the new bound $\Tilde{U}_{max}=\frac{1}{U_{max}}$;

(ii) Solve the problem \eqref{cost new}-\eqref{constraints new} using the results developed in this paper for $\Tilde{U}_{max}>1$;

(iii) Recover the optimal solution, denoted by superscript $^*$, with 
\begin{equation}
\label{eq:control recover}
u_g^*(t)=U_{max}\Tilde{v}^*({U_{max}}{t}), ~~v^*(t)=-U_{max}\Tilde{u}_g^*(U_{max}{t}).
\end{equation}

\emph{Numerical example:} Consider the case where $U_{max}=0.25<1$, and the desired terminal configuration 
\begin{align*}
    \mathbf{R_f}=
\begin{pmatrix}
-0.944360 & -0.283650 & \phantom{-}0.166512\\
\phantom{-}0.326943 & -0.754203 & \phantom{-}0.569461\\
-0.035944 & \phantom{-}0.592216 & \phantom{-}0.804977
\end{pmatrix},
\end{align*}
as shown in Fig. \ref{fig:config_transform}(a).
\begin{figure}[h]
    \centering
    \subfigure[$\mathbf{R_f}$]{\includegraphics[width=0.4\textwidth]{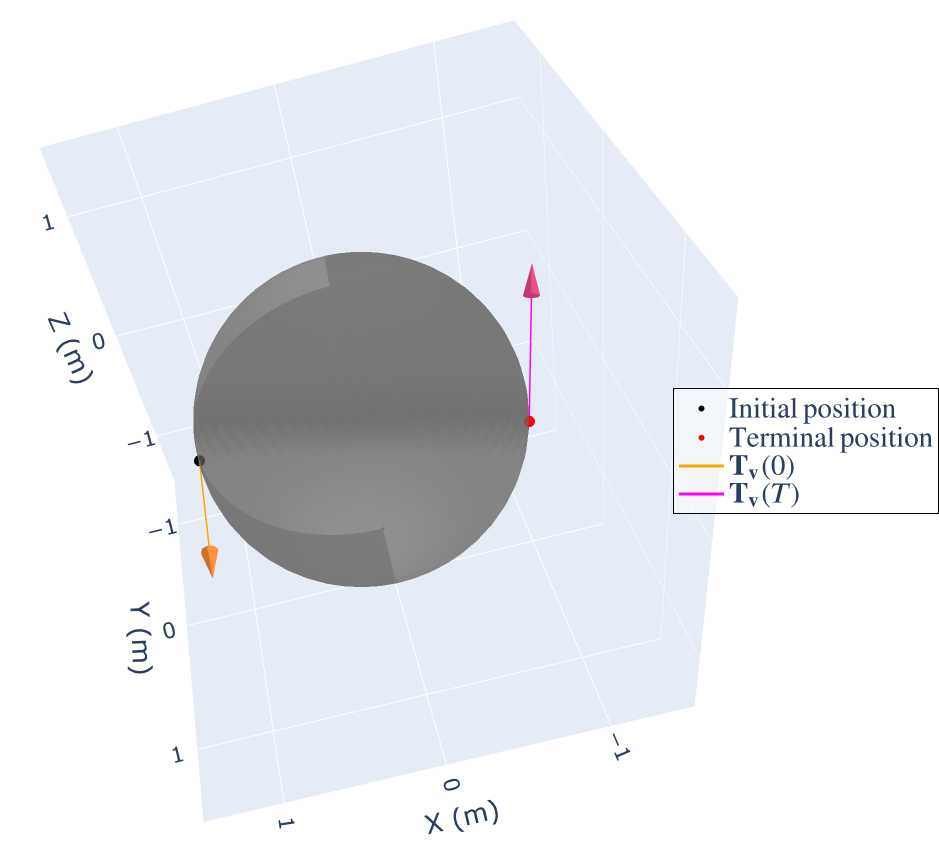}}
    \subfigure[$\mathbf{\tilde R_f}$]{\includegraphics[width=0.4\textwidth]{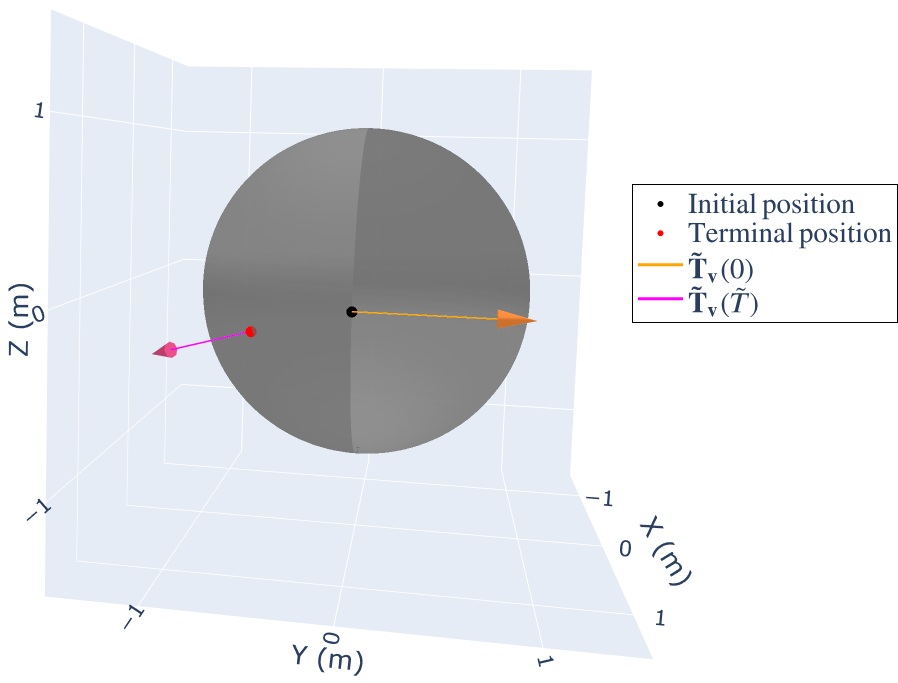}}
    \caption{Terminal configuration before and after the transformation}
    \label{fig:config_transform}
    \end{figure}

This results in the transformed new bound $\tilde U_{max}=\frac{1}{U_{max}}=4>1$, and the transformed terminal configuration
\begin{align*}
    \mathbf{\tilde R_f}&=\mathbf{Q}^{-1}\mathbf{R_fQ}\\&=
\begin{pmatrix}
\phantom{-}0.804977 &  -0.592216 &  \phantom{-}0.035944 \\
 -0.569461 & -0.754203 &  \phantom{-}0.326943 \\
 -0.166512 & -0.283650 &  -0.944360
\end{pmatrix},
\end{align*}
as shown in Fig. \ref{fig:config_transform}(b).

We then compute the feasible and time-optimal paths for the transformed problem (with $\tilde U_{max}>1$) using the results developed in this paper; see Fig.~\ref{fig:sol_transform}(a).
Finally, we recover the time-optimal path for the original problem by mapping the controls back via~\eqref{eq:control recover} (the feasible paths are recovered with a similar map as well); see Fig.~\ref{fig:sol_transform}(b). The segment labels are recovered via Table~\ref{tab:segment_mapping_compact}.
    \begin{figure}[h]
    \centering\subfigure[Solution of the transformed problem]{\includegraphics[width=0.4\textwidth]{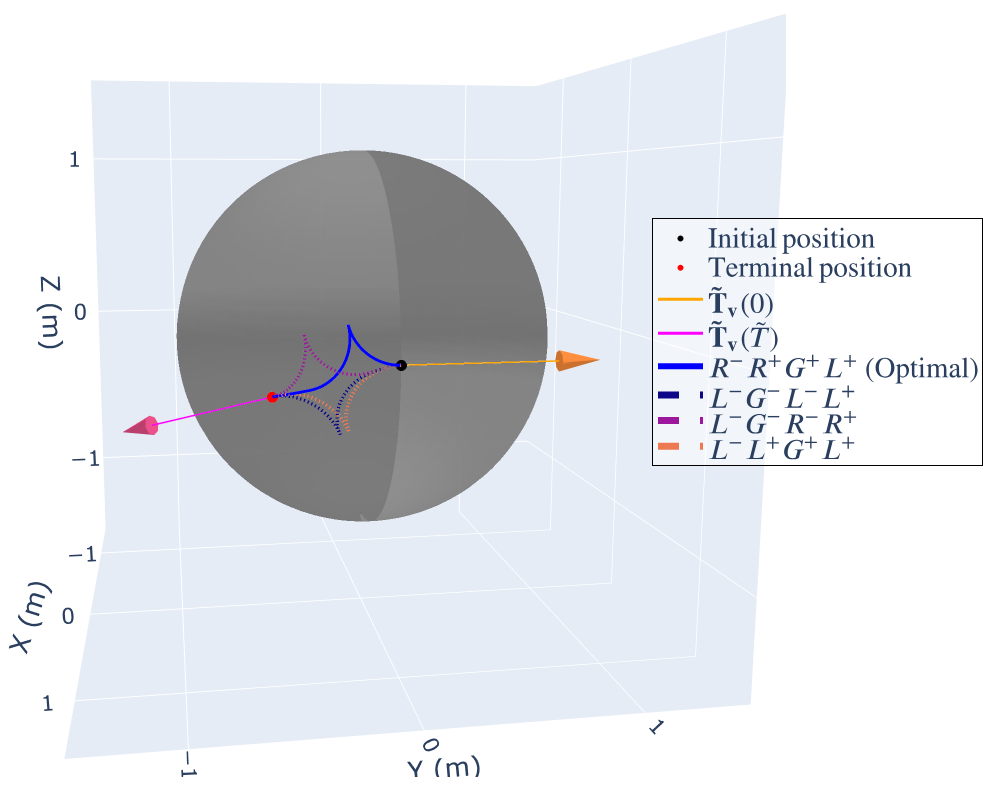}}
 \subfigure[Recovered solution of the original problem]{\includegraphics[width=0.4\textwidth]{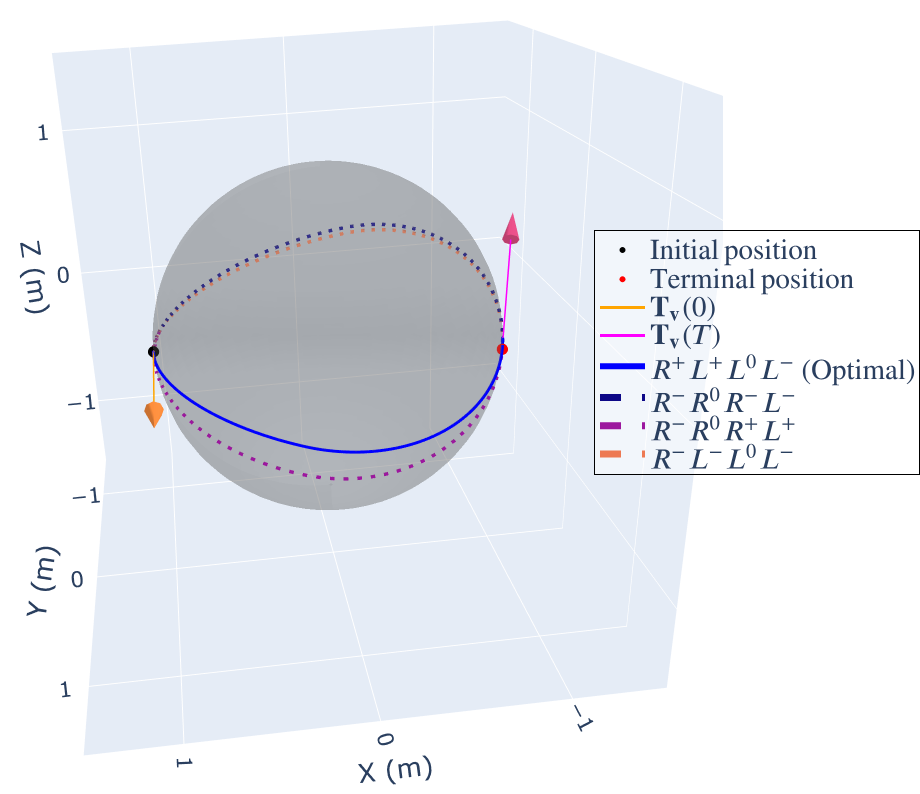}}
    \caption{Solving the transformed problem and recovering the solution of the original problem}
\label{fig:sol_transform}
    \end{figure}

\subsection{PMP derivations}
\label{appd PMP}
Equation (\ref{RS model}) can be rewritten in terms of left-invariant vector-fields $\overrightarrow{l}_1 \left(\bold{R} (t) \right)$ and $\overrightarrow{L}_{12} \left(\bold{R} (t) \right)$, defined as
\begin{align}
    \overrightarrow{l}_1 \left(\bold{R} (t) \right) &= \bold{R} (t) l_1 = \bold{R} (t) \begin{pmatrix}
        0 & -1 & 0 \\
        1 & 0 & 0 \\
        0 & 0 & 0
    \end{pmatrix}, \\
    \overrightarrow{L}_{12} \left(\bold{R} (t) \right) &= \bold{R} (t) L_{12} = \bold{R} (t) \begin{pmatrix}
        0 & 0 & 0 \\
        0 & 0 & 1 \\
        0 & -1 & 0
    \end{pmatrix}.
\end{align}
Here, $l_1$ and $L_{12}$ lie in $so (3),$ the Lie-algebra of the Lie group $SO (3).$ Equation (\ref{RS model}) can therefore be rewritten as
\begin{align}
    \frac{d\bold{R}(t)}{dt} = v (t) \overrightarrow{l}_1 (\bold{R} (t)) - u_g (t) \overrightarrow{L}_{12} (\bold{R} (t)).
\end{align}

Similar to \cite{monroy1998non}, PMP is applied for the symplectic formalism \cite{jurdjevic1997geometric}, wherein PMP is applied on the cotangent bundle of the Lie group, $T^*SO(3)$. Hence, the Hamiltonian is given by
\begin{align}
\begin{split}
\label{Halmil0_}
    &H_{u_g (t), v (t)} \left(\zeta_0, \zeta \right) \\
    &= \zeta_0 (1) + \zeta \left(v (t) \overrightarrow{l}_1 (\bold{R} (t)) - u_g (t) \overrightarrow{L}_{12} (\bold{R} (t)) \right) \\
    &= \zeta_0 + v (t) \zeta \left(\overrightarrow{l}_1(\bold{R} (t)) \right) -  u_g (t) \zeta\left(\overrightarrow{L}_{12} (\bold{R} (t)) \right) \\
    &= \zeta_0 + v (t) h_1 (\zeta (t)) -  u_g (t) H_{12} (\zeta (t)),
\end{split}
\end{align}
where $h_1$ and $H_{12}$ are smooth functions in $T^* SO (3),$ the cotangent bundle of the Lie group $SO (3),$ and the integral curve $\zeta \in T^* SO (3)$.

$\textbf{The Pontryagin maximum principle (PMP)}$ \cite{chang2011simple,monroy1998non}: If ($\bold{R(t)},v(t),u_g(t)$) is an optimal trajectory for the system in (\ref{RS model}) on an interval $[0, \Bar{t}]$, then it is the projection of the integral curve
$\zeta (t)$ corresponding to the Hamiltonian vector field associated with the
left-invariant vector field using which the system evolves, so that

1) $\zeta_0\leq0$ and is constant in $t$.

2) If $\zeta_0=0$, then $\zeta(t)$ is not identically zero in $[0,\Bar{t}]$.

3) $H_{u_g (t), v (t)} \left(\zeta_0, \zeta(t) \right)\geq H_{\Bar{u_g}, \Bar{v}} \left(\zeta_0, \zeta(t) \right)$ for all $\Bar{v}\in[-1,1],\Bar{u_g}\in[-U_{max},U_{max}]$, and almost all $t\in[0,\Bar{t}]$.

4) $H_{u_g (t), v (t)} \left(\zeta_0, \zeta(t) \right)$ is zero for almost all $t\in[0,\Bar{t}]$.
\vspace{0.5em}

Furthermore, $h_1, h_2,$ and $H_{12}$ are Hamiltonians corresponding to the left-invariant vector fields $\overrightarrow{l}_1, \overrightarrow{l}_2,$ and $\overrightarrow{L_{12}}$ (page 363 of \cite{jurdjevic1997geometric}), and the Poisson bracket of the Hamiltonians satisfies the same relations as the Lie bracket of the vector fields. That is,
\begin{align}
    [L_{12}, l_1] = l_2, \quad \implies \{H_{12}, h_1 \} = h_2, \\
    [l_1, l_2] = L_{12}, \quad \implies \{h_1, h_2 \} = H_{12}, \\
    [l_2, L_{12}] = l_1, \quad \implies \{h_2, H_{12} \} = h_1.
\end{align}
The evolution of the Hamiltonians from PMP are given by\footnote{Note that the Poisson bracket with a constant (in this case, $\zeta_0$), is zero.} \cite{monroy1998non}
\begin{align}
\label{Ham1_}
    \frac{d{h}_1 (\zeta (t))}{dt} &= \{h_1, H_{u_g, v} \} = \{h_1, v h_1 - u_g H_{12} \} \notag\\&= u_g h_2, 
    \end{align}
    \begin{align}
    \label{Ham2_}
    \frac{d{h}_2 (\zeta (t))}{dt} &= \{h_2, H_{u_g, v} \} = \{h_2, v h_1 - u_g H_{12} \} \notag\\&= v \{h_2, h_1 \} - u_g \{h_2, H_{12} \} \notag\\
    &= -v H_{12} - u_g h_1, \end{align}
    \begin{align}
    \label{Ham3_}
    \frac{d{H}_{12} (\zeta (t))}{dt} &= \{H_{12}, H_{u_g, v} \} = \{H_{12}, v h_1 - u_g H_{12} \}\notag\\& = v \{H_{12}, h_1 \} = v h_2.
\end{align}

\subsection{Proof of Lemma \ref{lem1}}
\label{appd lem1}

By Condition 4) of PMP, $H_{u_g,v}\equiv0$. Condition 3) implies $vh_1\geq0$,
$u_gH_{12}\leq0$, with $u_g\neq0$ if $H_{12}\neq0$ and $v\neq0$ if $h_1\neq0$.
If $\zeta_0=0$, then $H_{u_g,v}=vh_1-u_gH_{12}\equiv0$, which is only possible
when $h_1\equiv0$ and $H_{12}\equiv0$. Furthermore, by (\ref{Ham1_}) and (\ref{Ham3_}),
$\frac{dH_{12}}{dt}=vh_2\equiv0$ and $\frac{dh_1}{dt}=u_gh_2\equiv0$, implying either
$h_2\equiv0$ or $v=u_g\equiv0$. The latter yields $\frac{d\bold{R}(t)}{dt}\equiv0$
from (\ref{RS model}), corresponding to a stationary path, which is
non-optimal for a minimum-time problem. The former gives
$h_1=h_2=H_{12}\equiv0$, violating Condition 2) of PMP \cite{monroy1998non}.

\subsection{Proof of Lemma \ref{lem full cusp equivalence}}
\label{appnd B}
Consider a $L^+_\beta|L^-_\beta$ path and a $R^+_\beta|R^-_\beta$ path with $\beta=\arctan(\frac{1}{\sqrt{U_{max}^4-1}})+\frac{\pi}{2}$. By (\ref{R sol}), it suffices to show that $\bold{M}_{L^+}(r,\beta)\bold{M}_{L^-}(r,\beta)=\bold{M}_{R^+}(r,\beta)\bold{M}_{R^-}(r,\beta)$. It can be obtained that $\sin(\arctan(\frac{1}{\sqrt{U_{max}^4-1}}))=\frac{1}{U_{max}^2}$ and $\cos(\arctan(\frac{1}{\sqrt{U_{max}^4-1}}))=\frac{\sqrt{U_{max}^4-1}}{U_{max}^2}$, hence, $\sin(\beta)=\frac{\sqrt{U_{max}^4-1}}{U_{max}^2}$ and $\cos(\beta)=-\frac{1}{U_{max}^2}$. Substituting the value of $\sin(\beta), \cos(\beta)$, and $r=\frac{1}{\sqrt{1+U_{max}^2}}$ into (\ref{M_L+}), (\ref{M_L-}), (\ref{M_R+}), and (\ref{M_R-}), it is obtained that
    \begin{align}
        &\bold{M}_{L^+}(r,\beta)\bold{M}_{L^-}(r,\beta)=\bold{M}_{R^+}(r,\beta)\bold{M}_{R^-}(r,\beta) \notag \\
        =&\left(\begin{array}{ccc}
        \frac{U_{max}^2-2}{U_{max}^2} & \frac{2\sqrt{U_{max}^2-1}}{U_{max}^2} & 0 \\
        \frac{2\sqrt{U_{max}^2-1}}{U_{max}^2} & -\frac{U_{max}^2-2}{U_{max}^2} & 0 \\
        0 & 0 & -1 \\
        \end{array} \right).
    \end{align}

    With a similar proof for $L^-_\beta|L^+_\beta$ and $R^-_\beta|R^+_\beta$ paths, the lemma is proved.

\subsection{Proof of Lemma \ref{lem GC|CG replace}}
\label{append C}
 Suppose a path is of type $G^-_\sigma L^-_\beta|L^+_\beta$ with $\beta=\arctan(\frac{1}{\sqrt{U_{max}^4-1}})+\frac{\pi}{2}$ and $\sigma\geq0$. Suppose that it can be replaced by a path of type $L^-_\beta|L^+_\beta G^+_{\sigma}$. It suffices to show that $\bold{M}_{L^-}(r,\beta)\bold{M}_{L^+}(r,\beta)\bold{M}_{G^+}(\sigma)=\bold{M}_{G^-}(\sigma)\bold{M}_{L^-}(r,\beta)\bold{M}_{L^+}(r,\beta)$. As shown in the proof of Lemma \ref{lem full cusp equivalence}, $\sin(\beta)=\frac{\sqrt{U_{max}^4-1}}{U_{max}^2}$ and $\cos(\beta)=-\frac{1}{U_{max}^2}$. Substituting the values of $\sin(\beta), \cos(\beta)$, and $r=\frac{1}{\sqrt{1+U_{max}^2}}$ into (\ref{M_G+}), (\ref{M_L+}), and (\ref{M_L-}), it is obtained that
    \begin{align}
        &\bold{M}_{L^-}(r,\beta)\bold{M}_{L^+}(r,\beta)\bold{M}_{G^+}(\sigma) \notag \\
        =&\bold{M}_{G^-}(\sigma)\bold{M}_{L^-}(r,\beta)\bold{M}_{L^+}(r,\beta) \notag \\
        =&\left(\begin{array}{ccc}
        \xi_1 & \xi_2 & 0 \\
        \xi_2 & -\xi_1 & 0 \\
        0 & 0 & -1 \\
        \end{array} \right),
    \end{align}
    where $\xi_1=c(\sigma)\frac{U_{max}^2-2}{U_{max}^2}-s(\sigma)\frac{2\sqrt{U_{max}^2-1}}{U_{max}^2}$ and $\xi_2=-s(\sigma)\frac{U_{max}^2-2}{U_{max}^2}-c(\sigma)\frac{2\sqrt{U_{max}^2-1}}{U_{max}^2}$.
    
    With similar proofs for path types $G^-_\sigma R^-_\beta|R^+_\beta$, $G^+_\sigma L^+_\beta|L^-_\beta$, and $G^+_\sigma R^+_\beta|R^-_\beta$, the lemma is proved.

\subsection{Proof of Lemma \ref{lem CCC2G}}
\label{appnd D}
Consider a $L^+_\mu R^+_\epsilon L^+_\mu$ path where $0<\mu<\arctan(\frac{1}{\sqrt{U_{max}^4-1}})+\frac{\pi}{2}$. For it to be replaced by a $G^+_\theta$ path, the following equation must be satisfied: 
    \begin{equation}
    \label{Eq:LRL2G}
    \bold{M}_{L^+}(r,\mu)\bold{M}_{R^+}(r,\epsilon)\bold{M}_{L^+}(r,\mu)=\bold{M}_{G^+}(\theta).
    \end{equation}

    We utilize the following axial vectors for the characterization of $\epsilon$:
    \begin{equation}
    \label{Eq:uLp_u_Rp}
        \bold{u}_{L^+}:=
        \begin{pmatrix}
        \sqrt{1 - r^2}, 0, r
        \end{pmatrix}^T,
        \bold{u}_{R^+}:=
        \begin{pmatrix}
        -\sqrt{1 - r^2}, 0, r
        \end{pmatrix}^T,
    \end{equation}
    where $\bold{u}_{L^+}$ and $\bold{u}_{R^+}$ are the axial vectors of $\bold{M}_{L^+}(r,\cdot)$ and $\bold{M}_{R^+}(r,\cdot)$, respectively. Therefore, $\bold{M}_{L^+}(r,\mu)\bold{u}_{L^+}=\bold{u}_{L^+}$, $\bold{u}_{L^+}^T\bold{M}_{L^+}(r,\mu)=\bold{u}_{L^+}^T$, $\bold{M}_{R^+}(r,\mu)\bold{u}_{R^+}=\bold{u}_{R^+}$, and $\bold{u}_{R^+}^T\bold{M}_{R^+}(r,\mu)=\bold{u}_{R^+}^T$. 

    By pre-multiplying both sides of (\ref{Eq:LRL2G}) with $\bold{u}_{L^+}^T$ and post-multiplying with $\bold{u}_{L^+}$, it is obtained that:
    \begin{equation}
    \label{Eq:LRL2G_uLpuLp}
    \bold{u}_{L^+}^T\bold{M}_{R^+}(r,\epsilon)\bold{u}_{L^+}=\bold{u}_{L^+}^T\bold{M}_{G^+}(\theta)\bold{u}_{L^+}.
    \end{equation}

    By further substituting (\ref{M_G+}), (\ref{M_R+}) and (\ref{Eq:uLp_u_Rp}) into (\ref{Eq:LRL2G_uLpuLp}), it is obtained that:
    \begin{equation}
    \label{Eq:LRL2G_1}
        r^2-c(\theta)(r^2-1)=(4r^2-4r^4)\left(c(\epsilon)-1\right)+1.
    \end{equation}

    Similarly, by pre-multiplying both sides of (\ref{Eq:LRL2G}) with $\bold{u}_{L^+}^T$ and post-multiplying with $\bold{u}_{R^+}$, it is obtained that:
    \begin{align}
    \label{Eq:LRL2G_2}
        r^2&+c(\theta)(r^2-1) \notag \\
        =&(4r^2-4r^4)\Big(c(\mu)(-1+2r^2)+s(\mu)s(\epsilon)+c(\epsilon)(-1+2r^2) \notag \\
        &+c(\mu)c(\epsilon)(1-2r^2)+1-2r^2\Big)+2r^2-1.
    \end{align}

    By taking the sum of (\ref{Eq:LRL2G_1}) and (\ref{Eq:LRL2G_2}) on both sides and utilizing the fact that $4r^2-4r^4\neq0$ for $r\leq\frac{1}{\sqrt{2}}$, we obtain:
    \begin{equation}
    \label{eq 54}
        c(\epsilon)\Big(c(\mu)(1-2r^2)+2r^2\Big)+s(\epsilon)s(\mu)=c(\mu)(1-2r^2)+2r^2.
    \end{equation}
    
    We now divide each side of the above equation by $\sqrt{(c(\mu)(1-2r^2)+2r^2)^2+s^2(\mu)}$, noting that this quantity is non-zero. This is justified since $s^2(\mu)>0$ for $0<\mu<\arctan(\frac{1}{\sqrt{U_{max}^4-1}})+\frac{\pi}{2}\leq\pi$. Furthermore, defining $c(\delta):=\frac{c(\mu)(1-2r^2)+2r^2}{\sqrt{(c(\mu)(1-2r^2)+2r^2)^2+s^2(\mu)}}$ and $s(\delta):=\frac{s(\mu)}{\sqrt{(c(\mu)(1-2r^2)+2r^2)^2+s^2(\mu)}}$, (\ref{eq 54}) can be rewritten as:
    \begin{equation}
        c(\epsilon-\delta)=c(\delta).
    \end{equation}

    Therefore,
    \begin{equation}
    \label{eq 56}
        \epsilon=c^{-1}(c(\delta))+\delta.
    \end{equation}

    It is worth noting that $c^{-1}(c(\delta))$ has two possible values within $[0,2\pi]$, namely, $\delta$ and $2\pi-\delta$, we next prove that it can only be equal to $\delta$. Suppose $s\Big(c^{-1}(c(\delta))\Big)=s(2\pi-\delta)=-s(\delta)$, then $c(\epsilon)=c(\delta)c\Big(c^{-1}(c(\delta))\Big)-s(\delta)s\Big(c^{-1}(c(\delta))\Big)=c^2(\delta)+s^2(\delta)=1$. Since it is defined that $0<\epsilon<2\mu<2\pi$, this is not possible. Hence, $s\Big(c^{-1}(c(\delta))\Big)=s(\delta)$. Therefore, $c^{-1}(c(\delta))=\delta$, and by (\ref{eq 56}) we obtain $\epsilon=2\delta$.

    Next, we will show that $0<\epsilon=2\delta<2\mu$. Since $0<\mu<\arctan(\frac{1}{\sqrt{U_{max}^4-1}})+\frac{\pi}{2}\leq\pi$, $s(\mu)>0$. Therefore, $s(\delta)=\frac{s(\mu)}{\sqrt{(c(\mu)(1-2r^2)+2r^2)^2+s^2(\mu)}}>0$, hence, $0<\delta<\pi$. We further compare $\delta$ with $\mu$ by taking their difference:
    \begin{align}
        s(\mu-\delta)&=s(\mu)c(\delta)-c(\mu)s(\delta) \notag \\
        &=\frac{s(\mu)\Big(c(\mu)(1-2r^2)+2r^2\Big)-c(\mu)s(\mu)}{\sqrt{(c(\mu)(1-2r^2)+2r^2)^2+s^2(\mu)}} \notag \\
        &=\frac{2r^2s(\mu)\Big(1-c(\mu)\Big)}{\sqrt{(c(\mu)(1-2r^2)+2r^2)^2+s^2(\mu)}}.
    \end{align}

    Since $0<\mu<\arctan(\frac{1}{\sqrt{U_{max}^4-1}})+\frac{\pi}{2}\leq\pi$, $s(\mu)>0$ and $-1<c(\mu)<1$. Hence, $s(\mu-\delta)>0$. Furthermore, since $0<\delta<\pi$, $\mu-\delta>0$. Hence, $\mu>\delta>0$, and $0<\epsilon=2\delta<2\mu$.

    Substituting $c(\epsilon)=c(2\delta)=c^2(\delta)-s^2(\delta)$ and $s(\epsilon)=s(2\delta)=2s(\delta)c(\delta)$ into (\ref{Eq:LRL2G}), it can be checked that:
    \begin{align}
    \bold{M}_{G^+}(\theta)&=\bold{M}_{L^+}(r,\mu)\bold{M}_{R^+}(r,\epsilon)\bold{M}_{L^+}(r,\mu) \notag \\
    &=\left(\begin{array}{ccc}
    c(\theta) & -s(\theta) & 0 \\
    s(\theta) & c(\theta) & 0 \\
    0 & 0 & 1 \\
    \end{array} \right),
    \end{align}
    where $c(\theta)=\frac{(4r^2-8r^4)c(\mu)+(4r^2+4r^4)c^2(\mu)-8r^2+4r^4+1}{(4r^2-8r^4)c(\mu)+(-4r^2+4r^4)c^2(\mu)+4r^4+1}$ and $s(\theta)=\frac{4r\big((1-2r^2)s(\mu)+2r^2c(\mu)s(\mu)\big)}{(4r^2-8r^4)c(\mu)+(-4r^2+4r^4)c^2(\mu)+4r^4+1}$.

Visualizations of a $G^+_\theta$ path replacing a $L^+_\mu R^+_\epsilon L^+_\mu$ path are shown in Fig. \ref{fig:L_pos_R_pos_L_pos} for $\mu=\frac{\pi}{3}$ and $\mu=\frac{2\pi}{3}$. In both cases, $U_{max}=3$ (or $r=\frac{1}{\sqrt{10}}$) and the corresponding values of $\epsilon$ and $\theta$ are calculated.

 \begin{figure}[h]
    \centering
    \subfigure[$\mu=\frac{\pi}{3}$, $\epsilon=0.61\pi$, $\theta=0.35\pi$]{\includegraphics[width=0.24\textwidth]{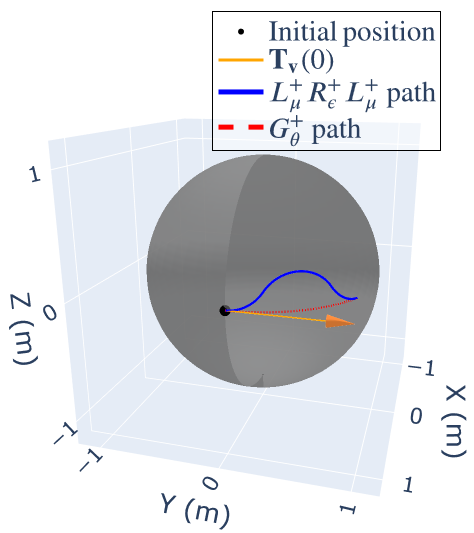}}
    \subfigure[$\mu=\frac{2}{3}\pi$, $\epsilon=1.14\pi$, $\theta=0.42\pi$]{\includegraphics[width=0.24\textwidth]{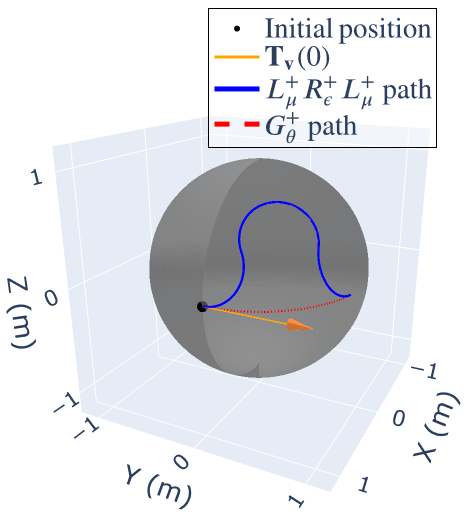}}
    \caption{Replacing a $L^+_\mu R^+_\epsilon L^+_\mu$ path with a $G^+_\theta$ path}
    \label{fig:L_pos_R_pos_L_pos}
    \end{figure}

    With similar proofs for $R^+_\mu L^+_\epsilon R^+_\mu$, $L^-_\mu R^-_\epsilon L^-_\mu$, and $R^-_\mu L^-_\epsilon R^-_\mu$ paths, the lemma is proved.

\bibliographystyle{IEEEtran}
\bibliography{IEEEabrv,root}

\end{document}